\PassOptionsToPackage{table}{xcolor}
\documentclass[10pt]{article} 
\usepackage[preprint]{rlj} 

\usepackage{amssymb}            
\usepackage{mathtools}          
\usepackage{mathrsfs}           
\usepackage{graphicx}           
\usepackage{subcaption}         
\usepackage[space]{grffile}     
\usepackage{url}                
\usepackage{lipsum}             
\usepackage{wrapfig}
\usepackage{relsize}

\usepackage{soul}
\usepackage{adjustbox}
\usepackage{tabularx, booktabs, array}
\usepackage{makecell,multirow,multicol}
\usepackage{float}
\usepackage[table]{xcolor}
\usepackage{mdframed}
\usepackage{etoc}

\usepackage{mathtools}
\usepackage{booktabs}
\usepackage{xspace}
\usepackage{xcolor}
\usepackage{bm} 
\usepackage{amsthm}
\usepackage{thmtools}

\declaretheorem[name=Lemma]{lemma}
\declaretheorem[name=Definition]{definition}

\declaretheorem[name=Proposition]{proposition}
\declaretheorem[name=Remark,style=remark,qed=$\square$]{remark}

\newcommand{\A}{\mathcal A}

\newcommand{\E}[2]{\mathbb{E}_{#1}{\left[#2\right]}}

\makeatletter
\DeclareRobustCommand\onedot{\futurelet\@let@token\@onedot}
\def\@onedot{\ifx\@let@token.\else.\null\fi\xspace}

\def\eg{\emph{e.g}\onedot}

\newcommand{\defeq}{\mathrel{\mathop:}=}

\newcommand{\supp}{\mathrm{supp}}

\title{{Understanding Behavioral Metric Learning}:\\ {\Large\smaller[1] A Large-Scale Study on Distracting Reinforcement Learning Environments}}

\setrunningtitle{Understanding Behavioral Metric Learning in Deep RL}

\author{Ziyan "Ray" Luo\textsuperscript{1,2}, Tianwei Ni\textsuperscript{1,3}, \\
Pierre-Luc Bacon\textsuperscript{1,3,5}, Doina Precup\textsuperscript{1,2,5}, Xujie Si\textsuperscript{1,4,5}\looseness=-1}

\coverPageAuthor{Ziyan "Ray" Luo, Tianwei Ni, Pierre-Luc Bacon, Doina Precup, Xujie Si}

\emails{ziyan.luo@mail.mcgill.ca, twni2016@gmail.com}

\affiliations{
$^{1}$\textbf{Mila - Quebec Artificial Intelligence Institute}
$^{2}$\textbf{McGill University}
$^{3}$\textbf{Université de Montréal}
$^{4}$\textbf{University of Toronto \& Vector Institute}
$^{5}$\textbf{Canada CIFAR AI Chair}

}

\contribution{
We analyze five recent metric learning approaches under the isometric embedding framework to identify key design choices.
    }
    {
Metric learning methods often diverge significantly between theory and implementation.
    }
\contribution{
We introduce the denoising factor to quantify an encoder’s ability to filter distractions.
    }
    {
Metric learning is often motivated by denoising ability but is rarely evaluated directly, with prior work relying mainly on qualitative analysis~\citep{zhang2020learning}.
    }
\contribution{
We benchmark five metric learning approaches across diverse distracting domains and find that common benchmarks add little difficulty to clean tasks, while certain noise settings remain challenging. We adopt most hyperparameters provided in their codebases, as we believe they are well-tuned for their benchmarks that share the same denoised states as ours.\looseness=-1   
    }
    {
Prior work primarily uses IID Gaussian noise with varied dimensions~\citep{ni2024bridging} and grayscale video backgrounds~\citep{zhang2020learning}.
    }
\contribution{
Through ablation studies, we identify layer normalization and self-prediction loss as key design choices across all methods. 
    }
    {
Prior work in metric learning does not isolate the effect of self-prediction loss and only shows the benefits of normalization in specific methods~\citep{zang2022simsr}.
    }
\contribution{
We show that the benefits of metric learning diminish in both return and denoising factor when key design choices are incorporated into the baseline.
    }
    {
Prior work does not report this limitation of metric learning.
    }


\summary{
A key approach to state abstraction is approximating \textit{behavioral metrics} (notably, bisimulation metrics) in the observation space, and embedding these learned distances in the representation space. While promising for robustness to task-irrelevant noise, as shown in prior work, accurately estimating these metrics remains challenging, requiring various design choices that create gaps between theory and practice. 
Prior evaluations focus mainly on final returns, leaving the quality of learned metrics and the source of performance gains unclear.
To systematically assess how metric learning works in deep reinforcement learning (RL), we evaluate five recent approaches, unified conceptually as isometric embeddings with varying design choices. 
We benchmark them with baselines across 20 state-based and 14 pixel-based tasks, spanning \textit{370 task configurations} with diverse noise settings.
Beyond final returns, we introduce the evaluation of a \textit{denoising factor} to quantify the encoder’s ability to filter distractions. 
To further isolate the effect of metric learning, we propose and evaluate an \textit{isolated metric estimation} setting, in which the encoder is influenced solely by the metric loss.
Finally, we release an open-source, modular codebase to improve reproducibility and support future research on metric learning in deep RL.
}

\begin{document}

\makeCover
\maketitle


\vspace{-1em}
\begin{abstract}
\vspace{-0.5em}
A key approach to state abstraction is approximating \textit{behavioral metrics} (notably, bisimulation metrics) in the observation space and embedding these learned distances in the representation space. While promising for robustness to task-irrelevant noise, as shown in prior work, accurately estimating these metrics remains challenging, requiring various design choices that create gaps between theory and practice. 
Prior evaluations focus mainly on final returns, leaving the quality of learned metrics and the source of performance gains unclear.
To systematically assess how metric learning works in deep reinforcement learning (RL), we evaluate five recent approaches, unified conceptually as isometric embeddings with varying design choices. 
We benchmark them with baselines across 20 state-based and 14 pixel-based tasks, spanning \textit{370 task configurations}\footnote{200 state-based IID Gaussian (20 tasks × 10 noises), 84 pixel-based ID generalization (14 tasks × 6 noises), 30 state-based IID Gaussian with random projection (6 tasks × 5 noises), and 56 pixel-based OOD generalization (14 tasks × 4 noises).\looseness=-1} with diverse noise settings.
Beyond final returns, we introduce the evaluation of a \textit{denoising factor} to quantify the encoder's ability to filter distractions. 
To further isolate the effect of metric learning, we propose and evaluate an \textit{isolated metric estimation} setting, in which the encoder is influenced solely by the metric loss.
Finally, we release an open-source, modular codebase to improve reproducibility and support future research on metric learning in deep RL.\footnote{\label{fn:artifact}The artifact is available at \url{https://github.com/Rayluo-mila/understanding-metric-learning}.}\looseness=-1
\end{abstract}

\begin{figure}[h]
\vspace{-1.6em}
    \centering\includegraphics[width=\linewidth]{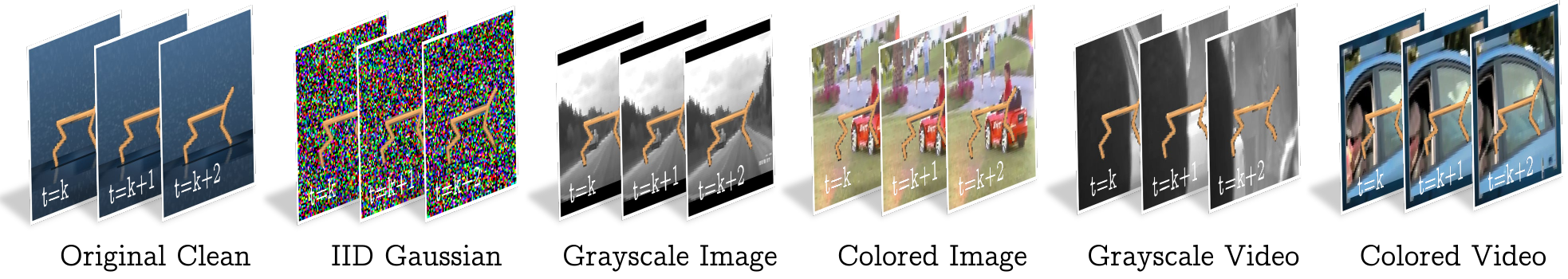}
\vspace{-2em}
    \caption{\small \textbf{Examples of background noise settings in pixel-based domains.} 
    In image settings, the background is fixed; in video settings, it varies slightly; IID Gaussian noise is independently sampled each timestep.
}
    \vspace{-1em}
\label{fig:pixel_noise}
\end{figure}

\vspace{-1em}
\section{Introduction}
\label{sec:intro}
\vspace{-0.8em}
\begingroup
\setlength{\parskip}{4pt}

    

Real-world environments often present high-dimensional, noisy observations, posing challenges for RL. For instance, in image-based settings, task-irrelevant variations in background, lighting, and viewpoint introduce distractions (e.g., \autoref{fig:pixel_noise}). Yet, despite this observational complexity, system dynamics are typically governed by \textit{a compact, task-relevant state}.
State abstraction~\citep{li2006towards,konidaris2019necessity} provides a framework for extracting such \textit{latent representations} from raw observations, filtering out irrelevant information while preserving task-critical structure.
A key principle of state abstraction is that behaviorally similar states should have similar representations. Traditionally, this is enforced through state aggregation~\citep{singh1994reinforcement,givan2003equivalence}, grouping states into discrete abstract classes based on equivalence relations.
However, state aggregation lacks a measure of \textit{how different} states are across classes and struggles with continuous representations, requiring infinitely many discrete classes.

To address this, \textit{bisimulation metrics}~\citep{ferns2004metrics,ferns2011bisimulation} and their scalable variants~\citep{castro2020scalable, zhang2020learning} have been proposed to define meaningful distances between observations.
These fall into the broader class of \textbf{behavioral metrics}~\citep{castro2023kernel}, which quantify state similarity based on differences in immediate rewards and transition probabilities.
By learning a metric alongside deep RL, prior work~\citep{zhang2020learning,kemertas2021towards,chen2022learning,zang2022simsr}  has shown progress in tackling high-dimensional, noisy tasks.\looseness=-1

Nevertheless, the role of behavioral metric learning in deep RL (\textbf{metric learning} for short) remains unclear due to a lack of systematic evaluation.
First, metric learning's effectiveness relies on accurate estimation in theory, but it is challenging in practice due to many design choices. Moreover, prior work primarily measures performance through \textit{returns}, without directly assessing the quality of learned metrics.
Second, metric learning is often combined with \textit{multiple losses} (e.g., self-prediction~\citep{zhang2020learning}, inverse dynamics~\citep{kemertas2021towards}), as well as \textit{architectural choices} (e.g., normalization, ensembles~\citep{zang2022simsr}), making it difficult to isolate metric learning's impact on performance gains.
Third, most studies evaluate only \textit{OOD generalization} in environments with \textit{grayscale} natural videos as distractions~\citep{zhang2020learning}, conflating robustness with generalization. 
Lastly, prior evaluations~\citep{tomar2021learning,li2022does} report inconsistent results for the same algorithms, raising concerns about reproducibility.\looseness=-1


\textbf{Contributions.}\hspace{0.6em} In this paper, we provide an understanding of \textbf{how metric learning works in deep RL} through a systematic large-scale study.
Our main contributions are as follows:
\begin{enumerate}[noitemsep, itemsep=0.1\baselineskip]
\item \textbf{Conceptual insights} (\autoref{sec:3}): We unify five recent metric learning approaches under an isometric embedding framework to identify key design choices. We analyze why some exact behavioral metrics provide a theoretical denoising guarantee, whereas others may not.
\item \textbf{Evaluation designs on denoising} (\autoref{sec:4}): To ensure a rigorous and comprehensive evaluation, we introduce diverse distraction benchmarks with varying difficulty levels, 
across both state-based and pixel-based domains, tested under both ID and OOD generalization.
Then, we quantify the \textbf{denoising} capability -- the encoder’s ability to filter out distractions by introducing the \textit{denoising factor} (DF). 
Finally, we propose an \textit{isolated metric estimation} setting (referred to as the isolated setting) to assess metric learning’s contribution to denoising, independent of other losses.\looseness=-1
\item \textbf{Comprehensive evaluation} (\autoref{sec:exp}): 
We conduct a comprehensive benchmark of five metric learning methods and baselines across 20 state-based tasks with 10 IID Gaussian noise levels, and 14 pixel-based tasks with 6 distraction types, in DeepMind Control suite~\citep{tassa2018deepmind}, evaluating both return and DF.
Beyond performance comparison, we assess the difficulty of our distracting benchmarks, and conduct targeted case studies to identify key design choices and examine the connection between metric learning and denoising throughout ablation study and isolated metric estimation. 
\item \textbf{Open-source codebase} (\autoref{fn:artifact}): We open-source a modular and efficient codebase to enhance the  reproducibility and extensibility of metric-based methods in the RL community. 
\end{enumerate}

\textbf{Main Findings.} Based on the evaluation, we highlight the main findings as follows:
\begin{enumerate}
    \item \textbf{Benchmarking results} (\autoref{sec:5.1}): SimSR, despite being designed for pixel tasks, outperforms all methods in return and denoising on state-based domains with IID Gaussian noise. RAP performs best in pixel-based tasks. Surprisingly, SAC and DeepMDP are strong baselines. 
    Interestingly, common distractions like varying noise dimensions and grayscale videos add little difficulty, while random projection (state-based) and pixel Gaussian noise remain challenging.
    \item \textbf{Case study insights} (\autoref{sec:5.2}): Further analysis reveals that SimSR’s success on state-based tasks is largely driven by its use of self-prediction loss and feature normalization. Additionally, applying layer normalization tends to improve both return and denoising across all methods.
    \item \textbf{Isolated setting results} (\autoref{sec:5.3}): When comparing DFs in the isolated setting, we find that the standalone benefit of learning metric by an explicit metric loss becomes marginal.
\end{enumerate}
\endgroup

\vspace{-1em}
\section{Background}
\label{sec:2}
\vspace{-0.5em}

\subsection{Problem Formulation}
\label{sec:2.1}
\vspace{-0.5em}

We consider a setting where observations contain distractions and focus on a special class of Markov decision processes -- exogenous block MDPs (EX-BMDPs)~\citep{efroni2021provable,islam2022agent}. 
This formulation (i) encompasses many distracting environments in prior work, (ii) retains the generality of standard MDPs, (iii) exactly characterizes problems solvable by bisimulation metrics, and (iv) permits concise theoretical analysis and straightforward experiment design.
Before introducing EX-BMDPs, we first define block MDPs as a prerequisite.

\vspace{-0.5em}
\paragraph{Block MDPs~\citep{du2019provably}.}   A block MDP (BMDP) is a tuple $\langle \mathcal{X}, \mathcal{Z}, \mathcal{A}, q, p, R, \gamma \rangle$,
where $\mathcal{X}$ is the observation space, $\mathcal{Z}$ is the latent state space, $\mathcal{A}$ is the action space, $p: \mathcal{Z} \times \mathcal{A} \to \Delta(\mathcal{Z})$ is latent transition function, $R: \mathcal{Z} \times \mathcal{A} \to \mathbb R$ is (latent) reward function, and $\gamma \in [0,1)$ is the discount factor. 
The emission function $q: \mathcal{Z} \to \Delta(\mathcal{X})$ generates observation $x \sim q(\cdot \mid z)$ from latent state $z$. 
Crucially, BMDP assumes the \emph{block structure}:
$
\forall\, z_1, z_2 \in \mathcal{Z},\, z_1 \neq z_2 \implies \supp(q(\cdot \mid z_1)) \cap \supp(q(\cdot \mid z_2)) = \emptyset.
$
This ensures that each observation uniquely determines its latent state, enabling the existence of the \textit{oracle encoder} $q^{-1}: \mathcal{X} \to \mathcal{Z}$ such that $q^{-1}(x) = z$ whenever $x \sim q(\cdot \mid z)$.
The goal of RL in BMDP is to find a policy $\pi: \mathcal X \to \Delta(\mathcal A)$ that maximizes the rewards:
$
\max_\pi \E{\pi}{\sum_{t=0}^\infty \gamma^t R(q^{-1}(x_t), a_t)}$. The policy only receives the observation $x$ without access to the latent state $z$, the latent space $\mathcal Z$, or the oracle encoder $q^{-1}$. 
While the class of BMDPs is equivalent to the class of MDPs~\citep{du2019provably}\footnote{From an MDP perspective, the grounded transition is $\mathcal P(x' \mid x,a) = \sum_{z'\in \mathcal Z} p(z' \mid q^{-1}(x),a) q(x' \mid z') $.}, they capture the underlying state from a high-dimensional observation. However, BMDPs do not differentiate between task-relevant (endogenous) state and task-irrelevant (exogenous) noise in the latent space.
\vspace{-0.5em}
\paragraph{Exogenous BMDPs~\citep{efroni2021provable}.} An EX-BMDP extends BMDP by factorizing a latent state into $z = (s,\xi)$, where $s\in\mathcal{S}$ is the \textit{task-relevant state} and $\xi\in\Xi$ is the \textit{task-irrelevant noise}, representing distraction. The latent state transition $p(s', \xi' \mid s, \xi, a)$ factorizes as $p(s' \mid s, a) p(\xi' \mid \xi)$, where the noise $\xi$ evolves independently and does not affect the reward function. To simplify notation, we denote the reward function as $R(s, a)$.  
EX-BMDPs guarantee the existence of a denoising map  \(D: \mathcal{Z} \to \mathcal{S}\) which extracts the task-relevant state \(s\) from latent state \(z \in \mathcal{Z}\). Combined with the oracle encoder in BMDPs, this enables recovery of the task-relevant state directly from observations:
$
s = D(q^{-1}(x)). 
$
We define this composite function $\phi^* = D \circ q^{-1}$ as the \textbf{oracle encoder} of EX-BMDP.\looseness=-1


\vspace{-0.5em}
\subsection{Representation Learning in RL}
\label{sec:2.2}
\vspace{-0.5em}

In actor-critic methods~\citep{konda1999actor}, representation learning is commonly used to handle complex MDPs such as EX-BMDPs. The idea is to learn an encoder that maps a raw observation to a representation, which is then shared by both actor and critic. Formally, an actor-critic algorithm employs an encoder $\phi:\mathcal{X}\to \Psi $, a (latent) actor $\pi_\theta: \Psi\to \Delta(\A)$, and a (latent) critic $Q_\omega: \Psi \times \A \to \mathbb R$, where $\Psi$ is the representation space. In this work, we focus on end-to-end actor-critic methods based on the soft actor-critic (\textbf{SAC}) algorithm~\citep{haarnoja2018soft}.  These methods jointly optimize the encoder and actor-critic using the RL loss in SAC, denoted as $J_{\text{SAC}}(\phi, \theta, \omega)$. 

Learning state representations solely from reward signals (i.e., RL loss) is challenging in complex tasks. To address this, various state abstraction frameworks and representation objectives have been proposed (see \citet{ni2024bridging} for a literature review). Among these, \textit{model-irrelevance abstraction}~\citep{li2006towards} defines two conditions for an effective encoder using \textit{bisimulation relation}~\citep{givan2003equivalence}. The first condition, known as \textbf{reward prediction (RP)}\footnote{Formally, in an EX-BMDP, RP condition is $\exists R_\kappa: \Psi \times \A \to \mathbb R, \,\text{s.t.} \, R(\phi^*(x), a) = R_\kappa(\phi(x), a), \forall x, a.$}, requires that the representation preserves reward information. The second, \textbf{self-prediction (ZP)}\footnote{Formally, in an EX-BMDP, ZP condition is $\exists P_\nu: \Psi \times \A \to \Delta(\Psi), \, \text{s.t.} \, \mathcal P(\psi' \mid x, a) = P_\nu(\psi' \mid \phi(x), a), \forall x, a, \psi'$, where $\mathcal P(\psi' \mid x, a) 
 = \sum_{x'\in \mathcal X}\mathcal P(x' \mid x, a) \mathbf{1}(\phi(x') = \psi')$.}~\citep{ni2024bridging}, requires that the representation preserves latent dynamics information. 
Model-irrelevance abstraction thus defines compact yet informative encoders that retain essential information for optimal decision-making~\citep{subramanian2022approximate}.
By definition, the RP and ZP conditions hold when $\phi = \phi^*$ and $\Psi = \mathcal S$.\footnote{In this case, $R_\kappa(s,a) = R(s,a)$ and $P_\nu(s' \mid s,a) = p(s'\mid s,a)$.} This implies that the oracle encoder $\phi^*$ serves as a model-irrelevance abstraction.\looseness=-1

To learn a model-irrelevance abstraction, \textbf{DeepMDP}~\citep{gelada2019deepmdp} introduces \textbf{RP and ZP losses} to approximate the RP and ZP conditions, respectively. Given a data tuple $(x, a, r, x')$, these losses jointly optimizes the encoder $\phi$, the reward model $R_\kappa$, and the latent transition model $P_\nu$:
{
  \setlength{\abovedisplayskip}{3pt}
  \setlength{\belowdisplayskip}{3pt}
\begin{align}
\label{eq:ZP_loss}
J_{\text{RP}}(\phi, \kappa) = (R_\kappa(\phi(x), a) - r)^2, \quad J_{\text{ZP}}(\phi, \nu) = - \log P_\nu(\bar\phi(x') \mid \phi(x), a),
\end{align}
}
where $\bar\phi$ detaches the encoder from gradient backpropagation. 
The overall objective $J_{\text{DeepMDP}}(\phi)$ for the encoder in DeepMDP combines SAC loss with RP and ZP losses (\autoref{eq:ZP_loss}, \autoref{fig:deepmdp_arch}).

\vspace{-0.5em}
\section{Conceptual Analysis on Behavioral Metrics Learning in RL}
\label{sec:3}
\vspace{-0.5em}


This section establishes a conceptual framework linking behavioral metrics to representations in deep RL (\autoref{sec:isometric}), and then summarizes how recent work instantiates it (\autoref{sec:3.2}). Please see Appendix \autoref{app:metrics} for background in metrics and metric learning and \autoref{sec:related} for other related work.

\vspace{-0.5em}
\subsection{Isometric Embedding: Between Behavioral Metrics and Representations}
\vspace{-0.5em}
\label{sec:isometric}

We aim to find an encoder that maps noisy observations into a structured representation space, where distances reflect differences in rewards and transition dynamics smoothly. This representation should facilitate RL by ensuring that task-relevant variations are captured.
A natural way to formalize this goal is through the concept of an \textit{isometric embedding (isometry)}\footnote{\url{https://en.wikipedia.org/wiki/Isometry}}: 


\begin{definition}[Isometric Embedding]
An encoder $\phi: \mathcal X \to \Psi$ is an \emph{isometric embedding} if the distances in the original space $(\mathcal X, d_{\mathcal X})$ are preserved in the representation space $(\Psi,d_{\Psi})$. Formally, 
{
  \setlength{\abovedisplayskip}{3pt}
  \setlength{\belowdisplayskip}{3pt}
\begin{equation}
\label{eq:isometric}
    d_{\mathcal X}(x_1, x_2) = d_{\Psi}(\phi(x_1), \phi(x_2)), \quad \forall x_1, x_2 \in \mathcal X,
\end{equation}
}
where $d_{\mathcal X}$ is the \textbf{target metric} (``desired'' metric) and $d_\Psi$ is the \textbf{representational metric}.
\end{definition}
\vspace{-0.5em}



The target metric captures differences in rewards and transition dynamics following a policy $\pi$.  We omit the dependency on $\pi$ for simplicity. 
The target metric is formulated as~\citep{castro2023kernel}:
{
  \setlength{\abovedisplayskip}{3pt}
  \setlength{\belowdisplayskip}{1pt}
\begin{align}
\label{eq:metric_def}
    d_{\mathcal X}(x_1, x_2) &\defeq c_R d_R(x_1, x_2) + c_Td_T(d_{\mathcal X})(\mathcal P(x'\mid x_1),\mathcal P(x'\mid x_2)) \\ 
    &\approx c_R \hat d_R(r_1, r_2) + c_T \hat d_T(\hat d_{\mathcal X})(\hat{\mathcal P}(x'\mid x_1),\hat{\mathcal P}(x'\mid x_2))=\hat d_{\mathcal X}(x_1, x_2), 
\end{align}
}

where $r_1,r_2\in\mathbb R$ are sampled immediate rewards based on $x_1,x_2$ following $\pi$ and $\mathcal P(x'\mid x)$ is a next-state distribution following $\pi$. Here, $d_R$ represents \textit{immediate} state similarity by rewards and $d_T$ is a probabilistic measure of \textit{long-term} state similarity through transition distance, and $\hat{d}_R$ and $\hat{d}_T$ are approximants of $d_R$ and $d_T$.

With isometric embedding assumption, we show the following lemma: for $x_1,x_2\in\mathcal X$,
{
  \setlength{\abovedisplayskip}{3pt}
  \setlength{\belowdisplayskip}{3pt}
\begin{equation}
\label{eq:transition_distance}
d_T(d_{\mathcal X})(\mathcal P(x'\mid x_1),\mathcal P(x'\mid x_2)) = d_T(d_{\Psi})(\mathcal P_\phi(\psi'\mid x_1),\mathcal P_\phi(\psi'\mid x_2)),
\end{equation}
}
where $\mathcal P_\phi(\psi'\mid x) = \sum_{x'}\mathcal P(x'\mid x)\mathbf{1}(\psi'=\phi(x'))$. 
The proof, provided in Appendix~\autoref{app:proof_isometric}, holds for all considered $d_T$. Intuitively, this result shows that isometry \textit{preserves} transition distances in $\mathcal X$ when mapped to $\Psi$, a property implicitly assumed in prior work.
 
\vspace{-0.5em}
\subsection{Design Choices in Behavioral Metric Learning}
\label{sec:3.2}
\vspace{-0.5em}

\begin{table}[ht]
\vspace{-0.5em}
\centering
\caption{\small\textbf{Summary of key implementation choices for the benchmarked methods.}}
\vspace{-0.7em}
\setlength{\tabcolsep}{3pt}
\resizebox{\linewidth}{!}{%
\begin{tabular}{l l l l l l l l l}
\toprule
\multirow{2}{*}{Method} & \multirow{2}{*}{\(\hat{d_R}\)} & \multirow{2}{*}{\(\hat{d_T}\)} & \multirow{2}{*}{\(d_\Psi\)} &
\multirow{2}{*}{\makecell{Metric\\Loss}} &
\multirow{2}{*}{\makecell{Target\\Trick}} &
\multirow{2}{*}{\makecell{Other\\Losses}} &
\multirow{2}{*}{\makecell{Transition\\Model}} &
\multirow{2}{*}{\makecell{Normali\\-zation}} \\
\\
\midrule
\textbf{SAC}~\citep{haarnoja2018soft}
  & --- & --- & --- & --- & --- & --- & --- & --- \\
\textbf{DeepMDP}~\citep{gelada2019deepmdp}
  & --- & --- & --- & --- & --- & RP + ZP & Probabilistic & --- \\
\textbf{DBC}~\citep{zhang2020learning}
  & Huber & $W_2$ closed-form & Huber & MSE & --- & RP + ZP & Probabilistic & --- \\
\textbf{DBC-normed}{\scriptsize~\citep{kemertas2021towards}}
  & Huber & $W_2$ closed-form & Huber & MSE & --- & RP + ZP & Deterministic & $\operatorname{MaxNorm}$ \\
\textbf{MICo}~\citep{castro2021mico}
  & Abs. & Sample-based & Angular & Huber & \(\checkmark\) & --- & --- & --- \\
\textbf{RAP}~\citep{chen2022learning}
  & RAP & $W_2$ closed-form & Angular & Huber & --- & RP + ZP & Probabilistic & --- \\
\textbf{SimSR}~\citep{zang2022simsr}
  & Abs. & Sample-based & Cosine & Huber & --- & ZP & Prob.\ ensemble & $\operatorname{L2Norm}$ \\
\bottomrule
\end{tabular}
}
\vspace{-0.5em}
\label{tab:cand_summary}
\end{table}

\autoref{sec:isometric} provides a general conceptual framework instantiated by several works through distinct design choices. Rather than detailing theoretical differences (Appendix \autoref{app:metric_def}) and design choices defined in papers (Appendix \autoref{app:metric_learning}), we focus on \textit{practical implementations} in their publicly available codebases, summarized in \autoref{tab:cand_summary} and illustrated in Appendix \autoref{app:arch}.\looseness=-1

\vspace{-0.5em}

\paragraph{Choices of Target Metric $d_{\mathcal X}$.}
Methods vary on the choice of $d_{\mathcal X}$ (see Appendix \autoref{app:metric_def})  and $\hat d_{\mathcal X}$ to approximate $d_{\mathcal X}$ -- specifically, $\hat{d_R}$ and $\hat{d_T}$ that approximate $d_R$ and $d_T$.
\vspace{-0.5em}
\begin{itemize}[itemsep=0pt, topsep=0pt, parsep=0pt, partopsep=0pt]
    \item $\hat{d_R}$: MICo and SimSR use absolute difference (``Abs.'' in \autoref{tab:cand_summary}, \autoref{eq:abs_diff}), DBC and DBC-normed use Huber distance (``Huber'' in \autoref{tab:cand_summary}, \autoref{eq:huber_dist}), and RAP has a specific form (Appendix \autoref{app:metric_learning}).
    \item $\hat{d}_T$: To avoid expensive 1-Wasserstein computations in bisimulation metrics~\citep{ferns2004metrics},  DBC, DBC-normed, and RAP approximate $d_T$ using a Gaussian transition model with a 2-Wasserstein metric. In contrast, MICo and SimSR rely on sample-based distance approximations.
\end{itemize}
\vspace{-0.5em}

\paragraph{Choices of Representational Metric $d_\Psi$.}
To approximate $d_\Psi$, DBC and DBC-normed employ a Huber distance (a surrogate for $L_2$ distance); MICo, SimSR, and RAP use an angular distance.

\vspace{-0.5em}
\paragraph{Metric Loss Function $J_M$ and Target Trick.}  To approximate an isometric embedding, metric learning methods optimize this general objective: 
{
  \setlength{\abovedisplayskip}{1pt}
  \setlength{\belowdisplayskip}{1pt}
\begin{equation}
\label{eq:metric_loss}
J_M(\phi) = \ell\left(d_{\Psi}(\phi(x_1), \phi(x_2)) - \hat d_{\mathcal X}(x_1,x_2)\right),
\end{equation}
}
where $\hat d_{\mathcal X}(x_1,x_2) = c_R \hat d_R(r_1, r_2) + c_T \hat d_T(d_{\Psi})(\hat{\mathcal P}(\psi'\mid x_1),\hat{\mathcal P}(\psi'\mid x_2))$ derived by \autoref{eq:metric_def} and \autoref{eq:transition_distance}. Here, $\ell$ is Huber loss~\citep{huber1992robust} in MICo, RAP and SimSR, or mean square error (MSE) in DBC and DBC-normed.
MICo employs a target network $\bar{\phi}$ for encoding one observation in $d_\Psi$ when approximating $d_\mathcal{X}$ to ensure learning stability. See \citet[Appendix C.2]{castro2021mico} for further details.

\vspace{-0.5em}
\paragraph{Self-prediction (ZP) and Reward Prediction (RP) Loss.} 
As discussed, approximating $d_\mathcal{X}$ often requires a transition model, and methods adopt distinct approaches: probabilistic models (DBC), ensembles of probabilistic models (SimSR), and deterministic models (DBC-normed). MICo, in contrast, employs a sample-based target metric that is free of ZP. 
Since all the listed metric learning methods use sampled immediate rewards to approximate $d_R$, an explicit reward model is not strictly necessary. However, following the convention of DeepMDP, in DBC, DBC-normed, and RAP, the RP loss is employed to further shape the representation.

\vspace{-0.5em}
\paragraph{Normalization in the Representation Space $\Psi$.}  DBC-normed employs max normalization to enforce boundedness, leveraging prior knowledge of value range constraints on target metrics. While SimSR requires $L_2$ normalization to enforce unit-length representations, all the other methods use LayerNorm~\citep{ba2016layer} in pixel-based encoders. See Appendix \autoref{app:normalization} for details.



\vspace{-0.5em}
\subsection{Candidate Methods}
\vspace{-0.5em}

We present the design choices of methods to be benchmarked in \autoref{tab:cand_summary}. 
In our experiment (\autoref{sec:exp}),
we follow DBC's implementation on DeepMDP which employs exponential moving average of ZP target~\citep{ni2024bridging} and excludes observation reconstruction loss. 
For DBC-normed, we exclude their additional components related to intrinsic rewards and inverse dynamics. For DBC-normed and SimSR, we replace their original transition models with a single probabilistic transition model. These modifications ensure that our study focuses on the effect of metric learning itself.

\vspace{-0.5em}
\subsection{Why do Metrics (Not) Help with Denoising?}
\vspace{-0.5em}

First, we define \textit{denoising} as a form of generalization in which task-irrelevant noise is removed from observations, allowing a model to generalize across observations with unseen noise. Formally:
\begin{mdframed}[backgroundcolor=yellow!20, linecolor=black, linewidth=1pt]
An encoder $\phi$ is said to achieve \textbf{perfect denoising} in a EX-BMDP if, for any triplet $x, x_+, x_- \in \mathcal X$ such that $\phi^*(x)=\phi^*(x_+) \neq \phi^*(x_-)$, it holds that $\phi(x)=\phi(x_+) \neq \phi(x_-)$. 
That is, $\phi$ replicates the abstraction behavior of the oracle encoder $\phi^*$. 
\end{mdframed}

We then discuss on the connection between denoising and target metric $d_\mathcal{X}$, which motivates further empirical investigation into whether metric learning facilitates denoising.

\vspace{-0.5em}
\paragraph{Metrics potentially help with denoising.} 
Bisimulation metric (BSM, \autoref{def:bisim_metric})~\citep{ferns2004metrics,ferns2011bisimulation} has perfect denoising in a EX-BMDP: for observations $x,x_+\in \mathcal{X}$, $d_\mathcal{X}(x,x_+)=0$ (see 
Appendix \autoref{app:proof_denoising} for a proof). Through \autoref{eq:isometric}, zero $d_{\Psi}$ is ensured and the two observations are assigned to the same representation (Appendix \autoref{app:metric_math}, Metric definition, (1)).
PBSM (\autoref{def:pibisim_metric}) has a denoising property when the policy is \textit{exo-free}~\citep{islam2022agent} (see Appendix
\autoref{app:proof_denoising} for a proof).
Generally, the MICo distance (\autoref{def:mico}) does not assign a zero distance to such $x,x_+$~\citep{castro2021mico} unless both the policy and transition function are deterministic, but empirical evidence indicates its potential to help with denoising~\citep{chen2022learning,zang2022simsr}.


\vspace{-0.5em}
\paragraph{Approximated metrics may not help with denoising.} 
Although BSM has perfect denoising, it is inherently challenging to approximate~\citep{castro2020scalable}. As a result, all of our candidate methods are based on PBSM~\citep{zhang2020learning,kemertas2021towards} or MICo~\citep{castro2021mico,zang2022simsr,chen2022learning}.
However, PBSM does not guarantee denoising observations under arbitrary policies, even when the policy is optimal
(see Appendix \autoref{app:proof_denoising} for a detailed discussion).
Furthermore, several gaps between theory and practice ($d_\mathcal{X}$ and $\hat d_\mathcal{X}$) exacerbate their denoising properties.
Firstly, both PBSM and MICo are on-policy metrics, but the sampled rewards used in  $\hat d_R$ (\autoref{sec:3.2}) are from a replay buffer (\autoref{eq:dbc_loss}), which are off-policy. 
Secondly, the methods use approximated transition models, where an approximation error is introduced (\citet{kemertas2021towards}, Appendix Sec. D).
Thirdly, in the line of work that leverages behavioral metrics in deep RL, the metric loss is not the sole factor shaping the representation. The interplay among the metric loss, ZP loss, and critic loss can lead to undesirable outcomes.

\vspace{-0.5em}
\section{Study Design on Metric Learning: Noise and Denoising}
\label{sec:4}
\vspace{-0.5em}



The ``denoising capability'' of behavioral metric learning is often cited as a motivation in prior work~\citep{zhang2020learning,kemertas2021towards,chen2022learning,zang2022simsr}.  However, most studies evaluate this \textit{indirectly} \textit{in limited settings} by (1) combining metric learning with RL, (2) training only on grayscale natural video backgrounds, (3) testing on unseen videos in training, and (4) evaluating solely through return performance. This leaves a gap between motivation and actual denoising assessment.\looseness=-1

This section bridges that gap with a systematic study design. First, we introduce a diverse range of noise settings from IID Gaussian noise and random projections to natural video backgrounds (\autoref{sec:noises}), enabling an analysis of how noise difficulty impacts metric learning.  Second, we separate the noise distributions during training and testing to examine denoising under both ID and OOD generalization settings (\autoref{sec:id_ood}).
Third, we introduce a direct evaluation measure, the \textit{denoising factor} (\autoref{sec:df}). Finally, to disentangle metric learning from RL, we propose the \textit{isolated metric estimation} setting, where metric learning affects only the encoder, not the RL agent (\autoref{sec:isolated}).

\vspace{-0.5em}
\subsection{Noise Settings}
\label{sec:noises}
\vspace{-0.5em}
We introduce four noise settings under the EX-BMDP framework (\autoref{sec:2.1}), where observations follow $x \sim q(\cdot \mid z)$ with $z = (s, \xi)$, each designed to reflect distinct forms of environmental variation. IID Gaussian noise is applied to both state-based and pixel-based domains to simulate sensor-level randomness. In the state-based setting, we further consider a more challenging variant with random projection, inspired by information security~\citep{dwork2006differential,gentry2009fully},  where data is privacy-preserving yet remains recoverable via decryption. Natural image and video settings apply only to pixel-based domains, simulating real-world background shifts in visual environments. The grayscale setting is widely adopted in metric learning~\citep{zhang2020learning,kemertas2021towards,zang2022simsr,chen2022learning}.\footnote{Some prior work present backgrounds in color but actually use grayscale in experiments, which may cause confusion.} An illustration of the pixel-based noise settings is shown in \autoref{fig:pixel_noise}.\looseness=-1


\vspace{-0.5em}
\paragraph{IID Gaussian Noise.} The task-irrelevant noise $\xi_t$ is sampled independently at each timestep from an $m$-dimensional isotropic Gaussian, $\xi_t \sim \mathcal N(\boldsymbol{\mu}, \sigma^2 \mathbf{I})$. 
For state-based domains, the observation is exactly the latent state, i.e.,  $x_t=z_t$ with $q$ as the identity mapping. We adjust the noise dimension $m$ or noise std $\sigma$ to modulate difficulty, whereas prior work~\citep{kemertas2021towards,ni2024bridging} only varies $m$ with a small $\sigma$. 
For pixel-based domains, noise is applied per pixel in the background and overlaid by the robot's foreground pixels, with $q$ as a rendering function.

\vspace{-0.5em}
\paragraph{IID Gaussian Noise with Random Projection.} 
It applies only to state-based domains where $s\in \mathbb R^n$. During initialization, a full-rank square matrix $\mathbf{A} \in \mathbb{R}^{(n+m) \times (n+m)}$ is constructed with entries sampled as $A_{ij} \sim \mathcal N(\mu_A, \sigma_A^2)$. At each time step, we generate $m$-dimensional IID Gaussian noise $\xi_t \sim \mathcal N(\boldsymbol{\mu}, \sigma^2 \mathbf{I})$ and then apply a linear projection to obtain observation $x_t = \mathbf{A} z_t$ where $z_t = (s_t, \xi_t)$. 
Since $\mathbf{A}$ is of full rank, $s_t$ can be recovered from $x_t$ using $\mathbf{A}^{-1}$. This setting is more challenging than IID Gaussian noise, as it linearly entangles $s_t$ and $\xi_t$, with $q$ as the linear projection.

\vspace{-0.5em}
\paragraph{Natural Images.} This setting applies only to pixel-based domains, replacing the clean background with a randomly selected natural image. As in the original environment, the background remains fixed during training. Images can be \textit{grayscale} or \textit{colored}, introducing different levels of visual complexity. In EX-BMDP notation, $\xi_t$ is stationary, and $q$ is a rendering function. 

\vspace{-0.5em}
\paragraph{Natural Videos.} This pixel-based noise setting replaces the clean background with \textit{grayscale} or \textit{colored} natural videos. The underlying noise $\xi_{t} \in \mathbb N$, representing the \textit{frame index}, follows the update rule $\xi_{t} = (\xi_{t-1} + 1)\mod N$, where $N$ is the total number of frames.

\vspace{-0.5em}
\subsection{Denoising Involves ID and OOD Generalization}
\label{sec:id_ood}
\vspace{-0.5em}

The evaluation settings differ based on whether the \textit{noise distribution} remains unchanged or shifts between training and testing.
In \textbf{in-distribution (ID) generalization evaluation} setting, the training and testing environments (EX-BMDPs) are identical, meaning the same noise distribution is applied in both phases. For example, IID Gaussian noise remains unchanged throughout training and testing.
For the \textbf{out-of-distribution (OOD) generalization evaluation setting}, the training and testing EX-BMDPs share the same task-relevant parts (i.e., $p(s'\mid s,a)$, $p(s_0)$, $\mathcal R(s,a)$) but differ in noise distributions (i.e., $p(\xi'\mid \xi)$, $p(\xi_0)$). For instance, natural videos from a training dataset are employed during training, while videos from a distinct test dataset are used during evaluation. This OOD evaluation setup is widely used in metric learning~\citep{zhang2020learning,kemertas2021towards,zang2022simsr,chen2022learning}. 

\vspace{-0.5em}
\subsection{Quantifying Denoising via the Denoising Factor}
\label{sec:df}
\vspace{-0.5em}

We introduce the \textit{denoising factor} (DF), a measure that quantifies an encoder $\phi$'s ability to filter out irrelevant details while retaining essential information.\footnote{While the oracle encoder $\phi^*$ achieves perfect denoising, direct comparison is impossible as $\phi$ lacks access to $\mathcal S$.} 
It also provides insight into how the behavioral metrics are approximated, given that exact behavioral metrics are nearly inaccessible via fixed-point iteration in high-dimensional state or action spaces.
To compute DF, we define a \textit{positive score} and a \textit{negative score} for an encoder $\phi$. Inspired by triplet loss~\citep{schroff2015facenet} in contrastive learning, we compute these scores by selecting an observation $x$ as an \textit{anchor} under a policy $\pi$, then constructing a \textit{positive example} $x_+$ that shares the same task-relevant state, i.e., $\phi^*(x) = \phi^*(x_+)$ (implying $x$ and $x_+$ are bisimilar), 
and a \textit{negative example} $x_-$ for which this equality does not necessarily hold.
We define DF under a specific policy $\pi$ because an agent may lack access to anchors from other policies, and has no reason to denoise observations outside its training distribution.\looseness=-1



\begin{definition}[Positive score] The positive score of an encoder $\phi$ w.r.t. the metric $d_\Psi$ measures the average representational distance between anchors and their positive examples:
{
  \setlength{\abovedisplayskip}{2pt}
  \setlength{\belowdisplayskip}{3pt}
\begin{equation}
\label{eq:pos_score}
\mathrm{Pos}^\pi_{d_\Psi}(\phi) \defeq \E{x\sim \rho_\pi(x),\xi_+ \sim \rho(\xi_+), x_+ \sim q(\cdot\mid \phi^*(x), \xi_+) }{d_\Psi(\phi(x), \phi(x_+))},
\end{equation}
}
where $\rho_\pi(x)$ is the stationary state distribution under the policy $\pi$ and $\rho(\xi_+)$ is a stationary noise distribution. The sampling $x_+ \sim q(\cdot\mid \phi^*(x), \xi_+)$ ensures that $x_+$ shares the same task-relevant state $s=\phi^*(x)$ but has different noise $\xi_+$. 
\end{definition}
\vspace{-0.5em}
In the temporally-independent noise setting, $\rho(\xi_+)$ matches the noise transition; in the natural-video setting, $\rho(\xi_+)$ is a uniform distribution over frame indices $\{0,1, \dots, N-1\}$. 

\begin{definition}[Negative score] The negative score of an encoder $\phi$ w.r.t. the metric $d_\Psi$ measures the average representational distance between anchors and their negative examples (IID sampled):
{
  \setlength{\abovedisplayskip}{2pt}
  \setlength{\belowdisplayskip}{2pt}
\begin{equation}
\label{eq:neg_score}
\mathrm{Neg}^\pi_{d_\Psi}(\phi) \defeq \E{x,x_-\stackrel{\mathrm{IID}}{\sim} \rho_\pi}{d_\Psi(\phi(x), \phi(x_-))}.
\end{equation}
}
\end{definition}

\begin{definition}[Denoising factor (DF)] 
\label{def:DF}
The denoising factor of an encoder $\phi$ w.r.t. the metric $d_\Psi$ is defined as the normalized difference between the negative and positive scores:
{
  \setlength{\abovedisplayskip}{3pt}
  \setlength{\belowdisplayskip}{3pt}
\begin{equation}
\label{eq:df}
    \mathrm{DF}^\pi_{d_\Psi}(\phi) \defeq \frac{\mathrm{Neg}^\pi_{d_\Psi}(\phi)-\mathrm{Pos}^\pi_{d_\Psi}(\phi)}{\mathrm{Neg}^\pi_{d_\Psi}(\phi)+\mathrm{Pos}^\pi_{d_\Psi}(\phi)} \in [-1, 1]. 
\end{equation}
}
\end{definition}
\vspace{-0.5em}
DF measures denoising, with values above $0$ indicating smaller distances for positive over negative pairs; higher values imply better denoising. The oracle encoder $\phi^*$ attains the maximum of $1$.\looseness=-1


\vspace{-0.5em}
\subsection{Decoupling Metric Learning from RL for Denoising Evaluation}
\label{sec:isolated}
\vspace{-0.5em}

In many behavioral metric learning methods, the encoder $\phi$ is optimized via a combination of losses: the RL loss (e.g., $J_{\text{SAC}}(\phi)$), the reward-prediction loss $J_{\text{RP}}(\phi)$, the self-prediction loss $J_{\text{ZP}}(\phi)$ (\autoref{eq:ZP_loss}), and a metric loss $J_{\text{M}}(\phi)$ (\autoref{eq:metric_loss}). 
This coupling makes it difficult to isolate the \textit{direct impact of metric learning} on representation quality.
Moreover, denoising factor (DF, \autoref{def:DF}) depends on both an encoder and a policy. 
Although DFs under different policies may offer initial quantitative insights, such comparisons are not rigorous, as each reflects denoising ability on policy-specific data. 
Notably, policies that frequently revisit similar task-relevant states under varying noise can significantly inflate DF.
Due to the above reasons, we propose to evaluate behavioral metric learning algorithms in an \textit{isolated metric estimation} setting.\looseness=-1

\vspace{-0.5em}
\paragraph{Isolated Metric Estimation Setting.} To isolate the effect of metric learning, we introduce an isolated \textit{metric encoder} $\tilde{\phi}$ that is optimized solely via the metric loss $J_{\text{M}}(\tilde{\phi})$, while the \textit{agent encoder} $\phi$ is updated using the RL objectives (\eg, $J_{\text{SAC}}(\phi)$ or $J_{\text{DeepMDP}}(\phi)$). 
In our experiments, regardless of the metric learning method, a SAC agent interacts with the environment and collects data for learning the metrics (illustration see \autoref{fig:iso_arch}).
This allows for a fair comparison of  $\mathrm{DF}^\pi_{d_\Psi}(\tilde{\phi})$ across different metric learning methods.
For methods that rely on self-prediction loss~\citep{zhang2020learning,kemertas2021towards,zang2022simsr}, we learn an \textit{isolated transition model} using $\tilde{\phi}$ while preventing gradient backpropagation to $\tilde{\phi}$ to ensure isolation. This setting can be naturally extended to cases where $\tilde{\phi}$ is optimized by a different combination of objectives than those used to optimize $\phi$, for example, using $J_{\text{RP}}$ and $J_{\text{ZP}}$ to optimize $\tilde{\phi}$ while using $J_{\text{SAC}}$ to optimize $\phi$.

\vspace{-0.5em}
\section{Experiments}
\label{sec:exp}
\vspace{-0.5em}

\paragraph{Experiment Organization.} We first conduct a comprehensive evaluation of all the methods (\autoref{tab:cand_summary}) across \textbf{20 state-based} DeepMind Control (DMC)  \citep{tassa2018deepmind,tunyasuvunakool2020dm_control} tasks (listed in \autoref{tab:state_levels}) and \textbf{14 pixel-based} DMC tasks (listed in \autoref{tab:pixel_levels}), under various noise settings with ID generalization evaluation.
Our results (\autoref{sec:5.1}) offer a broad assessment of agent performance and task difficulty across a significantly larger set of tasks and noise settings than prior work.
Based on these findings, we select a subset of representative tasks for case studies (\autoref{sec:5.2}) to identify key design choices (\autoref{sec:3.2}), and further investigate the isolated metric estimation setting (\autoref{sec:isolated}) in \autoref{sec:5.3}. OOD generalization following prior work is assessed in \autoref{sec:5.4}. 


\vspace{-0.5em}
\paragraph{Evaluation Protocol.} 
For aggregated scores, we report the \textit{mean episodic reward} rather than the IQM~\citep{agarwal2021deep} to avoid ignoring tasks that are too easy or too challenging. 
Corresponding per-task results for all aggregated scores in the main text are provided in Appendix~\autoref{app:addition_exp}.
In our tables, each run's mean episodic reward, bounded within $[0,1000]$, is computed as the average of 10 evaluation points collected between 1.95M and 2.05M steps, and then aggregated over seeds.
All figures and tables display $95\%$ confidence intervals (CIs) across tasks. 
We use 12 seeds for state-based and 5 seeds for pixel-based environments per task-noise combination.

\vspace{-0.5em}
\paragraph{Hyperparameters.} We adopt the original hyperparameters from the referenced implementations (details and exceptions see Appendix \autoref{app:hparams}). 
The hyperparameters are widely used and considered well-tuned for pixel-based DMC with grayscale and clean backgrounds. 
Our settings retain the same task-relevant state, varying only the noise type, which justifies this choice.
Following DBC-normed, we use the same key hyperparameters for both pixel- and state-based tasks.

\vspace{-0.5em}
\paragraph{Approximation of DF (\autoref{eq:df}).}
Observations collected in the evaluation stage are considered \textit{anchors}. We sample $16$ positive and $16$ negative examples for each anchor using the strategy in \autoref{sec:df}. 
We report $\mathrm{DF}^\pi_{||\cdot||_2}(\phi)$, where $\pi$ denotes the agent’s policy up to each evaluation point.

\vspace{-0.5em}
\subsection{Benchmarking Methods on Various Noise Settings}
\label{sec:5.1}
\vspace{-0.5em}

\paragraph{Settings.}
For state-based DMC tasks, we apply IID Gaussian noise ($\boldsymbol{\mu}=\boldsymbol{0}$), varying either (a) standard deviations $\sigma \in \{0.2,1.0,2.0,4.0,8.0\}$ (with a fixed $m=32$), or (b) noise dimensions $m\in\{2,16,32,64,128\}$ (with a fixed $\sigma=1.0$). 
For pixel-based DMC tasks, evaluation is conducted under \textbf{6 image background settings}: (1) clean background (the original pixel-based DMC setting), (2) grayscale natural images, (3) colored natural images, (4) grayscale natural videos, (5) colored natural videos, and (6) IID Gaussian noise (with $\sigma=1.0$). ID generalization evaluation is conducted in this subsection.
The aggregated reward and DF for settings (a), (b), and (1)-(6) are shown in \autoref{fig:5.1bar} and \autoref{fig:df_sensitivity},
respectively.
Per-task results are listed in Appendix \autoref{app:addition_exp}.\looseness=-1 

\vspace{-0.5em}
\paragraph{Implementation Details.} For state-based tasks, the encoder is a three-layer MLP, as used by SAC and DBC-normed. For pixel-based tasks, the encoder is a CNN followed by LayerNorm~\citep{ba2016layer}, as used by SAC-AE~\citep{yarats2021improving}. 
All the compared methods are implemented based on SAC. 
For a fair comparison, we adopt identical probabilistic latent transition models and reward models used in DBC and DBC-normed if applicable.

\begin{figure}[t]
\vspace{-2em}
    \centering
    \includegraphics[width=0.9\linewidth]{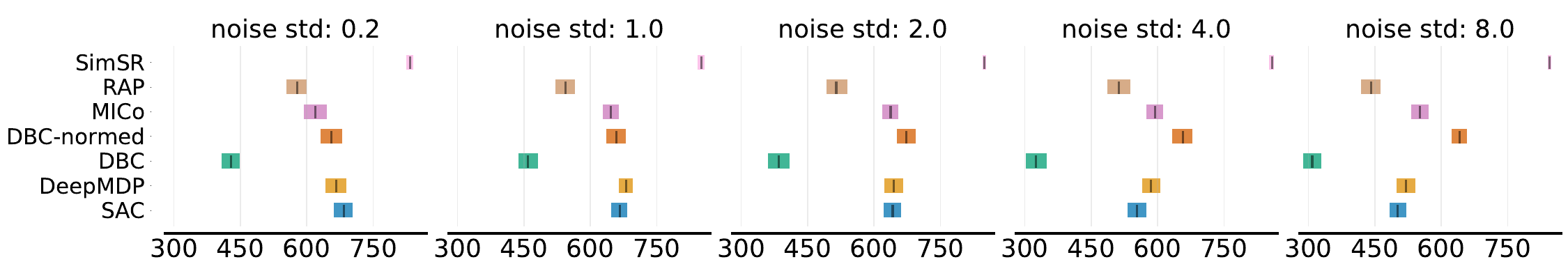}
    \vspace{0.1cm}
    \includegraphics[width=0.9\linewidth]{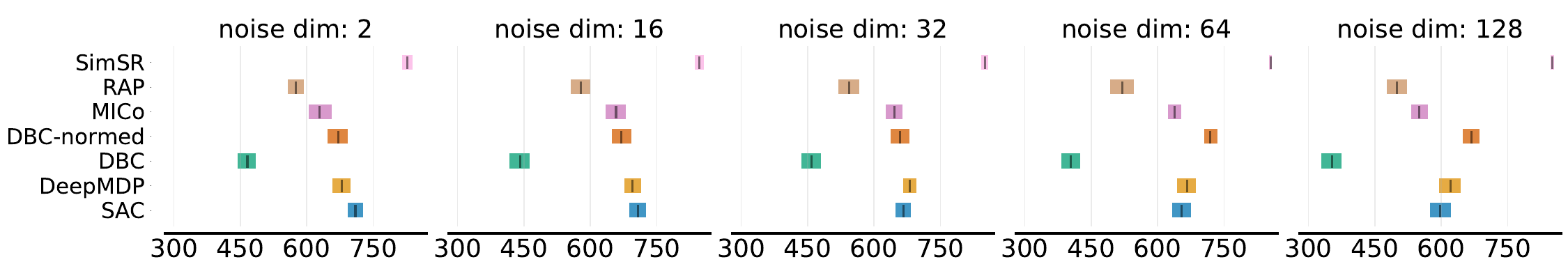}
    \vspace{0.1cm}
    \includegraphics[width=0.9\linewidth]{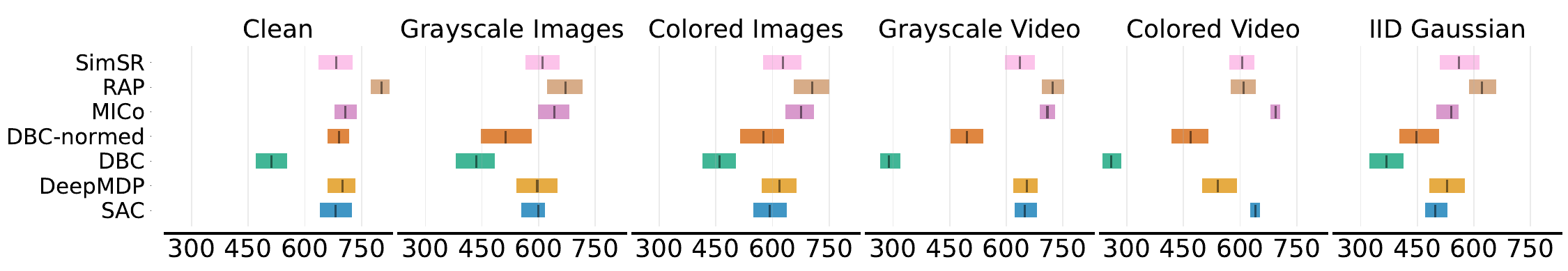}
    \vspace{-1em}
    \caption{\small\textbf{Benchmarking results: performance of seven methods across diverse noise settings}, aggregating episodic rewards from \textbf{20 state-based} (first two rows) and \textbf{14 pixel-based} tasks (last row). ``Noise std'' denotes the IID Gaussian noise's standard deviation $\sigma$, while ``noise dim'' denotes its dimension $m$. Bars show $95\%$ CI.}
    \label{fig:5.1bar}
    \vspace{-2em}
\end{figure}



\vspace{-0.5em}
\paragraph{Benchmarking Findings.} We summarize the key findings from  \autoref{fig:5.1bar} and \autoref{tab:model_update_time}.
\begin{itemize}[itemsep=1.5pt, topsep=0pt, parsep=0pt, partopsep=0pt]
    \item \textbf{SimSR} consistently achieves the highest performance in most state-based tasks, excelling in both return and DF. \textbf{RAP} performs best in most pixel-based tasks but suffers a moderate performance drop in state-based tasks. Interestingly, both SimSR and RAP were evaluated only in pixel-based domains in their papers, making our state-based findings novel.
    \item \textbf{SAC and DeepMDP}, although not metric learning methods,  deliver decent performance on both pixel-based and state-based tasks,  but are often overlooked in prior work. Conversely, \textbf{DBC}, a commonly used metric learning baseline, consistently performs the worst among all methods.
    \item Within the ranges we tested in state‐based tasks, both increasing the number of noise dimensions (at fixed $\sigma = 1.0$) and the noise standard deviation (at fixed $m = 32$) causes moderate reward drops.
    Well-performing methods remain robust to noise variation in both reward and DF (\autoref{fig:df_sensitivity}).\looseness=-1
    \item In pixel-based domains, \textit{grayscale natural video}, widely used in prior work, is not significantly harder than the clean background setting (e.g., for SAC and DeepMDP). Surprisingly, the \textit{IID Gaussian noise} setting is the most challenging, warranting further study.
    \item Different algorithms excel in different tasks (\autoref{tab:state_levels}, \autoref{tab:pixel_levels}), e.g., RAP in reacher/easy, MICo in point\_mass/easy (\autoref{fig:all_video}). Broad task coverage is essential to ensure generalizable insights.
    \item Adding objectives trades off computation efficiency. As shown in \autoref{tab:model_update_time}, the time cost of optimizing a metric loss is close to optimizing a ZP loss by comparing MICo with DeepMDP.
\end{itemize}

\begin{table}[t]
\vspace{-2em}
\centering
\caption{ \small
  Performance comparison without ($R$) and with LayerNorm ($R'$).
  The cell backgrounds in $R'$ rows reflect $R' - R$: 
  \colorbox{red!20}{\textbf{red}} if $R'-R > 0$,
  \colorbox{blue!20}{\textbf{blue}} if $R'-R < 0$, with darker shades for larger magnitude.
}
\vspace{-0.5em}
\begin{adjustbox}{max width=0.9\linewidth}
\begin{tabular}{llccccccc}
\toprule
 & & \multicolumn{7}{c}{\textbf{Methods}} \\
\cmidrule(lr){3-9}
\textbf{Task} & 
  & \textbf{SAC} & \textbf{DeepMDP} & \textbf{DBC} & \textbf{DBC-normed} & \textbf{MICo} & \textbf{RAP} & \textbf{SimSR} \\
\midrule
\multirow{2}{*}{\textbf{cartpole/balance}}
 & $R$ 
   & 967.5{\footnotesize $\pm$12.3}
   & 928.7{\footnotesize $\pm$32.3}
   & 814.1{\footnotesize $\pm$86.6}
   & 973.7{\footnotesize $\pm$12.4}
   & 966.6{\footnotesize $\pm$9.2}
   & 950.3{\footnotesize $\pm$71.2}
   & 999.5{\footnotesize $\pm$0.5} \\
 & $R'$
   & \cellcolor{red!6}   979.5{\footnotesize $\pm$20.1}
   & \cellcolor{red!15}  994.6{\footnotesize $\pm$3.6}
   & \cellcolor{red!25}  943.6{\footnotesize $\pm$24.1}
   & \cellcolor{red!3}   975.5{\footnotesize $\pm$19.9}
   & \cellcolor{blue!15} 936.1{\footnotesize $\pm$29.8}
   & \cellcolor{red!15}  981.7{\footnotesize $\pm$19.1}
   & \cellcolor{blue!8}  980.2{\footnotesize $\pm$19.3} \\
\midrule

\multirow{2}{*}{\textbf{finger/turn\_easy}}
 & $R$ 
   & 592.9{\footnotesize $\pm$176.6}
   & 327.3{\footnotesize $\pm$88.5}
   & 201.9{\footnotesize $\pm$38.5}
   & 619.0{\footnotesize $\pm$35.1}
   & 419.0{\footnotesize $\pm$75.9}
   & 240.6{\footnotesize $\pm$36.4}
   & 926.8{\footnotesize $\pm$10.9} \\
 & $R'$
   & \cellcolor{red!30} 770.6{\footnotesize $\pm$65.8}
   & \cellcolor{red!50} 955.0{\footnotesize $\pm$7.1}
   & \cellcolor{blue!5} 193.7{\footnotesize $\pm$22.2}
   & \cellcolor{blue!15} 577.5{\footnotesize $\pm$33.7}
   & \cellcolor{red!40} 745.3{\footnotesize $\pm$47.6}
   & \cellcolor{red!30} 412.8{\footnotesize $\pm$39.3}
   & \cellcolor{red!5}  934.6{\footnotesize $\pm$16.0} \\
\midrule

\multirow{2}{*}{\textbf{walker/run}}
 & $R$ 
   & 635.3{\footnotesize $\pm$19.8}
   & 347.8{\footnotesize $\pm$84.0}
   & 23.9{\footnotesize $\pm$2.6}
   & 628.9{\footnotesize $\pm$25.7}
   & 455.9{\footnotesize $\pm$41.3}
   & 649.4{\footnotesize $\pm$11.1}
   & 760.6{\footnotesize $\pm$19.4} \\
 & $R'$
   & \cellcolor{blue!25} 534.5{\footnotesize $\pm$53.6}
   & \cellcolor{red!45}  776.0{\footnotesize $\pm$5.9}
   & \cellcolor{red!35}  342.9{\footnotesize $\pm$54.5}
   & \cellcolor{red!25}  759.8{\footnotesize $\pm$19.4}
   & \cellcolor{red!28}  611.0{\footnotesize $\pm$22.5}
   & \cellcolor{red!6}   661.6{\footnotesize $\pm$88.4}
   & \cellcolor{red!3}   761.6{\footnotesize $\pm$20.0} \\
\midrule

\multirow{2}{*}{\textbf{quadruped/run}}
 & $R$ 
   & 233.8{\footnotesize $\pm$59.0}
   & 381.1{\footnotesize $\pm$64.9}
   & 219.5{\footnotesize $\pm$63.5}
   & 433.3{\footnotesize $\pm$47.3}
   & 417.9{\footnotesize $\pm$44.2}
   & 441.1{\footnotesize $\pm$93.7}
   & 847.4{\footnotesize $\pm$21.7} \\
 & $R'$
   & \cellcolor{red!35} 483.8{\footnotesize $\pm$6.0}
   & \cellcolor{red!45} 891.1{\footnotesize $\pm$17.8}
   & \cellcolor{red!18} 291.3{\footnotesize $\pm$55.0}
   & \cellcolor{red!20} 509.5{\footnotesize $\pm$35.4}
   & \cellcolor{red!15} 467.4{\footnotesize $\pm$21.8}
   & \cellcolor{red!35} 687.3{\footnotesize $\pm$59.8}
   & \cellcolor{blue!8} 832.9{\footnotesize $\pm$63.4} \\
\midrule

\multirow{2}{*}{\textbf{finger/turn\_hard}}
 & $R$ 
   & 177.6{\footnotesize $\pm$66.1}
   & 168.3{\footnotesize $\pm$50.4}
   & 97.9{\footnotesize $\pm$11.8}
   & 414.7{\footnotesize $\pm$49.5}
   & 207.2{\footnotesize $\pm$53.8}
   & 110.8{\footnotesize $\pm$17.0}
   & 885.4{\footnotesize $\pm$24.5} \\
 & $R'$
   & \cellcolor{red!35} 495.7{\footnotesize $\pm$53.1}
   & \cellcolor{red!60} 925.8{\footnotesize $\pm$14.7}
   & \cellcolor{blue!5} 95.9{\footnotesize $\pm$12.4}
   & \cellcolor{red!18} 473.4{\footnotesize $\pm$39.9}
   & \cellcolor{red!25} 335.1{\footnotesize $\pm$42.6}
   & \cellcolor{red!20} 201.1{\footnotesize $\pm$26.3}
   & \cellcolor{red!15} 917.1{\footnotesize $\pm$13.9} \\
\midrule

\multirow{2}{*}{\textbf{hopper/hop}}
 & $R$ 
   & 0.1{\footnotesize $\pm$0.0}
   & 31.3{\footnotesize $\pm$16.7}
   & 0.3{\footnotesize $\pm$0.3}
   & 51.1{\footnotesize $\pm$13.4}
   & 0.4{\footnotesize $\pm$0.3}
   & 0.8{\footnotesize $\pm$0.5}
   & 233.9{\footnotesize $\pm$22.6} \\
 & $R'$
   & \cellcolor{red!6}  12.4{\footnotesize $\pm$4.9}
   & \cellcolor{red!30} 195.4{\footnotesize $\pm$19.9}
   & \cellcolor{red!4}  6.2{\footnotesize $\pm$4.8}
   & \cellcolor{red!18} 125.8{\footnotesize $\pm$22.3}
   & \cellcolor{red!3}  1.8{\footnotesize $\pm$2.0}
   & \cellcolor{red!2}  1.0{\footnotesize $\pm$0.3}
   & \cellcolor{blue!12} 207.4{\footnotesize $\pm$36.4} \\
\bottomrule
\end{tabular}
\end{adjustbox}
\label{tab:5.2ln}
\vspace{-1.5em}
\end{table}

\vspace{-0.5em}
\subsection{What Matters in Metric and Representation Learning?}
\label{sec:5.2}
\vspace{-0.5em}


To identify key factors in metric learning, we conduct case studies on the design choices outlined in \autoref{sec:3.2}. Six \textit{easy-to-hard} state-based DMC tasks (see  \autoref{tab:5.2ln}) are selected
for detailed analysis.

\vspace{-0.5em}
\paragraph{Case studies design.} First, a notable difference between our default encoder implementations for state-based and pixel-based tasks is the inclusion of normalization, which may significantly impact benchmarking outcomes.
SimSR, the best-performing algorithm in state-based environments, employs $L_2$ normalization in the representation space and discusses its effectiveness~\citep{zang2022simsr}.
This inspires us to examine whether normalization benefits \textit{other} metric learning methods.
Considering the target metric \textit{misspecification} issue (see Appendix \autoref{app:normalization}), we examine the effect of incorporating LayerNorm on the representation space rather than using $L_2$ normalization in methods that do not inherently require it.
Second, several techniques used by the best-performing methods merit further analysis. Specifically, SimSR, RAP, and MICo (which excels in \textit{colored natural video} settings) utilize Huber metric loss instead of MSE, while MICo incorporates the target trick (\autoref{sec:3.2}).
To evaluate the effectiveness of these techniques, we apply them to DBC-normed, which originally does not include any of these modifications.
Third, we investigate the performance of methods \textit{with LayerNorm} in a challenging setting: \textit{IID Gaussian noise with random projection} (\autoref{sec:noises}) with $\sigma \in \{0.2,1.0, 2.0,4.0,8.0\}$ (with a fixed noise dimension $m=32$), shown in \autoref{fig:5.2RP}.
Important findings in \autoref{tab:5.2ln}, \autoref{fig:5.2DBC-normed}, \autoref{fig:5.2simsr}, and \autoref{fig:5.2RP} are as follows:\looseness=-1

\begin{figure}[t]
\vspace{-2em}
    \centering

    \includegraphics[width=0.62\linewidth]{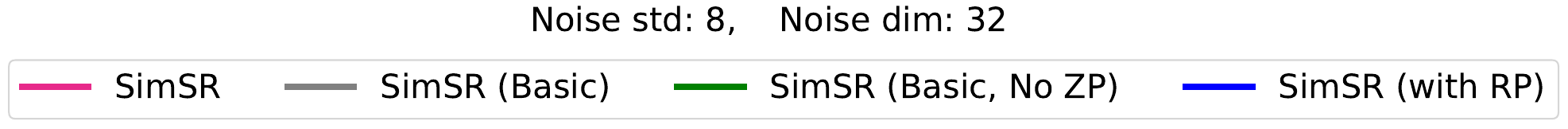}
    \vspace{-0.1cm}
    \includegraphics[width=\linewidth]{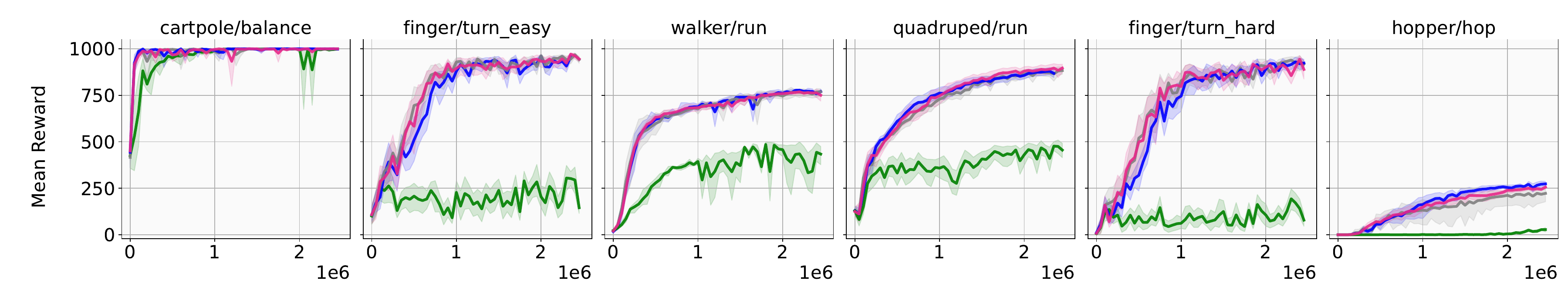}
    \includegraphics[width=\linewidth]{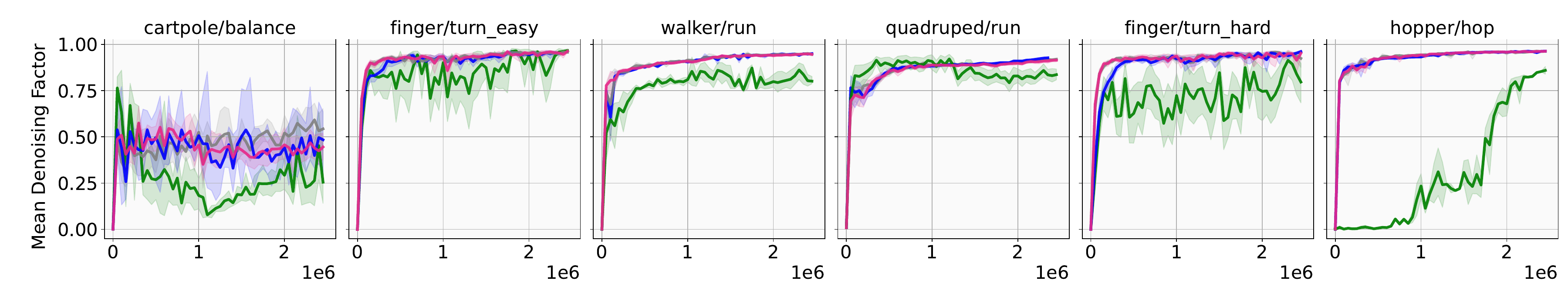}
    \vspace{-2em}
   \caption{\small\textbf{Ablation study on ZP loss on SimSR.} 
   ``SimSR'' is the agent benchmarked in \autoref{sec:5.1}, where ZP is integral to the metric estimation. Therefore, we resort to
   ``SimSR (Basic)'' setting (Theorem 2, ~\citet{zang2022simsr}), where ZP is independent of the metric estimation, and ``SimSR (Basic, No ZP)'' is the setting that ZP is detached from SimSR (Basic). ``SimSR (with RP)'' adds RP loss to original SimSR.
   This ablation highlights the impact of detaching ZP on the overall performance.
   X-axis stands for the environmental step.
   }
    \vspace{-1em}
    \label{fig:5.2simsr}
\end{figure} 

\begin{itemize}[itemsep=1.5pt, topsep=0pt, parsep=0pt, partopsep=0pt]
    \item \textbf{Most methods benefit from LayerNorm in the representation space}, improving both reward (\autoref{tab:5.2ln}) and DF (\autoref{fig:5.2ln_huber}). Notably, \textbf{DeepMDP with LayerNorm performs comparably to SimSR}.\footnote{Our additional experiments reveal that removing LayerNorm in pixel-based environments causes a substantial performance drop across all methods, highlighting the critical role of normalization.} For DBC-normed, employing Huber loss for the metric and incorporating the target trick yield a modest performance improvement (\autoref{fig:5.2DBC-normed}).
    \item \textbf{ZP loss is crucial for SimSR's success} in noisy state-based tasks (\autoref{fig:5.2simsr}).\looseness=-1
    \item A significant performance drop occurs for all agents when increasing the noise standard deviation in \textbf{IID Gaussian with random projection} setting (\autoref{fig:5.2RP}, \autoref{fig:rp_sensitivity}), even with LayerNorm applied. Nevertheless, DeepMDP and SimSR remain relatively robust to the noise.\looseness=-1
\end{itemize}

\begin{figure}[ht]
    \centering
    \includegraphics[width=0.9\linewidth]{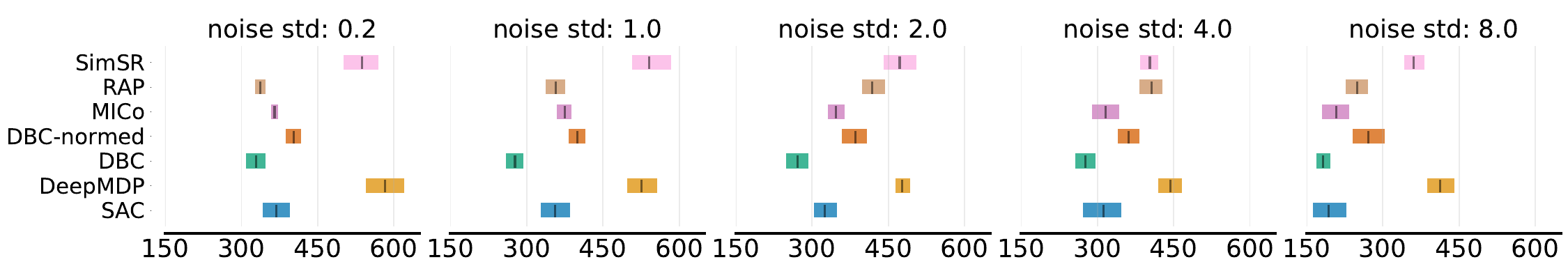}
    \vspace{0.1cm}
    \includegraphics[width=0.9\linewidth]{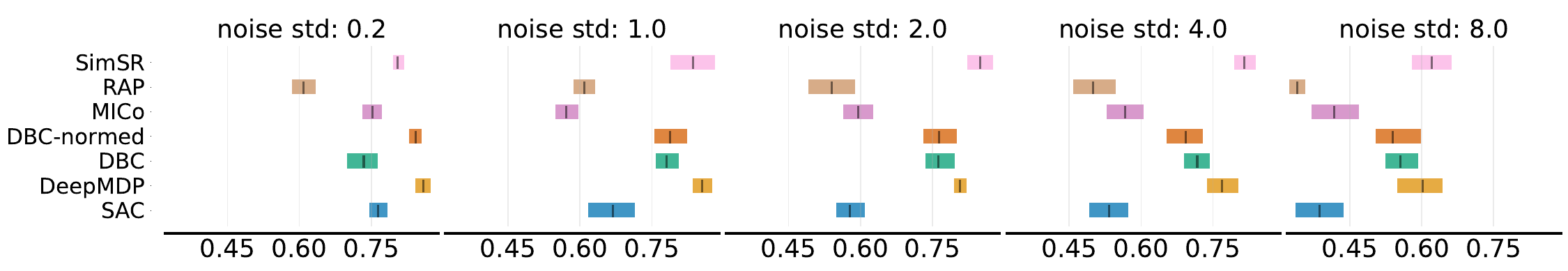}
    \vspace{-1em}
    \caption{\small \textbf{Aggregated reward (top row) and DF (bottom row)} of seven agents on \textbf{IID Gaussian with random projection} settings, varying noise standard deviation, in the 6 selected state-based tasks.}
    \label{fig:5.2RP}
    \vspace{-1em}
\end{figure}

 

\vspace{-0.5em}
\subsection{Isolated Metric Estimation Setting: Does Metric Learning Help with Denoising?}
\label{sec:5.3}
\vspace{-0.5em}
We further analyze the proposed isolated setting (\autoref{sec:isolated}) in the 6 selected tasks (\autoref{fig:5.3compared}).
The experiments include:
(i) isolated metric estimation for metric learning methods, SAC, and DeepMDP where $\tilde{\phi}$ is optimized using either the Q loss, ZP loss alone (i.e., the minimalist algorithm~\citep{ni2024bridging}), or both RP and ZP losses.
(ii) The same as (i) but with LayerNorm applied to $\tilde{\phi}$.
(iii) Building on (ii), additionally applying ZP loss to $\tilde{\phi}$ for all metric learning methods. 
All evaluations are conducted on ID generalization using SAC with LayerNorm as the base agent.
We observe that:
\vspace{-0.5em}
\begin{itemize}[itemsep=1.5pt, topsep=0pt, parsep=0pt, partopsep=0pt]
    \item Generally, metrics learned in isolation provide some denoising but underperform relative to optimizing the ZP loss on $\tilde{\phi}$ (i.e., DeepMDP (w/o RP) in \autoref{fig:5.3compared}, middle row). Including an additional RP loss to shape $\tilde{\phi}$ (as in DeepMDP) yields limited DF improvement to the ZP-only setting.
    \item Applying LayerNorm to the isolated encoder $\tilde{\phi}$ substantially improves DF for DeepMDP (with or w/o RP), but offers only modest gains for metric learning methods (top and middle rows of \autoref{fig:5.3compared}).\looseness=-1
    \item Adding metric losses to ZP loss on $\tilde{\phi}$ generally does not improve DF (\autoref{fig:5.3compared}, mid and bottom rows).\looseness=-1
    \item MICo’s DF remains relatively low, which aligns with its theoretical property that the metric for positive examples is non-zero (\autoref{def:mico}), as MICo does not enforce zero self-distance.
\end{itemize} 

\begin{figure}[t]
     \vspace{-2em}
    \centering
    \includegraphics[width=0.8\linewidth]{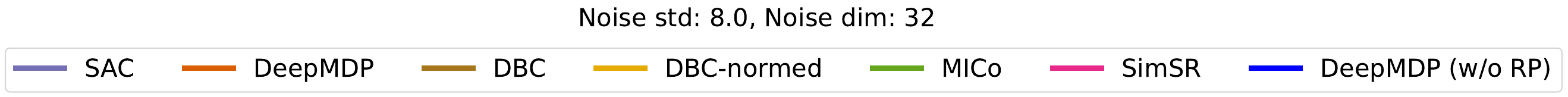}
     \vspace{-0.3cm}
     \includegraphics[width=1\linewidth]{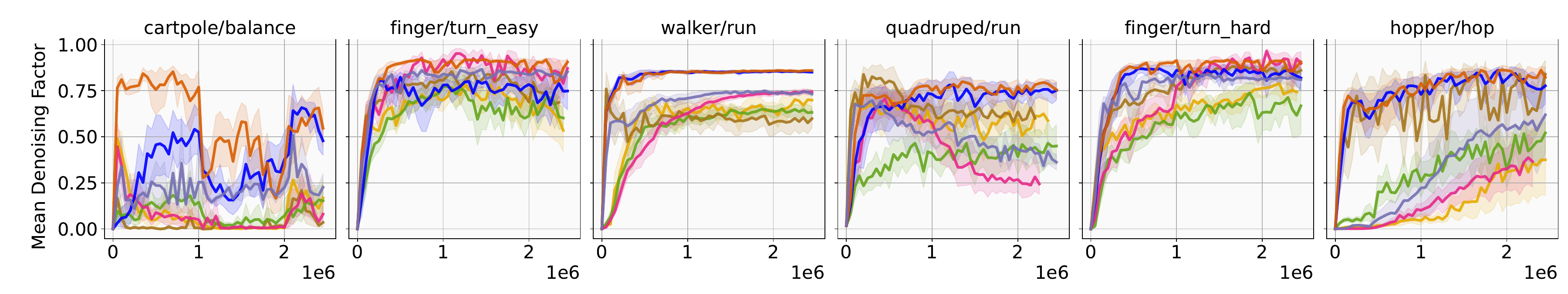}
     \vspace{-0.3cm}
     \includegraphics[width=1\linewidth]{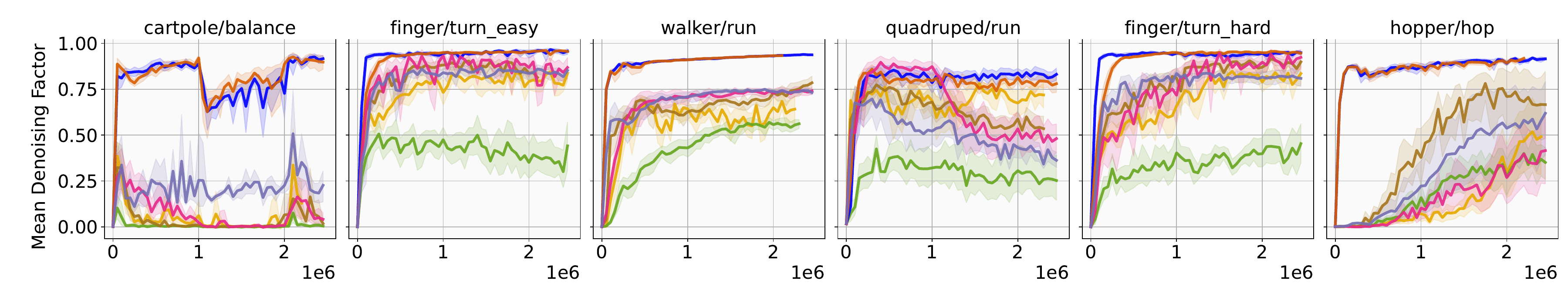}
     \vspace{-0.1cm}
     \includegraphics[width=1\linewidth]{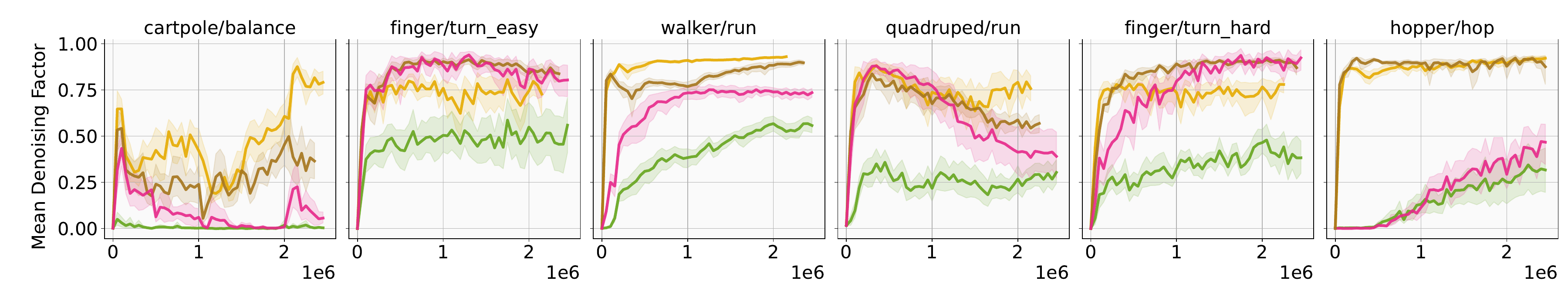}
     \vspace{-2em}
    \caption{\small DF for isolated encoder $\tilde{\phi}$ -- \textbf{top row}: without LayerNorm; \textbf{middle row}: with LayerNorm; \textbf{bottom row}: with LayerNorm, $\Tilde{\phi}$ co-trained with metric and ZP losses. See \autoref{fig:5.3iso_sacln_rew} for reward curves and \autoref{fig:df_co_training_RL} for DF for the agent encoder $\phi$ co-trained with metric and RL losses.}
    \vspace{-1em}
    \label{fig:5.3compared}
\end{figure}



\vspace{-0.8em}
\subsection{OOD Generalization Evaluation on Pixel-based Tasks}
\label{sec:5.4}
\vspace{-0.5em}

\begin{wrapfigure}[9]{r}{0.6\linewidth}
\vspace{-1.7em}
\centering
\includegraphics[width=0.4\linewidth]{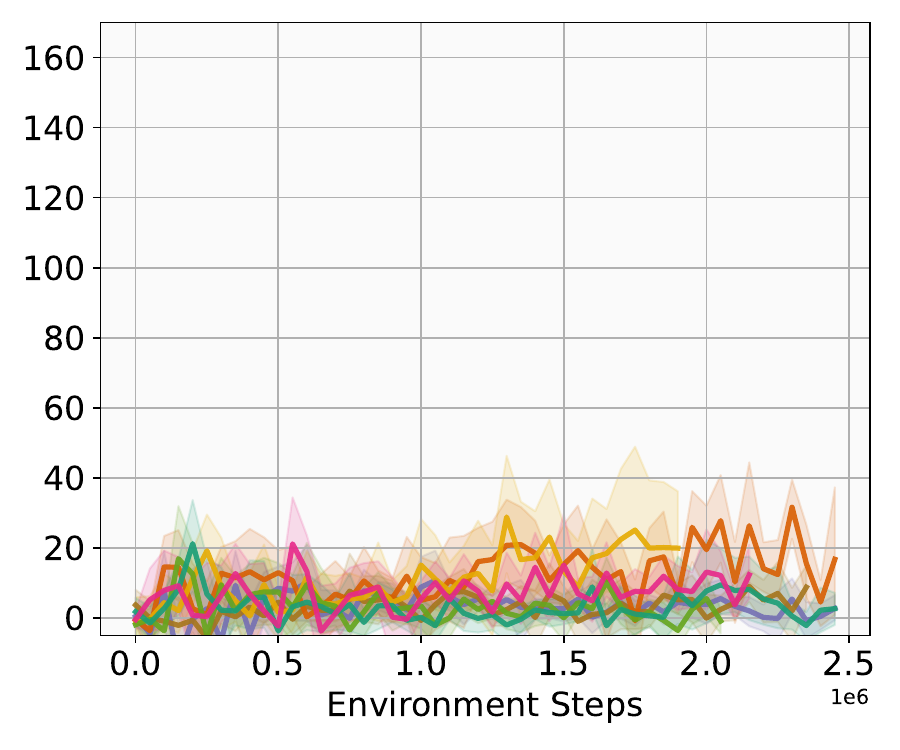}%
\hfill
\includegraphics[width=0.4\linewidth]{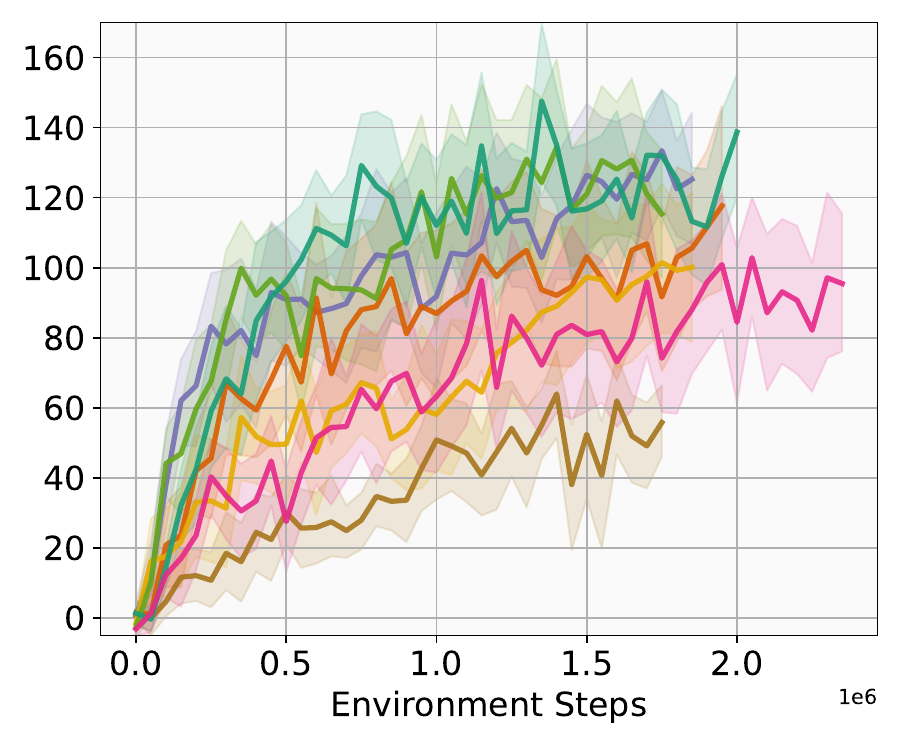}%
\includegraphics[width=0.18\linewidth]{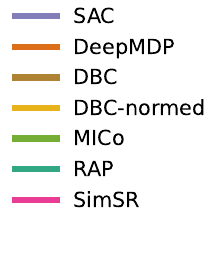}
\vspace{-1.2em}
\caption{\small \textbf{Reward difference} between ID and OOD evaluation in the grayscale video setting (left) and colored video setting (right), aggregated over 14 pixel-based tasks in \autoref{tab:pixel_levels}.}
\label{fig:gen_gap}
\end{wrapfigure}

While prior work has focused on OOD generalization in pixel-based settings, we extend this analysis by evaluating all \textbf{14 pixel-based tasks}. 
The ``grayscale video'' setting (and similarly for other settings) in \autoref{fig:5.4ood} denotes using grayscale videos as distracting backgrounds for both training and evaluation, with distinct video samples in each phase.
Takeaways in \autoref{fig:gen_gap} and \autoref{fig:5.4ood} are as follows:

\begin{figure}[htb]
    \centering
    \includegraphics[width=0.9\linewidth]{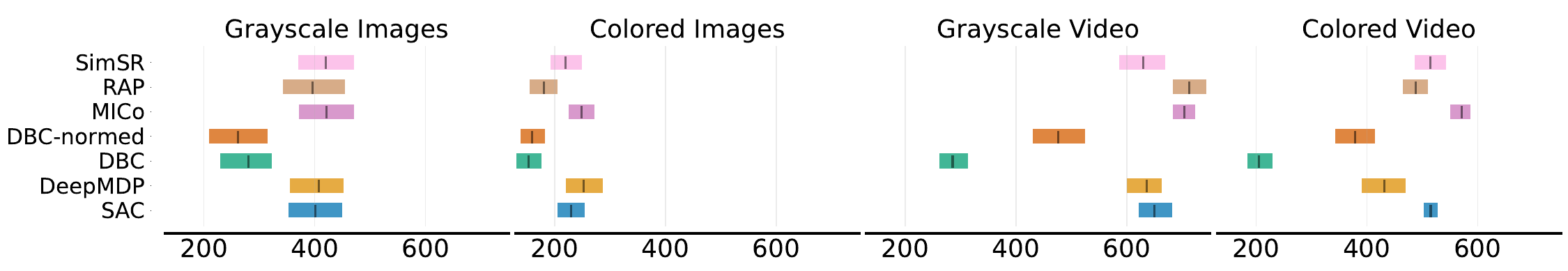}
    \vspace{0.1cm}
    \includegraphics[width=0.9\linewidth]{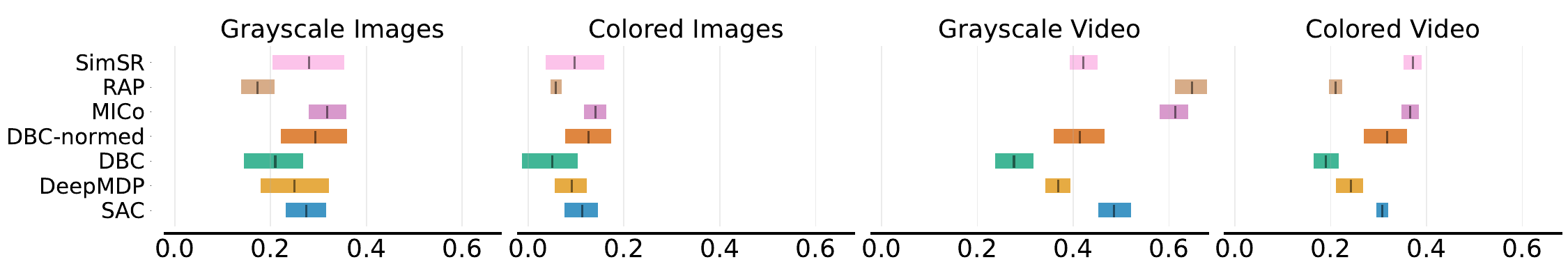}
\vspace{-1em}
    \caption{\small Aggregated reward (top row) and DF (bottom row) of 7 agents on various noise settings in 14 pixel-based tasks in \autoref{tab:pixel_levels} with \textbf{OOD generalization} evaluation.}
    \label{fig:5.4ood}
    \vspace{-1.5em}
\end{figure}

\begin{itemize}[itemsep=1.5pt, topsep=0pt, parsep=0pt, partopsep=0pt]
    \item Comparing \autoref{fig:5.4ood} (OOD) with \autoref{fig:5.1bar} (ID), all methods struggle to generalize in both grayscale and colored image settings. 
    Unlike video backgrounds, which provide changes over time that help distinguish relevant and irrelevant features, static images lack such variation, making adaptation to unseen backgrounds more difficult.
    \item Even with OOD generalization evaluation, SAC and DeepMDP remain competitive (\autoref{fig:5.4ood}). 
    \item OOD generalization is significantly more challenging in the \textit{colored} video setting than in the \textit{grayscale} video setting (\autoref{fig:gen_gap}). Surprisingly, even baselines like SAC exhibit a low reward difference, questioning the necessity of incorporating a metric loss in the widely used grayscale setting.
\end{itemize}



\vspace{-1em}
\section{Future Work}
\vspace{-1em}
Our empirical study focuses on continuous control, particularly locomotion, with SAC as the base RL algorithm. As future work, we plan to extend this to discrete control domains and other embodiments with alternative base RL algorithms.
To better disentangle algorithmic effects, future metric learning research may benefit from evaluation in \textit{conceptually simple yet empirically challenging} environments with distracting noise, such as Gaussian noise with random projection in state-based domains (\autoref{sec:noises}).
In addition, evaluation could include the denoising factor (\autoref{sec:df}) and comparisons between in- and out-of-distribution generalization (\autoref{sec:id_ood}). 
Building on our findings, future work could design new metric-based objectives that work well with self-prediction and normalization, providing additional benefits beyond what these components already offer.



\subsubsection*{Acknowledgments}
\label{sec:ack}
Ziyan Luo is funded by the Canada CIFAR AI Chair Program.
Tianwei Ni is funded by RBC Borealis Graduate Fellowship and Google DeepMind. 
We thank Sahand Rezaei-Shoshtari, Jianda Chen, Weijian Liao, 
Pablo Samuel Castro, Ayoub Echchahed, Lu Li, Xutong Zhao, Wesley Chung, Gandharv Patil, 
Aditya Mahajan, and anonymous reviewers for valuable discussions that helped shape our paper.
This research was enabled by compute resources provided by Mila (\url{https://mila.quebec}) and the Digital Research Alliance of Canada (\url{https://alliancecan.ca/}).


{\small
\bibliography{main}
\bibliographystyle{rlj}
}
\beginSupplementaryMaterials
\appendix

\part*{Appendix}
\etocsetnexttocdepth{subsection}         
\addcontentsline{toc}{part}{Appendix}
\localtableofcontents

\section{Related Work}
\label{sec:related}

\paragraph{State Abstraction and Behavioral Metrics.} State abstraction in MDPs is traditionally achieved by grouping equivalent states into a single abstract state.
Various types of abstractions are proposed using the criteria of different equivalence relations on aggregating the states, e.g., bisimulation (model-irrelevance) abstraction and $Q^*$-irrelevance abstraction~\citep{li2006towards,jiang2018notes}.
Bisimulation, as a canonical equivalence relation, originates from concurrency theory in the context of labeled transition systems~\citep{milner1980calculus,park1981concurrency} and labeled Markov process~\citep{desharnais2002bisimulation,panangaden2009labelled}. In the context of MDPs, it can be regarded as a refinement of both action-sequence equivalence (\textit{identical future reward sequences} given the same action sequence as input) and optimal value equivalence~\citep{givan2003equivalence}. 
To relax the strict dichotomy inherent in bisimulation relations, \citet{ferns2004metrics} propose using pseudometrics to measure the degree of bisimilarity between two states in a finite MDP and~\citet{ferns2011bisimulation} extend it further to continuous MDPs.
Policy-dependent bisimulation metric (PBSM)~\citep{castro2020scalable} is a scalable variant of the bisimulation metric (BSM) that restricts behavioral similarity to the current policy, eliminating the need to evaluate all actions.
GCB~\citep{hansen2022bisimulation} extends PBSM to goal-conditioned RL, proposing a similar metric that measures the distance between two state-goal pairs.
\citet{castro2021mico} propose the MICo distance, which enables sample-based computation of transition distance while also offering theoretical benefits~\citep{castro2023kernel}. Building on MICo, improvements have been introduced by SimSR~\citep{zang2022simsr} and RAP~\citep{chen2022learning} (see \autoref{app:metric_def}). \citet{chen2024state} further 
propose a multi-step behavioral metric that measures the distance between two pairs of states.\looseness=-1

Beyond state abstraction, state-action abstraction has been extensively studied as another paradigm of model minimization.
As early as 1978, \citet{whitt1978approximations,whitt1979approximations} proposes constructing approximate dynamic‑programming models by restricting to subsets of states and actions, establishing bounds on how the approximate return function diverges from the true optimum.
 In modern RL, this abstraction paradigm manifests in the form of \emph{MDP homomorphism}~\citep{ravindran2002model,ravindran2004algebraic}, which compresses state‑action pairs under equivalence relations that preserve reward and transition dynamics, as well as in their relaxation via the \textit{lax bisimulation metric}~\citep{taylor2008lax}. A promising application is the exploitation of environmental symmetries, and recent work has focused on scalable methods for discovering and approximating such homomorphisms~\citep{biza2018online,van2020mdp,van2020plannable,rezaei2022continuous,liao2023policy,panangaden2024policy}.


\paragraph{Representation and Metric Learning in RL.} 
As shown in \autoref{tab:cand_summary}, many compared methods adopt self-prediction loss in addition to metric loss to shape the representation; see~\citet{ni2024bridging} for a comprehensive review.
Notably, DeepMDP~\citep{gelada2019deepmdp} provides a theoretically grounded and empirically validated framework for self-predictive representation learning in RL, incorporating ZP and RP auxiliary losses alongside standard RL objectives to construct a latent MDP.
SPR~\citep{schwarzer2020data} adopts a similar self-supervised approach by learning a transition model and shaping the representation using signals from the $k$-th future observation rather than the next observation in DeepMDP. 

There are multiple ways to connect representation and metric learning in RL. For example, in \autoref{sec:3}, we introduce isometric embedding as a unifying approach adopted by all our candidate methods. 
Beyond this, \citet{agarwal2021contrastive} relax the binary indicator in the contrastive loss by substituting it with a metric‑induced continuous similarity measure. \citet{liu2023robust} use metrics for prototype representation clustering, and \citet{wang2023efficient,wang2024rethinking} utilize the learned metrics to shape the reward.\looseness=-1

Beyond their use in online, model-free RL, behavioral metrics have also been applied to representation learning in offline RL~\citep{dadashi2021offline,hong2023offline,pavse2023state,zang2023understanding} and model-based RL~\citep{shimizu2024bisimulation}.


\paragraph{Evaluation on Representation Learning and Denoising.} 
A closely related work, \citet{tomar2021learning}, evaluate representation learning methods on six \textit{distracting pixel-based} DMC tasks and Atari games. They show that a simple baseline using reward prediction and self-prediction (akin to DeepMDP in our implementation) outperforms one metric learning method (i.e., DBC). Similarly, \citet{li2022does} study self-supervised losses like BYOL~\citep{grill2020bootstrap} on six \textit{clean pixel-based} DMC tasks and Atari games. They find that BYOL, akin to self-prediction, benefits model-free agents but is inferior to data augmentation.
For denoising evaluation, previous works~\citep{zhang2020learning,chen2022learning,zang2022simsr} \textit{qualitatively} assess it by visualizing representations in a two-dimensional space using t-SNE plot, showing their connection to the approximated state-value function. 
The learned encoder is evaluated by determining whether anchor-positive pairs are closer than anchor-negative pairs in these plots, although these results may not be statistically significant.

Unlike prior empirical studies, our work specifically evaluates metric learning (five methods) across a broader range of state-based and pixel-based DMC tasks, considering both denoising factors and returns.  Beyond self-prediction, we identify other key design choices, such as layer normalization and metric loss function, that impact performance. Denoising factors are computed from batches of random samples throughout evaluation.

\paragraph{Benchmarks on Denoising and Generalization.} 
In state-based domains, IID Gaussian noises are commonly employed in prior work~\citep{kemertas2021towards,nikishin2022control,ni2024bridging} as a challenging benchmark by varying noise dimensions. However, our findings in DMC show that this protocol does not present sufficient difficulty for the best-performing method. 
In pixel-based domains, IID Gaussian noise is less common but has been studied in \citet{zhang2018natural} on Atari games. Unlike our findings, they report that Atari with natural video backgrounds is more challenging than with Gaussian noises, possibly due to domain differences (Atari vs. DMC).
Similar to our IID Gaussian with random projection setting, \citet{voelcker2024does} projects observations using a random binary matrix in discrete domains.

In pixel-based domains, beyond the commonly used grayscale video distractions~\citep{zhang2020learning,kemertas2021towards,zang2022simsr,chen2022learning}, several other works have explored injecting distracting video backgrounds. 
In DeepMind Control Remastered~\citep{grigsby2020measuring}, environments are initialized with a visual seed that randomly selects graphical elements -- such as floor textures, backgrounds, body colors, and camera/lighting settings -- to drastically alter the appearance of image renderings while leaving the transition dynamics unchanged, allowing millions of distinct visualizations of the same state sequence. The Distracting Control Suite~\citep{stone2021distracting} provides a benchmark by introducing challenging visual distractions including similar graphical variations,
resulting in a spectrum of difficulty levels.
The DMControl Generalization Benchmark (DMC-GB)~\citep{hansen2021softda} provides a framework for evaluating agents' generalization ability in pixel-based DMC tasks with distractions from random colors and video backgrounds. Unlike our setting, it trains agents in a fixed, clean background and evaluates them in a distracted environment, emphasizing the effectiveness of image-augmentation-based methods.
RL-ViGen~\citep{yuan2023rl} creates a benchmark including domains across indoor navigation, autonomous driving, and robotic manipulation, and various generalization types including visual appearances, camera views, lighting conditions, scene structures, and cross embodiments.

\paragraph{Feature Normalization in RL.} Recent work has demonstrated the benefits of layer normalization (LayerNorm)~\citep{ba2016layer} in deep RL, especially in state-based domains~\citep{hiraoka2021dropout,smith2022walk,nauman2024overestimation}.  Our findings also highlight LayerNorm's critical role in state-based tasks but with key differences. Unlike most prior work, we apply LayerNorm only to the encoder's output layer, not every dense layer, as in~\citet{fujimoto2023sale}. Additionally, we demonstrate the combined impact of LayerNorm and self-prediction on distracting tasks, rather than in purely model-free RL on clean tasks.
In pixel-based domains, most prior metric learning methods incorporate LayerNorm in the encoder following SAC-AE~\citep{yarats2021improving} but do not evaluate its contribution in ablation studies. 

Beyond LayerNorm, $L_2$ normalization is also widely used in RL~\citep{hussing2024dissecting}. SimSR~\citep{zang2022simsr} demonstrates its effectiveness for its own method in pixel-based domains. SPR~\citep{schwarzer2020data} integrates $L_2$ normalization into the self-prediction loss, effectively transforming MSE into cosine loss. However, unlike our implementation, SPR does not normalize representations for the actor or critic loss.

Finally, various other normalization techniques~\citep{bhatt2019crossq,fujimoto2023sale,li2023normalization,hansen2023td} have been explored in RL, focusing on sample efficiency and generalization. Our work complements these efforts by analyzing the role of LayerNorm in representation and metric learning on distracting tasks.

\paragraph{Beyond Behavioral Metrics.}
The idea of learning metrics that represent state discrepancy and embedding them into the representation space (\autoref{eq:isometric}) can be extended to a broader context beyond behavioral metrics that rely heavily on the reward signal.
As prominent examples, \citet{agarwal2021contrastive} utilize a state distance on optimal policies to shape the representation. 
\citet{wang2023efficient} introduce an inverse dynamics bisimulation metric that incorporates the discrepancy in predicted inverse dynamics into the PBSM formulation.
Similarly, \citet{rudolph2024learning} propose action-bisimulation metric that captures the equivalence of states in terms of action controllability only, where a small distance indicates that the states share similar inverse dynamics on the representation space.
\citet{myers2024learning} highlight the link between metrics and goal-conditioned RL by introducing a reward-free temporal distance between a state and a goal.

\section{Notation}
\label{app:notation}
\autoref{tab:abbreviations} presents the abbreviations used frequently in our work and their full names.

\autoref{tab:notation_mdp}, \autoref{tab:notation_agent}, and \autoref{tab:notation_metric} show the glossary used in this paper.

\begin{table}[h]
\centering
\small
\caption{Abbreviations and their full names.}
\vspace{-1em}
\begin{adjustbox}{max width=\linewidth}
\begin{tabular}{ll}
\toprule
\textbf{Abbreviation} & \textbf{Full Name} \\ \midrule
RP & Reward Prediction \\
ZP & Self-prediction \\
BMDP & Block Markov Decision Process~\citep{du2019provably} \\ 
EX-BMDP & Exogenous BMDP~\citep{efroni2021provable} \\ \midrule
BSM & Bisimulation Metric~\citep{ferns2004metrics,ferns2011bisimulation} \\
PBSM & Policy-dependent Bisimulation Metric~\citep{castro2020scalable} \\ \midrule
SAC   & Soft Actor-Critic~\citep{haarnoja2018soft} \\
SAC-AE & Soft Actor-Critic + Autoencoder~\citep{yarats2021improving} \\
DBC   & Deep Bisimulation Control~\citep{zhang2020learning} \\
DBC-normed  & DBC with max normalization~\citep{kemertas2021towards} \\
MICo  & Matching under Independent Couplings~\citep{castro2021mico} \\
SimSR & Simple Distance-based State Representation~\citep{zang2022simsr} \\
RAP   & Reducing Approximation Gap~\citep{chen2022learning} \\ \midrule
DF    & Denoising Factor \\ 
IID & Independent and Identically Distributed \\
ID Generalization & In-distribution Generalization \\
OOD Generalization & Out-of-distribution Generalization \\
DMC   & DeepMind Control Suite~\citep{tassa2018deepmind,tunyasuvunakool2020dm_control} \\
\bottomrule
\end{tabular}
\end{adjustbox}
\label{tab:abbreviations}
\end{table}

\begin{table}[h]
    \centering
    \small
    \begin{minipage}{0.53\linewidth}
        \centering
        \caption{\small Glossary of notations in EX-BMDP (\autoref{sec:2.1}). The top section lists symbols related to the latent states, while the bottom section defines symbols related to grounded observations.}
\setlength{\tabcolsep}{0.2em}
        \begin{tabular}{cc}
            \toprule
            \textbf{Symbol} & \textbf{Description} \\ 
            \midrule
            $z = (s,\xi)\in \mathcal Z$ & Environment's latent state  \\ 
            $p(z'\mid z,a)$ & Latent state transition \\            
            $s\in \mathcal S$ & Task-relevant state  \\
            $\xi\in \Xi$ & Task-irrelevant noise \\ 
            $a\in \mathcal A$ & Action \\ 
            $R(s,a)$ & Latent reward function \\
            $r\in \mathbb R$ & Reward  \\ 
            $\gamma \in [0,1)$ & Discount factor \\
            $p(s' \mid s,a)$ & Task-relevant state transition    \\
            $p(\xi' \mid \xi) $ & Task-irrelevant noise transition \\
            \midrule
            $x \in \mathcal X$ & Observation  \\ 
            $q(x \mid z)$ & Emission function \\
            $q^{-1}: \mathcal X \to \mathcal Z$ & Oracle encoder to $\mathcal Z$ \\
            $\phi^*:\mathcal X \to \mathcal S$ & Oracle encoder (to $\mathcal S$) \\
            $\mathcal R(x,a)$ & Grounded reward function \\ 
    $\mathcal{P}(x' \mid x,a)$ & Grounded transition\\
            \bottomrule
        \end{tabular}
        \label{tab:notation_mdp}
    \end{minipage}
    \hfill
    \begin{minipage}{0.46\linewidth}
        \centering
        \caption{\small Glossary of notations in RL agents.}
\setlength{\tabcolsep}{0.2em}
        \begin{tabular}{cc}
            \toprule
            \textbf{Symbol} & \textbf{Description} \\
            \midrule
            $\psi\in\Psi$ & Agent's representation \\
            $\phi: \mathcal X \to \Psi$ & Agent's encoder \\
            $\Tilde{\phi}: \mathcal X \to \Psi$ & (Isolated) Metric encoder \\
            $\pi_\theta(a \mid \psi)$ & (Latent) Actor \\ 
            $Q_\omega(\psi, a)$ & (Latent) Critic \\
            $R_\kappa(\psi, a)$ & (Latent) Reward model \\
            $P_\nu(\psi' \mid \psi, a)$ & (Latent) Transition model \\
            $\mathcal P_\phi(\psi' \mid x, a)$ & Grounded-to-latent transition\\ 
            \bottomrule
        \end{tabular}
        \label{tab:notation_agent}
\vspace{1em}
        \caption{\small Glossary of notations in metrics.}
\setlength{\tabcolsep}{0.4em}
        \begin{tabular}{cc}
            \toprule
            \textbf{Symbol} & \textbf{Description} \\
            \midrule
            $d_\mathcal X: 
            \mathcal X \times \mathcal X \to \mathbb R$ & Target metric \\
            $d_\Psi: \Psi \times \Psi \to \mathbb R$ & Representational metric \\
            $d_R: \mathbb R \times \mathbb R \to \mathbb R$ & Reward distance \\
            $d_T(d)(\cdot,\cdot)$ & Transition distance \\
            $x,x_+,x_-\in \mathcal X$ & Anchor, positive, negative \\
            $\mathrm{DF}^\pi_{d}(\phi)$ & Denoising factor\\
            \bottomrule
        \end{tabular}
        \label{tab:notation_metric}
    \end{minipage}
\end{table}

\section{Background on Metrics and Metric Learning}
\label{app:metrics}



\subsection{Metric, Pseudometric, and Diffuse Metric}
\label{app:metric_math}

\noindent \textbf{Metric.}
A function $d: \mathcal{X} \times \mathcal{X} \to \mathbb{R}_{\ge 0}$ is called a \emph{metric} on the space $\mathcal X$ if for all $x,y,z \in \mathcal{X}$:
\[
\begin{aligned}
&(1)\quad d(x,y) = 0 \;\Longleftrightarrow\; x = y,\\
&(2)\quad d(x,y) = d(y,x),\\
&(3)\quad d(x,z) \le d(x,y) + d(y,z).
\end{aligned}
\]

\noindent \textbf{Pseudometric\footnote{Sometimes termed ``semimetric''~\citep{ferns2004metrics}.}.}
A function $d: \mathcal{X} \times \mathcal{X} \to \mathbb{R}_{\ge 0}$ is a \emph{pseudometric} if it satisfies (2) and (3) above, and for (1) we only require $d(x,x) = 0$ for all $x \in \mathcal{X}$ (i.e.,\ $d(x,y) = 0$ does not imply $x=y$).

\noindent \textbf{Diffuse Metric (Definition 4.9 in \cite{castro2021mico}).}
A function 
\(d: \mathcal{X} \times \mathcal{X} \to \mathbb{R}_{\ge 0}\)
is a \emph{diffuse metric} if it satisfies properties \((2)\) and \((3)\) above, and for \((1)\)
we only require \(d(x,y)\ge 0\). 
That is, we do not demand \(d(x,x) = 0\) or that 
\(d(x,y)=0 \Longleftrightarrow x = y\).

\subsection{Definitions of Various Behavioral Metrics}
\label{app:metric_def}

In this section, we discuss the behavioral metrics for an EX-BMDP (\autoref{sec:2.1}) that serve as candidates for $d_\mathcal{X}$ in \autoref{sec:3}. From the observation space $\mathcal X$, the grounded transition function is defined as $\mathcal{P}(x' \mid x,a)=\sum_{z'\in \mathcal{Z}} q(x' \mid z')p(z' \mid q^{-1}(x),a)$ and the grounded reward function as $\mathcal R(x,a) = R(\phi^*(x),a)$. 
Let $x_1,x_2\in\mathcal{X}$ be two arbitrary observations.

Bisimulation metric is a relaxation of bisimulation relation~\citep{givan2003equivalence} by allowing a smooth variation based on differences in the reward function and transition dynamics. 
The bisimulation metric thus quantifies the behavioral similarity between two states and is formally defined as follows:
\begin{definition}[Bisimulation metric $d^\sim$ \citep{ferns2004metrics,ferns2011bisimulation}]
\label{def:bisim_metric}
There exists a unique pseudometric $d^\sim:\mathcal{X}\times \mathcal{X} \rightarrow \mathbb{R}$, called the bisimulation metric (BSM)\footnote{As noted by \citet{ferns2004metrics}, BSM relates to the largest bisimulation relation, $\sim$. For brevity, we simplify the original definition that uses the fixed-point of an operator and omit the proof for the existence of such a fixed-point.}, defined as:
\begin{equation}
d^\sim(x_1,x_2) \defeq \max_{a\in\mathcal{A}}\left(c_R|\mathcal R(x_1,a)-\mathcal R(x_2,a)|+c_T\mathcal{W}_1(d^\sim)(\mathcal P(\cdot\mid x_1,a),\mathcal P(\cdot\mid x_2,a))\right),
\end{equation}
where 1-Wasserstein (Kantorovich) distance $\mathcal W_1(d^\sim)(P, Q) = \inf_{\delta\in \mathcal T(P, Q)} \E{(x_1', x_2')\sim \delta}{d^\sim(x_1', x_2')}$ with $\mathcal T(P, Q)$ the coupling space for $P$ and $Q$. Here, $c_R$ and $c_T$ are coefficients for short-term and long-term behavioral differences, which are commonly set to $c_R=1$ and $c_T=\gamma$, where $\gamma$ is the MDP's discount factor.
\end{definition}    

In practice, applying the $\max$ operator over actions is intractable in continuous action spaces and pessimistically accounts for behavioral similarity across all actions, including those leading to low rewards. Policy-dependent bisimulation metrics~\citep{castro2020scalable} address this limitation by restricting behavioral similarity to the current policy, eliminating the need to evaluate all actions.
\begin{definition}[Policy-dependent bisimulation metric $d^\pi$~\citep{castro2020scalable}]
\label{def:pibisim_metric}
Given a policy $\pi: \mathcal X \to \Delta(\A)$, there exists a unique pseudometric $d^\pi:\mathcal{X}\times \mathcal{X} \rightarrow \mathbb{R}$, called a policy-dependent bisimulation metric (PBSM), defined as:
\begin{equation}
d^\pi(x_1,x_2) \defeq c_R|\mathcal R^\pi(x_1)-\mathcal R^\pi(x_2)|+c_T\mathcal{W}_1(d^\pi)(\mathcal P^\pi(\cdot\mid x_1),\mathcal P^\pi(\cdot\mid x_2)),
\end{equation}
where $\mathcal R^\pi(x) \defeq \E{a\sim\pi}{\mathcal R(x,a)}$ and $\mathcal P^\pi(\cdot\mid x) \defeq \E{a\sim\pi}{\mathcal P^\pi(\cdot\mid x,a)}$ are policy-dependent reward and transition, respectively.
\end{definition}

To approximate the 1-Wasserstein distance in PBSM, DBC~\citep{zhang2020learning} assumes a Gaussian transition kernel and uses 2-Wasserstein distance which has a closed-form solution under such assumption (\autoref{app:metric_learning}). 
To further circumvent the costly computation of the 1-Wasserstein distance, \citet{castro2021mico} propose MICo distance which uses the independent coupling rather than all coupling of the distributions in the 1-Wasserstein distance, which enables its computation using samples.

\begin{definition}[MICo distance $u^\pi$~\citep{castro2021mico}]
\label{def:mico}
Given a policy $\pi: \mathcal X \to \Delta(\A)$, there exists a unique diffuse metric $u^\pi:\mathcal{X}\times \mathcal{X} \rightarrow \mathbb{R}$, called MICo distance:
\begin{equation}
u^\pi(x_1,x_2) \defeq c_R|\mathcal R^\pi(x_1)-\mathcal R^\pi(x_2)|+c_T\E{x_1'\sim \mathcal P^\pi(\cdot|x_1),x_2'\sim \mathcal P^\pi(\cdot|x_2)}{u^\pi(x_1', x_2')}.
\end{equation}

\end{definition}
The independent coupling term (second term on the RHS) may yield a non-zero distance between states that share identical immediate rewards and transition dynamics for all actions. Consequently, the MICo distance can be characterized as a diffuse metric (see \autoref{app:metric_math}).
Based on MICo, several improvements are made by SimSR~\citep{zang2022simsr} and RAP~\citep{chen2022learning}.
\begin{definition}[SimSR distance~\citep{zang2022simsr}]
\label{def:simsr}
Given a policy $\pi: \mathcal X \to \Delta(\A)$, there exists a unique distance $u^\pi:\mathcal{X}\times \mathcal{X} \rightarrow \mathbb{R}$, called the Simple State Representation (SimSR) distance:
\begin{equation}
u^\pi (x_1, x_2)
\;\defeq\;
c_R\bigl\lvert \mathcal R^\pi(x_1) - \mathcal R^\pi(x_2)\bigr\rvert
\;+\;
c_T \,\mathbb{E}_{\,x'_1 \sim \widehat{\mathcal P}^\pi(\cdot|x_1),\;x'_2 \sim \widehat{\mathcal P}^\pi(\cdot|x_2)}\bigl[u^\pi(x'_1,\,x'_2)\bigr],
\end{equation}
where $\widehat{\mathcal P}^\pi$ is an approximated transition dynamics model. Specifically in SimSR, through isometric embedding (\autoref{eq:isometric}), $u^\pi(x_1,x_2)=d_\mathcal{X}(x_1,x_2)=d_\Psi(\phi(x_1), \phi(x_2))
=1-\cos\bigl(\phi(x_1),\phi(x_2)\bigr)$, which is the cosine distance (normalized dot product distance).
\end{definition}

\begin{definition}[RAP distance~\citep{chen2022learning}]
\label{def:rap}
Given a policy $\pi: \mathcal X \to \Delta(\A)$, there exists a unique distance $u^\pi: \mathcal{X} \times \mathcal{X} \rightarrow \mathbb{R}$, called the Robust Approximate (RAP) distance:
\begin{equation}
\scalebox{0.9}{$
u^\pi(x_1, x_2)
\;\defeq\; 
c_R\bigl\lvert \mathcal R^\pi(x_1) - \mathcal R^\pi(x_2)
\bigr\rvert
\\[1ex]
\;+\;
c_T
\,\mathbb{E}_{a_1\sim \pi,\,a_2\sim \pi}
\Bigl[
\,u^\pi\Bigl(
\,\mathbb{E}_{x'_1 \sim \hat{\mathcal P}(\cdot |x_1,a_1)}[\,x'_1\,],
\,\mathbb{E}_{x'_2 \sim \hat{\mathcal P}(\cdot |x_2,a_2)}[\,x'_2\,]
\Bigr)
\Bigr]
$}.
\end{equation}
\end{definition}

\subsection{Approximating the Behavioral Metrics}
\label{app:metric_learning}
Behavioral metrics are approximated and isometrically embedded into the representation space via an auxiliary loss in recent works. 
In this section, we extend the explanation of design choices used to approximate the metrics in their implementations presented in \autoref{tab:cand_summary}.

In DBC~\citep{zhang2020learning} and DBC-normed~\citep{kemertas2021towards}, to approximate PBSM (\autoref{def:pibisim_metric}), the metric loss is defined in the following form:
{
\small
\begin{equation} 
J_M(\phi) = \Biggl( 
\underbrace{\|\phi(x_1) - \phi(x_2)\|_1}_{=d_{\Psi}(\phi(x_1),\phi(x_2))} 
\;-\; 
\Bigl(\underbrace{|r_1 - r_2| \;+\; \gamma\, \mathcal W_2(\|\cdot\|_1)\Bigl(\hat{ \mathcal P}\Bigl(\psi^\prime\,\left|\,\bar{\phi}(x_1), a_1\right.\Bigr), \hat{ \mathcal P}\Bigl(\psi^\prime\,\left|\,\bar{\phi}(x_2), a_2\right.\Bigr)\Bigr)\Bigr)}_{\approx d_R(x_1, x_2)+d_T(d_\Psi)(\mathcal P_\phi(\psi'|x_1), \mathcal P_\phi(\psi'|x_2))=d_{\mathcal{X}}(x_1, x_2)} 
\Biggr)^2.
\label{eq:dbc_loss}
\end{equation}
}

where $(x_1, x_2, a_1, a_2, r_1, r_2)$ is sampled from a replay buffer, $\bar{\phi}$ denotes the encoder $\phi$ with gradient detached, the transition model $\hat{\mathcal{P}}$ outputs a \textit{factorized} Gaussian distribution,
with mean $\mu_{\hat{\mathcal P}}(\bar{\phi}(x),a)$ and covariance $\sigma_{\hat{\mathcal P}}(\bar{\phi}(x),a)$.
For brevity, we denote $\mu_1 = \mu_{\hat{\mathcal{P}}}(\bar{\phi}(x_1), a_1)$ and $\sigma_1 = \sigma_{\hat{\mathcal{P}}}(\bar{\phi}(x_1), a_1)$.
$\mathcal{W}_2$ is the 2-Wasserstein distance, which serves as a surrogate for the 1-Wasserstein distance in \autoref{def:pibisim_metric} and admits a convenient closed-form solution when comparing two Gaussian distributions:
\begin{equation}
\mathcal W_{2}(\|\cdot\|_{2})\Bigl(\mathcal N(\mu_{1},\Sigma_{1}),\,\mathcal N(\mu_{2},\Sigma_{2})\Bigr)
=
\sqrt{\,
\|\mu_{1}-\mu_{2}\|_{2}^{2}
\;+\;
\bigl\|\Sigma_{1}^{\tfrac{1}{2}}-\Sigma_{2}^{\tfrac{1}{2}}\bigr\|_{F}^{2}
\,}\,.
\label{eq:w2_dist}
\end{equation}

where $\|\cdot\|_F$ is the Frobenius norm and 
$\Sigma_{1},\Sigma_{2} \in \mathbb{R}^{k \times k}$ are the covariance matrices.
In the special case where both Gaussians are factorized (i.e., $\Sigma_i = \mathrm{diag}(\sigma_i^2)$), the Frobenius norm simplifies to the Euclidean norm over the standard deviations, yielding:
\begin{equation}
\mathcal{W}_{2}(\|\cdot\|_{2})\Bigl(\mathcal{N}(\mu_{1}, \mathrm{diag}(\sigma_{1}^{2})),\, \mathcal{N}(\mu_{2}, \mathrm{diag}(\sigma_{2}^{2}))\Bigr)
=
\sqrt{\,
\|\mu_{1} - \mu_{2}\|_2^2
\;+\;
\|\sigma_{1} - \sigma_{2}\|_2^2
\,}\,.
\label{eq:w2_factorized}
\end{equation}

In their implementation, several major modifications are applied to \autoref{eq:dbc_loss}. 
First of all, DBC and DBC-normed use a \textit{scaled Huber distance} instead of the $L_1$ distance as the approximant of representation distance. 
We first define the Huber distance\footnote{\url{https://pytorch.org/docs/stable/generated/torch.nn.SmoothL1Loss.html}} for two vectors $x,y\in\mathbb{R}^k$ as:
\begin{align}
d_{\mathrm{Huber}}(x,y)
&= \sum_{i=1}^{k}
\begin{cases}
\tfrac12,(x_{i}-y_{i})^{2},
&\text{if }|x_{i}-y_{i}|<1,\\[6pt]
|x_{i}-y_{i}|-\tfrac12,
&\text{otherwise.}
\end{cases}
\end{align}

The scaled Huber distance can be formulated as:
\begin{equation} d_{\Psi}\bigl(\phi(x_{1}),\,\phi(x_{2})\bigr)
=\frac{1}{k} d_{\mathrm{Huber}}\bigl(\phi(x_{1}),\phi(x_2)\bigr).
\label{eq:smoothl1_vec}
\end{equation}
As for approximating reward distance $d_R$, instead of using the absolute difference,
\begin{equation}
\hat{d_R}(r_1, r_2) = |r_1-r_2|,
\label{eq:abs_diff}
\end{equation} DBC and DBC-normed apply Huber distance:
\begin{equation}
\hat d_{R}(r_{1},r_{2})
=d_{\mathrm{Huber}}\bigl((r_{1}),(r_{2})\bigr).
\label{eq:huber_dist}
\end{equation}

When approximating the transition distance $d_T$, DBC use the following form that slightly differs from \autoref{eq:w2_dist} to compare two Gaussian distributions:
\begin{equation}
\hat d_T\bigl((\mu_1,\sigma_1),\,(\mu_{2},\sigma_{2})\bigr)
=\frac{1}{k}\sum_{i=1}^{k}
\sqrt{
(\mu_{1,i}-\mu_{2,i})^{2}
\;+\;
(\sigma_{1,i}-\sigma_{2,i})^{2}
}\,.
\label{eq:pairwise_euclid_avg}
\end{equation}

DBC-normed, instead, utilize this form: 
\begin{align}
\hat d_T\bigl((\mu_1,\sigma_1),\,(\mu_{2},\sigma_{2})\bigr)
&= \frac{1}{k} \bigl(d_{\mathrm{Huber}}(\mu_1,\mu_{2})
  \;+\;d_{\mathrm{Huber}}(\sigma_1,\sigma_{2})\bigr).
\label{eq:smoothl1_mu_sigma}
\end{align}

MICo~\citep{castro2021mico} shares a similar metric loss structure with DBC:
\begin{equation}
    J_{M}(\phi) =  \Bigl( \underbrace{U_{\phi}(x_1,x_2)}_{=d_{\Psi}(\phi(x_1),\phi(x_2))} - \underbrace{\bigl(\bigl| r_1 - r_2 \bigr| + \gamma\, U_{\bar{\phi}}(x_1', x_2')\bigr)}_{\approx d_{\mathcal{X}}(x_1, x_2)} \Bigr)^2 ,
    \label{eq:mico_metric_loss}
\end{equation}
where $(x_1, x_2, r_1, r_2, x'_1, x'_2)$ is sampled from a replay buffer, $U_\phi(x_1,x_2)$ is the representation distance, parameterized as:
\begin{equation}
     d_\Psi(\phi(x_1), \phi(x_2))=U_\phi(x_1,x_2) \;:=\; \frac{\|\phi(x_1)\|_2^2 + \|\phi(x_2)\|_2^2}{2} \;+\; \beta\, \theta\bigl(\phi(x_1),\phi(x_2)\bigr),
\end{equation}
where $\theta$ represents an angular distance function defined in Appendix Sec. C.2 in~\citet{castro2021mico}.
In the implementation of MICo, the Huber loss~\citep{huber1992robust} is used instead of MSE in \autoref{eq:mico_metric_loss}.

SimSR~\citep{zang2022simsr} differs from MICo in its metric loss along two main dimensions. First, the parameterization of $U_\phi$:
\begin{equation}
    d_\Psi(\phi(x_1), \phi(x_2))=U_\phi(x_1,x_2):=1-\cos\bigl(\phi(x_1),\phi(x_2)\bigr).
\end{equation}
Second, the next latents $\phi(x'_1), \phi(x'_2)$ are sampled from a transition model rather than encoded from next observations in the replay buffer.

RAP~\citep{chen2022learning} reduces the gap of $d_R$ and $\hat d_R$ in the prior work by introducing a better approximant of $d_R$, motivated by the following derivation of $d_R$:
{
\small
\begin{equation}
    d_R(x_1,x_2)=\bigl\lvert \mathcal R^\pi(x_1) - \mathcal R^\pi(x_2)\bigr\rvert=\sqrt{
\mathbb{E}_{a_1 \sim \pi,\, a_2 \sim \pi}
\left[
\left| \mathcal{R}(x_1,a_1) - \mathcal{R}(x_2,a_j) \right|^2
\right]
-
\mathrm{Var}[r_{x_1}]
-
\mathrm{Var}[r_{x_2}]
},
\end{equation}
}

where $r_{x}$ is a random variable defined by $p(r_x = \mathcal R(x,a)) = \pi(a \mid x)$, and $\mathrm{Var}[r_{x}]$ is the variance of $r_x$. 
To approximate $\mathrm{Var}[r_x]$, RAP introduces an observation-reward model\footnote{Note that this model introduces additional gradients that influence the encoder's representation learning.} that maps the current observation to the mean and variance of the expected reward. The estimated $\hat{d}_R$ is then computed by substituting $\mathrm{Var}[r_x]$ with the model-predicted variance, and approximating the expected reward difference using sampled reward pairs from a replay buffer. 
For $\hat{d}_T$, RAP adopts the same computation procedure as DBC.

\subsection{Normalization Techniques}
\label{app:normalization}
Several related normalization techniques and their connections are introduced in this section. We also justify our study design by incorporating LayerNorm in candidate methods rather than the $L_2$ normalization employed in SimSR~\citep{zang2022simsr}.

\paragraph{Max Normalization.}
DBC-normed~\citep{kemertas2021towards} derives an upper bound of $L_p$-norm of the representation, and imposes this bound to the representation space. Specifically, assuming the boundedness of the reward, the target metric $d_\mathcal{X}$ (PBSM in DBC-normed) and the representational metric $d_\Psi$ (through isometric embedding (\autoref{eq:isometric})) can be upper bounded by a constant $C$:
\begin{equation}
    d_\Psi(\phi(x_1),\phi(x_2))=d_\mathcal{X}(x_1,x_2)\leq \frac{c_R}{1-c_T}(\max_{x,a}\mathcal{R}(x,a)-\min_{x,a}\mathcal{R}(x,a)) \defeq C.
    \label{eq:DBC-normed_distance_bound}
\end{equation}
For example, in our hyperparameter setting in DMC, the constant $C$ can be $100, 200,$ or $400$, depending on the action repeat.
Such a bound can also be naturally generalized to MICo, SimSR, and RAP distance.
Specifically, when $d_\Psi$ is the $L_p$-distance, if the $L_p$-norm of $\psi=\phi(x)$ is upper bounded as:
\begin{equation}
    \|\psi\|_p\leq \frac{C}{2},
\end{equation}
then \autoref{eq:DBC-normed_distance_bound} can be satisfied.

A ``max normalization'' can then be imposed on $\psi$ to constrain the approximated metrics within a reasonable numerical range, thereby improving metric estimation:
\begin{equation}
\operatorname{MaxNorm}(\psi) \coloneqq 
\begin{cases}
\psi, & \text{if } \|\psi\|_p < \frac{C}{2}, \\
\frac{C}{2}\frac{\psi}{\|\psi\|_p}, & \text{otherwise.}
\end{cases}
\end{equation}

\paragraph{$L_2$ Normalization.}
The $L_2$ normalization of $\psi$ enforces $\|\psi\|_2=1$, which is defined as:
\begin{equation}
    \operatorname{L2Norm}(\psi) = \frac{\psi}{\|\psi\|_2}. 
\end{equation}

\begin{remark}[\textbf{Caveat of applying $L_2$ normalization to metric learning}]
Note the boundedness of $L_p$ distance for $L_2$-normalized vectors $\phi(x_1),\phi(x_2)\in \mathbb{R}^k$ are given as follows by Hölder's inequality: 
\begin{equation}
\|\phi(x_1)-\phi(x_2)\|_p \le  2 \|\phi(x_1)\|_p \le 
\begin{cases}
2\,k^{\frac{1}{p}-\frac{1}{2}}, & \text{if } 1\le p \le 2, \\[1mm]
2, & \text{if } p\ge 2.
\end{cases}
\end{equation}

Note that if the $L_p$-distance is used as $d_\Psi$, directly applying $L_2$ normalization on $\psi$ can lead to the \textbf{misspecification of the metrics}. 
In other words, the numerical range of $d_\Psi$ may not be \textit{sufficiently expressive} to capture the ground truth metrics, i.e., the target metric space $(\mathcal{X}, d_\mathcal{X})$ cannot always be isometrically embedded into the Euclidean unit sphere $(\Psi_{\text{unit}}, \|\cdot\|_p)$. Consider a counterexample when $p=2$. Suppose the existence of a pair of $x_1,x_2$ that $c_R|\mathcal{R}^\pi(x_1)-\mathcal{R}^\pi(x_2)|>2$, and $d_\mathcal{X}(x_1,x_2)>2$. But $d_\Psi(\phi(x_1),\phi(x_2))=\|\phi(x_1)-\phi(x_2)\|_p \le 2$. Thus, in this case, there is no such $\phi$ that $\|\phi(x_1)\|_2=\|\phi(x_2)\|_2=1$ and $d_\mathcal{X}(x_1,x_2)=d_\Psi(\phi(x_1),\phi(x_2)) $.

As another counterexample, in SimSR~\citep{zang2022simsr}, $L_2$ normalization is imposed on the representation space, i.e., $\|\phi(x_1)\|_2=\|\phi(x_2)\|_2=1$. Under such condition, one can show the equivalence of the cosine distance and the squared $L_2$ distance:
\begin{equation}
\begin{split}
d_\Psi(\phi(x_1), \phi(x_2))
&=1-\cos\bigl(\phi(x_1),\phi(x_2)\bigr) \\
&= 1-\langle\phi(x_1),\phi(x_2)\rangle\\[1mm]
&=\frac{1}{2}\Bigl(\|\phi(x_1)\|_2^2+\|\phi(x_2)\|_2^2-2\langle\phi(x_1),\phi(x_2)\rangle\Bigr)\\[1mm]
&=\frac{1}{2}\|\phi(x_1)-\phi(x_2)\|_2^2\in[0,2].
\end{split}
\end{equation}
Similarly, in such a case, a $\phi$ that satisfies $\forall x_1,x_2\in \mathcal{X}$, $\|\phi(x_1)\|_2=\|\phi(x_2)\|_2=1$ and $d_\mathcal{X}(x_1,x_2)=d_\Psi(\phi(x_1),\phi(x_2)) $ does not always exist.
\end{remark}

\paragraph{Layer Normalization.} LayerNorm~\citep{ba2016layer} can be formalized as follows:
\begin{equation}
\operatorname{LayerNorm}(\psi) = \alpha \odot \frac{\psi - \mu(\psi)}{\sqrt{\sigma^2(\psi)+\epsilon}} + \beta, 
\end{equation}
where $\alpha, \beta \in \mathbb{R}^k$ are learnable parameters, $\epsilon>0$ is a small constant for numerical stability, and $\odot$ denotes element-wise multiplication. $\mu$ and $\sigma^2$ are the mean and variance: 
\begin{equation}
    \mu(\psi)=\frac{1}{k}\sum_{i=1}^{k}\psi_i,\qquad \sigma^2(\psi)=\frac{1}{k}\sum_{i=1}^{k}(\psi_i-\mu(\psi))^2.
\end{equation}
Note that, if we assume $\alpha = \alpha_0 \,\mathbf{1}_k \in \mathbb{R}^k$ where $\alpha_0$ is a constant and $\beta\approx \boldsymbol{0}$ (as it is initialized to 0 in many implementations and remains small early in training), we have:

\begin{equation}
\begin{split}
\big\|\operatorname{LayerNorm}(\psi)\big\|_2^2 &\approx \sum_{i=1}^k \alpha_i^2 \cdot \frac{(\psi_i - \mu(\psi))^2}{\sigma^2(\psi)+\epsilon} \\
&= \alpha_0^2 \cdot \frac{\sum_{i=1}^k (\psi_i - \mu(\psi))^2}{\sigma^2(\psi)} \\
&= \alpha_0^2 \cdot \frac{k\sigma^2(\psi)}{\sigma^2(\psi)} \quad \text{(from variance definition)} \\
&= \alpha_0^2 k.
\label{eq:layer_norm_sq}
\end{split}
\end{equation}

Then the $L_2$ norm of the representation after LayerNorm satisfies:
\begin{equation}
\big\|\operatorname{LayerNorm}(\psi)\big\|_2 \approx \alpha_0\sqrt{k}.
\end{equation}

For layer-normalized $\phi(x_1)$ and $\phi(x_2)$, the upper bound of $d_\Psi(\phi(x_1), \phi(x_2))$ depends on both $\alpha$ and $\beta$, providing flexibility in expressing the target metrics $d_\mathcal{X}$.  
As a result, we conduct case studies in \autoref{sec:5.2} on methods with LayerNorm rather than $L_2$ normalization to ensure full expressivity of the target metrics.

\section{Proof}

\subsection{Proof of Transition Distance Preservation under Isometric Embedding}
\label{app:proof_isometric}

\begin{lemma}[Transition distance preservation (\autoref{eq:transition_distance})]
Let $\phi:\mathcal X \to \Psi$ be an isometric embedding with $d_{\mathcal X}(x_1, x_2) = d_{\Psi}(\phi(x_1), \phi(x_2))$ for any $x_1,x_2\in\mathcal X$ (\autoref{eq:isometric}). Then,
\begin{equation}
d_T(d_{\mathcal X})(\mathcal P(x'\mid x_1),\mathcal P(x'\mid x_2)) = d_T(d_{\Psi})(\mathcal P_\phi(\psi'\mid x_1),\mathcal P_\phi(\psi'\mid x_2)).
\end{equation}
\end{lemma}

\begin{proof}
We prove it by considering two common forms of $d_T$: Wasserstein distance~\citep{ferns2004metrics,ferns2011bisimulation,zhang2020learning} and sampling-based distance~\citep{castro2021mico,zang2022simsr}. 

Case 1: $d_T$ is Wasserstein distance. For convenience, we take 1-Wasserstein distance as example, and the proof naturally extends to $p$-Wasserstein. Let $\mathcal T(P,Q)$ be the coupling space for $P$ and $Q$, 
\begin{equation}
\begin{split}
&\mathcal{W}_1(d_\mathcal X)(\mathcal P(x'\mid x_1),\mathcal P(x'\mid x_2)) =  \inf_{\mu \in \mathcal T(\mathcal P(x'\mid x_1), \mathcal P(x'\mid x_2))} \sum_{x_1'\in\mathcal X, x_2'\in\mathcal X} d_\mathcal X(x_1',x_2') \mu(x_1', x_2') \\ 
&= \inf_{\mu \in \mathcal T(\mathcal P(x'\mid x_1), \mathcal P(x'\mid x_2))}  \sum_{x_1'\in\mathcal X, x_2'\in\mathcal X} d_\Psi(\phi(x_1'),\phi(x_2')) \mu(x_1', x_2') \quad \text{(isometric embedding)}\\
&= \inf_{\mu \in \mathcal T(\mathcal P(x'\mid x_1), \mathcal P(x'\mid x_2))}  \sum_{\psi_1'\in\Psi, \psi_2'\in\Psi} d_\Psi(\psi_1', \psi_2') \sum_{x_1',x_2': \phi(x_1')=\psi_1',\phi(x_2')=\psi_2'}\mu(x_1',x_2') \\
&= \inf_{\nu \in \mathcal T(\mathcal P_\phi(\psi'\mid x_1), \mathcal P_\phi(\psi'\mid x_2))}  \sum_{\psi_1'\in\Psi, \psi_2'\in\Psi} d_\Psi(\psi_1', \psi_2') \nu(\psi_1',\psi_2') \\
&= \mathcal{W}_1(d_\Psi)(\mathcal P_\phi(\psi'\mid x_1),\mathcal P_\phi(\psi'\mid x_2)).
\end{split}
\end{equation}

Case 2: $d_T$ as sampling-based distance.
\vspace{-1em}
\begin{equation}
\begin{split}
&\E{x_1'\sim \mathcal P(\cdot\mid x_1),x_2'\sim \mathcal P(\cdot \mid x_2)}{d_\mathcal X(x_1', x_2')} = \sum_{x_1'\in\mathcal X, x_2'\in\mathcal X} d_\mathcal X (x_1', x_2') \mathcal P(x_1'\mid x_1) \mathcal P(x_2' \mid x_2)\\
&= \sum_{x_1'\in\mathcal X, x_2'\in\mathcal X} d_\Psi(\phi(x_1'),\phi(x_2')) \mathcal P(x_1'\mid x_1) \mathcal P(x_2' \mid x_2) \quad \text{(isometric embedding)}\\
&= \sum_{\psi_1'\in\Psi, \psi_2'\in\Psi} d_\Psi(\psi_1', \psi_2') \sum_{x_1',x_2': \phi(x_1')=\psi_1',\phi(x_2')=\psi_2'} \mathcal P(x_1'\mid x_1) \mathcal P(x_2' \mid x_2) \\
&= \sum_{\psi_1'\in\Psi, \psi_2'\in\Psi} d_\Psi(\psi_1', \psi_2') \left(\sum_{x_1': \phi(x_1')=\psi_1'} \mathcal P(x_1'\mid x_1)\right) \left(\sum_{x_2': \phi(x_2')=\psi_2'}\mathcal P(x_2' \mid x_2) \right)\\
&= \E{\psi_1'\sim \mathcal P_\phi(\cdot\mid x_1),\psi_2'\sim \mathcal P_\phi(\cdot \mid x_2)}{d_\Psi(\psi_1', \psi_2')}.
\end{split}
\end{equation}
\end{proof}

\subsection{Proof of Denoising Property of Bisimulation Metrics}
\label{app:proof_denoising}

\begin{proposition}[\textbf{Denoising property of BSM}]
For any $x, x_+ \in \mathcal X$ of an EX-BMDP (\autoref{sec:2.1}) with $\phi^*(x) = \phi^*(x_+)$, the bisimulation metric (\autoref{def:bisim_metric}) is zero: $d^\sim (x, x_+) = 0$.
\end{proposition}

\begin{proof}
Let $d:\mathcal X \times\mathcal X \to \mathbb R$ be any pseudometric. We define the bisimulation operator $\mathcal F$ on $d$: $\forall x_1,x_2\in \mathcal X$,
\begin{equation}
(\mathcal F (d))(x_1,x_2) \defeq \max_{a\in\mathcal{A}}\left(c_R|\mathcal R(x_1,a)-\mathcal R(x_2,a)|+c_T\mathcal{W}_1(d)(\mathcal P(x'\mid x_1,a),\mathcal P(x'\mid x_2,a))\right).
\end{equation}
We initialize $d^{(0)}(x_1, x_2) = 0, \forall x_1,x_2$ and then apply the operator: $d^{(n+1)} = \mathcal F(d^{(n)}),\forall n\in \mathbb N$. Iteratively, the metric will converge to the unique fixed point  $d^{\sim}$~\citep{ferns2004metrics,ferns2011bisimulation}: $\lim_{n\to \infty} d^{(n)} = d^\sim.$

We show by induction that $d^{(n)}(x,x_+) = 0$ whenever $\phi^*(x)=\phi^*(x_+)$, and hence $d^{\sim}(x,x_+) = 0$ in the limit. 
The base case $n=0$ is immediate, because $d^{(0)}\equiv 0$.  Assume for some $n \in \mathbb N$, $d^{(n)}(x,x_+)=0$ whenever $\phi^*(x)=\phi^*(x_+)$. Consider the $n+1$-case: 

\begin{equation}
d^{(n+1)}(x,x_+) = \max_{a\in\mathcal{A}}\left(c_R|\mathcal R(x,a)-\mathcal R(x_+,a)|+c_T\mathcal{W}_1(d^{(n)})(\mathcal P(x'\mid x,a),\mathcal P(x'\mid x_+,a))\right).
\end{equation}
Recall in an EX-BMDP, $\mathcal R(x,a)$ only depends on $s=\phi^*(x)$ and $a$. Since $\phi^*(x)=\phi^*(x_+)$, the reward difference term is zero. The remainder of the expression is governed by the transition distance term:
\begin{equation}
\mathcal{W}_1(d^{(n)})(\mathcal P(x'\mid x,a),\mathcal P(x'\mid x_+,a)) = 
\inf_{\mu\in \mathcal T(\mathcal P(\cdot\mid x,a), \mathcal P(\cdot\mid x_+,a))} \sum_{x',x_+'\in\mathcal X} d^{(n)}(x',x_+')\mu(x',x_+'),
\end{equation}
where $\mathcal T(P,Q)$ is the coupling space for $P$ and $Q$. 
Consider the coupling that pairs next‐state samples by forcing them to share the same task‐relevant state. Let $\xi,\xi_+$ be the task-irrelevant noise underlying $x, x_+$ (i.e., $[\phi^*(x),\xi]=q^{-1}(x)$). Since $\phi^*(x) = \phi^*(x_+)$, the coupling $\mu$ is defined by
\begin{equation}
\mu(x',x_+') = \sum_{s'\in\mathcal S,\xi'\in\Xi,\xi_+'\in \Xi} p(s' \mid \phi^*(x), a) p(\xi' \mid \xi) p(\xi'_+ \mid \xi_+) q(x'\mid s',\xi') q(x_+'\mid s', \xi_+').
\end{equation}
In this construction, any sample $(x',x_+')$ shares the same task-relevant state $s'$. By the inductive hypothesis, $d^{(n)}(x',x_+') = 0$. Thus,
\begin{equation}
\begin{split}
&\sum_{x',x_+'\in\mathcal X} d^{(n)}(x',x_+')\mu(x',x_+') =\sum_{x',x_+'\in\mathcal X}  0 = 0.
\end{split}
\end{equation}
Therefore, $\mathcal{W}_1(d^{(n)})(\mathcal P(x'\mid x,a),\mathcal P(x'\mid x_+,a)) = 0$ and $d^{(n+1)}(x,x_+) = 0$.
\end{proof}

\begin{proposition}[\textbf{Denoising property of PBSM with an exo-free policy}]
Define an \textit{exo-free} policy~\citep{islam2022agent} $\pi:\mathcal X \to \Delta(\mathcal A)$ that is independent of noise, i.e., $\pi(a\mid x) = \pi(a\mid x_+),\forall a \in \mathcal A$ whenever $\phi^*(x) = \phi^*(x_+)$. 
For any $x, x_+ \in \mathcal X$ of an EX-BMDP with $\phi^*(x) = \phi^*(x_+)$, the policy-dependent bisimulation metric (\autoref{def:pibisim_metric}) with an \textit{exo-free} policy $\pi$ is zero: $d^{\pi} (x, x_+) = 0$.
\end{proposition}

\begin{proof}
This is a proof sketch, adapting the bisimulation metric (BSM) proof above to PBSM. 

By induction, consider the $n+1$-case: 
\begin{equation}
d^{(n+1)}(x,x_+) = \left(c_R|\mathcal R^\pi(x)-\mathcal R^\pi(x_+)|+c_T\mathcal{W}_1(d^{(n)})(\mathcal P^\pi(x'\mid x),\mathcal P^\pi(x'\mid x_+))\right).
\end{equation}
Since $\phi^*(x) = \phi^*(x_+)$ and both reward and policy are exo-free, we have $\mathcal R^\pi(x) = \sum_a \pi(a\mid x) \mathcal R(x,a) = \sum_a \pi(a \mid x_+) \mathcal R(x_+,a) = \mathcal R^\pi(x_+)$, i.e., the reward difference term is zero. Consider the transition difference term, we construct a similar coupling:
{\small
\begin{equation}
\mu(x',x_+') = \sum_{a\in \mathcal A,s'\in\mathcal S,\xi'\in\Xi,\xi_+'\in \Xi} \pi(a\mid x) p(s' \mid \phi^*(x), a) p(\xi' \mid \xi) p(\xi'_+ \mid \xi_+) q(x'\mid s',\xi') q(x_+'\mid s', \xi_+').
\end{equation}
}
In this coupling, any sample $(x',x_+')$ shares the same task-relevant state $s'$. Hence $d^{(n+1)}(x,x_+) = 0$ and by the convergence of PBSM operator~\citep{castro2020scalable}, we have $d^{\pi} (x, x_+) = 0$. 
\end{proof}

\begin{remark}[\textbf{PBSM may \textit{not} exhibit denoising property at convergence}]
\label{remark:PBSM}
In general, an optimal policy is \emph{not necessarily} exo-free. That implies even if a policy converges to an optimal one, denoted as $\pi^*$, the corresponding PBSM, $d^{\pi^*}$, may still lack the denoising property. 

Consider the following counterexample. For $x,x_+\in \mathcal X$ be an anchor-positive pair with  the same task-relevant state. By bisimulation relation~\citep{li2006towards}, they have the same optimal-value function: $Q^*(x,a) = Q^*(x_+,a),\forall a\in \mathcal A$. Now, suppose there exist two distinct optimal actions $a_1,a_2\in \mathcal A$, such that $Q^*(x,a_1)=Q^*(x,a_2) =\max_{a\in \mathcal A} Q^*(x,a)$ and $a_1\neq a_2$. 
Construct a deterministic optimal policy, $\pi^*$, such that $\pi^*(x)=a_1$ and $\pi^*(x_+)=a_2$.
In this case, although $Q^*(x,a_1)=Q^*(x_+,a_2)$, the reward difference $|\mathcal{R}^\pi(x)-\mathcal{R}^\pi(x_+)|$ can be still nonzero when $\mathcal R(\phi^*(x),a_1) \neq \mathcal R(\phi^*(x),a_2)$. Thus, for this optimal policy $\pi^*$, $d^{\pi^*}(x,x_+)$ can also be nonzero. 

This example indicates that the PBSM under an optimal policy -- namely, the fixed point attained by PBSM-based methods (e.g., DBC, DBC-normed) -- may not exhibit the denoising property.
\end{remark}

In fact, similarly to \autoref{remark:PBSM} illustrating negative results for policy-dependent metrics, we find that model-irrelevance abstraction does not necessarily imply $Q^\pi$-irrelevance abstraction, thereby clarifying the scope of the classic result~\citep{li2006towards}. 

\begin{remark}[\textbf{Model-irrelevance abstraction $\phi_{\text{model}}$ may \textit{not} imply $Q^\pi$-irrelevance abstraction $\phi_{Q^\pi}$}]
In \cite[Theorem 2]{li2006towards}, it is established that $\phi_{\text{model}}$ implies $\phi_{Q^\pi}$ for \textit{any} policy $\pi$. Here, we revisit this statement and demonstrates that it does not extend to \emph{exo-dependent} policies.
To see why the classical result fails in this setting, let $x,x_+\in \mathcal X$ be an anchor-positive pair with the same task-relevant state. If $\phi_{\text{model}}$ implies $\phi_{Q^\pi}$, we would have $Q^\pi(x,a) = Q^\pi(x_+,a),\forall a\in \mathcal A,\forall \pi.$ 

However, consider a counterexample similar to \autoref{remark:PBSM}. Suppose $x,x_+$ deterministically transits to $x',x_+'$, respectively under an action $a$, and the policy $\pi$ is deterministic. Since $Q^\pi(x,a) =R(x,a) + \gamma Q^\pi(x',\pi(x')) $ and $\mathcal R(x,a) = \mathcal R(x_+,a)$, the key question becomes $Q^\pi(x',\pi(x')) \stackrel{?}{=} Q^\pi(x_+',\pi(x_+'))$. 
Although $x',x_+'$ are bisimilar, an \emph{exo-dependent} policy may select different actions for these two states, i.e., $\pi(x') \neq \pi(x_+')$.  
In a terminal step of a finite-horizon MDP, for example, we could have $Q^\pi(x',\pi(x')) = \mathcal R(x',\pi(x')) \neq \mathcal R(x_+',\pi(x_+')) = Q^\pi(x_+',\pi(x_+'))$. Therefore, $\phi_{\text{model}} \implies \phi_{Q^\pi}$ does not hold once we allow for exo-dependent policies.
\end{remark}

\section{Difficulty Levels of the Tasks}

Difficulty levels for each task in state-based and pixel-based DMC are provided in \autoref{tab:state_levels} and \autoref{tab:pixel_levels} respectively. 
These levels are based on the average reward across all compared methods in \autoref{tab:cand_summary}. 
Tasks listed in \autoref{tab:state_levels} and \autoref{tab:pixel_levels} are sorted in ascending order of difficulty. This ordering is also applied to all other per-task performance tables and figures, facilitating a clearer understanding of how different methods perform across tasks of varying difficulty levels.



\begin{table}[ht]
\centering
\small
\caption{\small\textbf{Difficulty levels for 20 state-based DMC tasks, as determined by the compared methods (\autoref{tab:cand_summary}).} 
``Avg Reward'' stands for the average reward across all IID Gaussian noise settings in \autoref{sec:5.1}.
For each run, the reported reward is the average of 10 evaluation points collected around 2M steps.
``Max (Min) Reward'' denotes the best (worst) agent's average reward over 12 runs, while ``Max/Min'' is the ratio of the best to worst performance, indicating a task's ability to discriminate between agent performances. }
\vspace{-0.5em}
\begin{adjustbox}{max width=\linewidth}
\begin{tabular}{l l r r r r l}
\toprule
Task      &             & Avg Reward & Max Reward & Min Reward & Max/Min & Difficulty \\
\midrule
ball\_in\_cup  & catch           & 934.8      & 977.4      & 841.7      & 1.2   & Easy \\
cartpole    & balance         & 919.4      & 997.3      & 791.2      & 1.3   & Easy \\
cartpole    & balance\_sparse & 877.7      & 983.6      & 772.3      & 1.3   & Easy \\
walker      & stand           & 834.6      & 979.0      & 437.8      & 2.2   & Easy \\
cartpole    & swingup         & 818.1      & 874.1      & 707.6      & 1.2   & Easy \\
walker      & walk            & 805.7      & 961.9      & 382.4      & 2.5   & Easy \\
\midrule
reacher     & easy            & 740.1      & 955.1      & 453.0      & 2.1   & Medium \\
finger      & spin            & 728.8      & 923.6      & 498.5      & 1.9   & Medium \\
quadruped   & walk            & 703.1      & 948.9      & 245.5      & 3.9   & Medium \\
cartpole    & swingup\_sparse & 647.3      & 839.1      & 531.9      & 1.6   & Medium \\
reacher     & hard            & 641.1      & 853.0      & 340.3      & 2.5   & Medium \\
finger      & turn\_easy      & 587.8      & 926.5      & 207.7      & 4.5   & Medium \\
walker      & run             & 545.8      & 776.1      & 117.4      & 6.6   & Medium \\
cheetah     & run             & 533.4      & 859.0      & 129.8      & 6.6   & Medium \\
pendulum    & swingup         & 514.3      & 824.5      & 247.2      & 3.3   & Medium \\
\midrule
quadruped   & run             & 460.7      & 864.3      & 199.0      & 4.3   & Hard \\
finger      & turn\_hard      & 435.6      & 893.0      & 102.6      & 8.7   & Hard \\
hopper      & stand           & 261.9      & 878.4      & 22.3       & 39.3  & Hard \\
acrobot     & swingup         & 75.7       & 246.1      & 11.2       & 22.0  & Hard \\
hopper      & hop             & 64.7       & 243.4      & 1.5        & 162.4 & Hard \\
\bottomrule
\end{tabular}
\end{adjustbox}
\label{tab:state_levels}
\end{table}

\begin{table}[ht]
\centering
\small
\vspace{-1em}
\caption{\small\textbf{Difficulty levels for 14 pixel-based DMC tasks, as determined by the compared methods (\autoref{tab:cand_summary}).}
``Avg Reward'' stands for the average reward across clean background, natural video (colored and grayscale), natural image (colored and grayscale), and IID Gaussian noise settings described in \autoref{sec:5.1}.
For each run, the reported reward is the average of 10 evaluation points collected around 2M steps.
``Max (Min) Reward'' denotes the best (worst) agent's average reward over 5 runs, while ``Max/Min'' is the ratio of the best to worst performance, indicating a task's ability to discriminate between agent performances. }
\vspace{-0.5em}
\begin{adjustbox}{max width=\linewidth}
\begin{tabular}{l l r r r r l}
\toprule
Task      &              & Avg Reward & Max Reward & Min Reward & Max/Min & Difficulty \\
\midrule
cartpole   & balance         & 949.3      & 986.7      & 905.5      & 1.1   & Easy \\
cartpole   & balance\_sparse & 915.3      & 999.4      & 804.6      & 1.2   & Easy \\
walker      & stand            & 887.7      & 959.1      & 633.7      & 1.5   & Easy \\
finger      & spin             & 815.2      & 909.5      & 426.4      & 2.1   & Easy \\
\midrule
cartpole    & swingup          & 765.0      & 853.2      & 551.1      & 1.5   & Medium \\
ball\_in\_cup  & catch            & 719.2      & 887.8      & 263.4      & 3.4   & Medium \\
walker     & walk             & 718.9      & 909.1      & 360.8      & 2.5   & Medium \\
point\_mass   & easy             & 421.5      & 558.6      & 256.1      & 2.2   & Medium \\
cartpole    & swingup\_sparse  & 409.6      & 680.4      & 57.8       & 11.8  & Medium \\
\midrule
reacher   & easy             & 336.8      & 949.1      & 113.0      & 8.4   & Hard \\
pendulum  & swingup          & 313.2      & 468.9      & 9.3        & 50.5  & Hard \\
cheetah   & run              & 299.4      & 411.0      & 144.7      & 2.8   & Hard \\
walker    & run              & 285.5      & 441.9      & 77.3       & 5.7   & Hard \\
hopper     & hop              & 73.6       & 122.5      & 5.6        & 22.0  & Hard \\
\bottomrule
\end{tabular}
\end{adjustbox}
\label{tab:pixel_levels}
\end{table}

\section{Implementation Details}

\subsection{Hyperparameters}
\label{app:hparams}
Our hyperparameter settings for the benchmarked agents are based on their open-source codebases and the values reported in their respective papers.
\autoref{tab:hyperparams} (above the double rule) shows the general hyperparameter setting adopted by most agents. For the action repeat, we set it to $4$ for most tasks, to $8$ for cartpole (swingup, swingup\_sparse), and to $2$ for finger spin and walker (walk, run, stand) following the convention~\citep{yarats2021mastering,zang2022simsr,chen2022learning}. Exceptions to \autoref{tab:hyperparams} are detailed below:

\begin{itemize}
    \item For MICo~\citep{castro2021mico}\footnote{\url{https://github.com/google-research/google-research/tree/master/mico}}, $\beta$ in MICo distance parametrization is set to $0.1$. In MICo's code, they use a different hidden unit size 1024 rather than 256 in our implementation, and a reward scale of 0.1 rather than 1 in our implementation. We assume that changing the hidden unit size alters the network architecture, and modifying the reward scale effectively changes the environment. Since other algorithms are evaluated under a unified setting, we avoid such modifications to isolate algorithmic differences, which are the primary focus of our study.
    \item For RAP~\citep{chen2022learning}\footnote{\url{https://github.com/jianda-chen/RAP_distance}}, we use the same hyperparameter setting in their open-source code, where the actor, critic, and encoder learning rates are set to $5\times 10^{-4}$, $\beta$ in MICo distance parametrization is set to $10^{-6}$ (which actually almost disables metric learning), RP and ZP loss coefficients are set to $10^{-4}$, and the encoder feature dimensionality is set to $100$.
\end{itemize}

\begin{table}[ht]
  \centering
    \caption{ \textbf{Hyperparameter settings} for most \textit{agents} we benchmarked (above the double rule) and for \textit{environmental} configurations (below the double rule).}
\vspace{-0.5em}
    \small
  \begin{tabular}{ll}
    \toprule
    \textbf{Hyperparameter Name} & \textbf{Value} \\
    \midrule
    Replay buffer capacity           & $1\times 10^6$ \\
    Replay ratio & $0.2$ \\
    Batch size                       & $128$ \\
    Discount $\gamma$                & $0.99$ \\
    Optimizer                        & Adam \\
    Encoder feature dimensionality & $50$ \\
    Hidden unit size in neural networks & $256$ \\
    
    \midrule
    Critic learning rate             & $1\times10^{-3}$ \\
    Critic target update frequency   & $2$ \\
    Critic Q-function soft-update rate $\tau_Q$ & $0.01$ \\
    Actor learning rate              & $1\times10^{-3}$ \\
    Actor update frequency           & $2$ \\
    Actor log stddev bounds          & $[-10,\,2]$ \\
    Encoder learning rate            & $1\times10^{-3}$ \\
    Encoder soft-update rate $\tau_\phi$ & $0.05$ \\
    Reward model and transition model's learning rate            & $1\times10^{-3}$ \\
    Reward model and transition model's weight decay             & $1\times10^{-7}$ \\
    SAC temperature learning rate        & $1\times10^{-4}$ \\
    SAC initial temperature & $0.1$ \\
    \midrule
    Metric loss coefficient $\lambda_{\text{M}}$ & $0.5$ \\
    Metric reward coefficient $c_R$ & $1$ \\
    Metric transition coefficient $c_T$ & $0.99$ \\
    RP loss coefficient $\lambda_{\text{RP}}$ & $1$ \\
    ZP loss coefficient $\lambda_{\text{ZP}}$ & $1$ \\
    \midrule
    \midrule
    Image size  & $84\times 84 \times 3$ \\
    Frame stack & $3$ \\
    Paralleled environments & $10$ \\
    Distracting video frames $N$ (per paralleled environment) & $1000$ \\
    \bottomrule
  \end{tabular}
  \label{tab:hyperparams}
\end{table}

\subsection{Model Architecture}
\label{app:arch}
The general model architectures of all our implemented methods and isolated metric estimation setting are illustrated in \autoref{fig:sac_arch}, \autoref{fig:deepmdp_arch}, \autoref{fig:metric_arch}, and \autoref{fig:iso_arch}, respectively.
For brevity, only the forward pass of the neural network is shown, referring to the computation from inputs to outputs without gradient updates.

\begin{figure}[htp]
\vspace{0em}
\centering
\begin{minipage}[t]{0.48\linewidth}
    \centering
    \includegraphics[width=\linewidth]{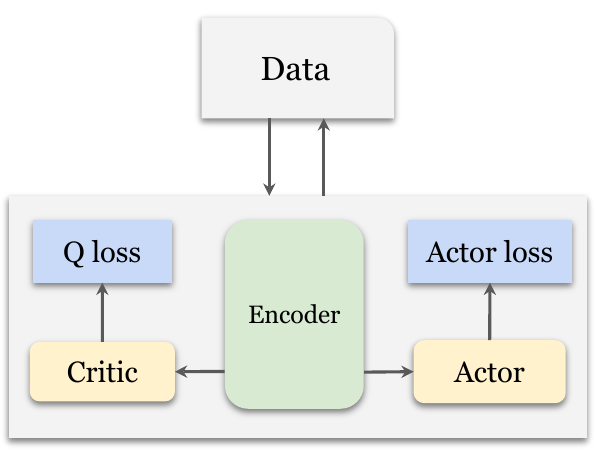}
    \caption{\small \textbf{SAC} architecture used in our experiments.}
    \label{fig:sac_arch}
\end{minipage}
\hfill
\begin{minipage}[t]{0.48\linewidth}
    \centering
    \includegraphics[width=\linewidth]{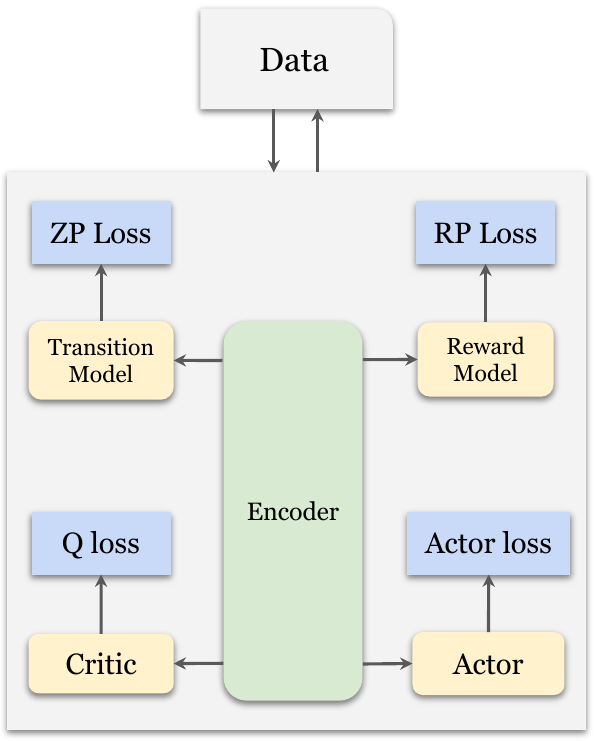}
    \caption{\small \textbf{DeepMDP} architecture used in our experiments.}
    \label{fig:deepmdp_arch}
\end{minipage}
\vspace{0em}
\end{figure}


\begin{figure}[htp]
\vspace{0em}
\centering
\begin{minipage}[t]{0.48\linewidth}
    \centering
    \includegraphics[width=\linewidth]{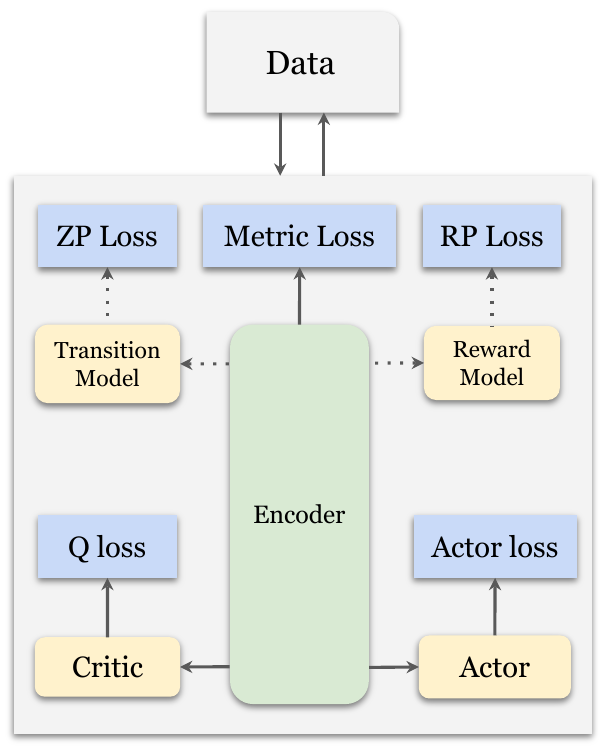}
    \caption{\small General architecture of \textbf{metric learning methods}, summarizing the architecture used in our benchmarked metric learning methods (\autoref{tab:cand_summary}). Dotted lines show optional data flows.}
    \label{fig:metric_arch}
\end{minipage}
\hfill
\begin{minipage}[t]{0.48\linewidth}
    \centering
    \includegraphics[width=\linewidth]{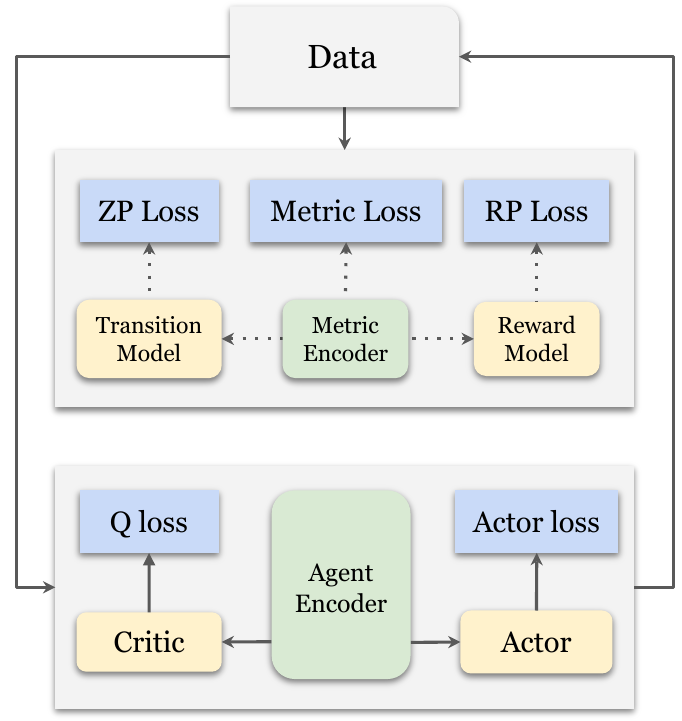}
    \caption{\small An instance of architecture described in the \textbf{isolated metric estimation setting} (\autoref{sec:isolated}). Shown is the case where a base agent, SAC (bottom gray module), is used to collect experiences, which are then used to train the isolated metric encoder, as implemented in our experiments in \autoref{sec:5.3}. Dotted lines show optional data flows. For example, in \autoref{sec:5.3}, isolated metric estimation for DeepMDP enables data flow only to the ZP and RP losses, whereas for MICo, only the metric loss is used.}
    \label{fig:iso_arch}
\end{minipage}
\vspace{0em}
\end{figure}

\subsection{Pixel-based Environmental Setup}
In our pixel-based settings, following~\citep{zhang2020learning}, we use natural videos and images from the Kinetics-400 dataset~\citep{kay2017kinetics} labeled \texttt{driving\_car} as distracting backgrounds.
The dataset is split into $817$ training and $90$ test video clips, each $10$ seconds long, and is publicly available in our codebase (\autoref{fn:artifact}).

For each run, a subset of videos of the training set is sampled for use during training, with the total number of frames determined by the hyperparameter ``distracting video frame $N$'' in \autoref{tab:hyperparams}; that is, the sampled videos must collectively provide at least $N$ frames.
In our experiments, we set $N=1000$, sampling $4$ to $5$ videos depending on their frame rates.
The test set is used for OOD generalization evaluation, ensuring that videos seen during training are excluded from evaluation. During OOD evaluation, each parallel environment is assigned a unique set of $N$ distracting frames, which remain fixed within a run. For example, in our experiments with $N=1000$ and $10$ parallel environments, a total of $10000$ frames are sampled and fixed during each run.

We elaborate on the instantiation of pixel-based noise settings introduced in \autoref{sec:noises}. In the natural video noise setting, clean background pixels are replaced with those from distracting videos; each parallel environment independently samples an initial starting index $\xi_0$, from which the video is played sequentially as background. 
In the natural image noise setting, background pixels are replaced with those from a single randomly sampled frame in the video dataset, which remains fixed across all parallel environments throughout each run.

Additionally, \autoref{tab:hyperparams} (below the double rule) presents the environmental hyperparameters. \autoref{tab:noise-ratios} quantifies the distraction in pixel-based distracting DMC tasks by presenting the percentage of noised pixels.

\begin{table}[ht]
    \centering
    \small
    \caption{Percentage of distracting pixels (the pixels that are task-irrelevant) for 14 pixel-based DMC tasks. These ratios remain consistent across all different pixel-based noise settings introduced in \autoref{sec:noises}.}
\vspace{-0.5em}
    \begin{tabular}{lll}
        \toprule
        \textbf{Task} &  & \textbf{Noise Ratio (\%)} \\
        \midrule
        cartpole    & balance         & 98.3\% \\
        cartpole    & balance\_sparse & 98.3\% \\
        walker      & stand           & 92.6\% \\
        finger      & spin            & 94.3\% \\
        \midrule
        cartpole    & swingup         & 98.3\% \\
        ball\_in\_cup & catch        & 99.0\% \\
        walker      & walk            & 92.6\% \\
        point\_mass & easy            & 99.7\% \\
        cartpole    & swingup\_sparse & 98.3\% \\
        \midrule
        reacher     & easy            & 96.5\% \\
        pendulum    & swingup         & 98.9\% \\
        cheetah     & run             & 95.4\% \\
        walker      & run             & 92.6\% \\
        hopper      & hop             & 97.3\% \\
        \bottomrule
    \end{tabular}
\vspace{-0.5em}
    \label{tab:noise-ratios}
\end{table}


\section{Additional Experiment Results}
\label{app:addition_exp}

\autoref{tab:extended_exp_summary} summarizes the \textbf{per-task result figures} presented in this section, each corresponding to an aggregated score reported in the main text. Detailed noise settings are indicated at the top of each figure and table.

Additionally, \autoref{tab:model_update_time} reports model update time comparisons from \autoref{sec:5.1}.  \autoref{fig:5.2DBC-normed} presents a case study on six state-based DMC tasks using DBC-normed with LayerNorm and its design variants, as discussed in \autoref{sec:5.2}.
\autoref{fig:df_sensitivity} and \autoref{fig:rp_sensitivity} serve as alternative illustrations of \autoref{fig:5.1bar} and \autoref{fig:5.2RP}, highlighting the agents' sensitivity to noise hyperparameters.

\begin{table}[h]
    \centering
    \caption{Summary of extended experimental results in Appendix. Per-task results corresponding to all aggregated scores shown in the main text are included.}
\vspace{-0.5em}
    \label{tab:extended_exp_summary}
    \renewcommand{\arraystretch}{1.5}
    \resizebox{\linewidth}{!}{%
    \begin{tabular}{@{}p{1.2cm} p{4.5cm} p{2.5cm} p{3.3cm}@{}}
        \toprule
        \textbf{Section} & \textbf{Description of Settings} & \textbf{Main Text \newline Reference} & \textbf{Per-task Figures / \newline Tables} \\
        \midrule

        \multirow{2}{*}{\autoref{sec:5.1}} 
        & Noise std.\ sweep:\newline $\sigma \in \{0.2,1.0,2.0,4.0,8.0\}$ \newline (fixed $m=32$) 
        & \autoref{fig:5.1bar} 
        & \autoref{tab:noise_std_0.2__noise_dim_32} -- \autoref{tab:noise_std_8__noise_dim_32}\newline \autoref{fig:all_std02} -- \autoref{fig:all_std8} \\
        
        & Noise dim.\ sweep:\newline $m \in \{2,16,32,64,128\}$ \newline (fixed $\sigma=1.0$) 
        & \autoref{fig:5.1bar} 
        & \autoref{tab:noise_std_1__noise_dim_2} -- \autoref{tab:noise_std_1__noise_dim_128}\newline \autoref{fig:all_dim2} -- \autoref{fig:all_dim128} \\
        
        \midrule

        \multirow{2}{*}{\autoref{sec:5.2}} 
        & LayerNorm ablation 
        & \autoref{tab:5.2ln} 
        & \autoref{fig:5.2ln_huber} \\
        
        & IID Gaussian + Random \newline projection noise std.\ sweep: \newline$\sigma \in \{0.2,1.0,2.0,4.0,8.0\}$\newline (fixed $m=32$) 
        & \autoref{fig:5.2RP} 
        & \autoref{fig:rp_pertask02} -- \autoref{fig:rp_pertask8} \\
        
        \midrule

        \multirow{3}{*}{\autoref{sec:5.3}} 
        & SAC (with LayerNorm) \newline reward curves 
        & \autoref{fig:5.3compared} 
        & \autoref{fig:5.3iso_sacln_rew} \\
        & DF curves on agent encoders \newline co-trained with RL 
        & \autoref{fig:5.3compared} 
        & \autoref{fig:df_co_training_RL} \\
        & Pixel-based isolated evaluation\newline setting curves 
        & -- 
        & \autoref{fig:5.3iso_id_vidgray_agent} -- \autoref{fig:5.3iso_ood_vid} \\
        
        \midrule

        \multirow{2}{*}{\autoref{sec:5.4}} 
        & OOD generalization\newline reward curves 
        & \autoref{fig:5.4ood} 
        & \autoref{fig:ood_images_gray} -- \autoref{fig:ood_v} \\
        
        & Reward difference curves 
        & \autoref{fig:gen_gap} 
        & \autoref{fig:gengap_vg} -- \autoref{fig:gengap_v} \\
        
        \bottomrule
    \end{tabular}
    }
\end{table}
\clearpage

\begin{table}[t]
\vspace{-1em}
    \centering
    \caption{\textbf{Relative time spent on model updates} on NVIDIA L40S GPUs under the same task (walker/walk, with $\mathcal S = \mathbb R^{24}$ and $\Xi = \mathbb R^{32}$). Values represent the multiple of SAC's updating time. Key hyperparameters affecting the speed are set identically for all methods to \autoref{tab:hyperparams}.}
    \vspace{-0.5em}
    \resizebox{0.8\linewidth}{!}{%
        \begin{tabular}{lccccccc}
            \toprule
             & SAC & DeepMDP & DBC & DBC-normed & MICo & RAP & SimSR \\
            \midrule
            Pixel-based & 1.00 & 1.44 & 2.03 & 2.12 & 1.53 & 2.20 & 1.75 \\
            State-based & 1.00 & 1.42 & 1.76 & 1.95 & 1.39 & 2.08 & 1.68 \\
            \bottomrule
        \end{tabular}%
    }
    \label{tab:model_update_time}
\end{table}

\begin{figure}[h]
\vspace{-1em}
    \centering
\includegraphics[width=0.75\linewidth]{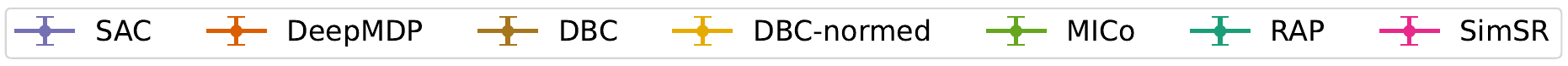} \\
\vspace{0cm}
\includegraphics[width=0.248\linewidth]{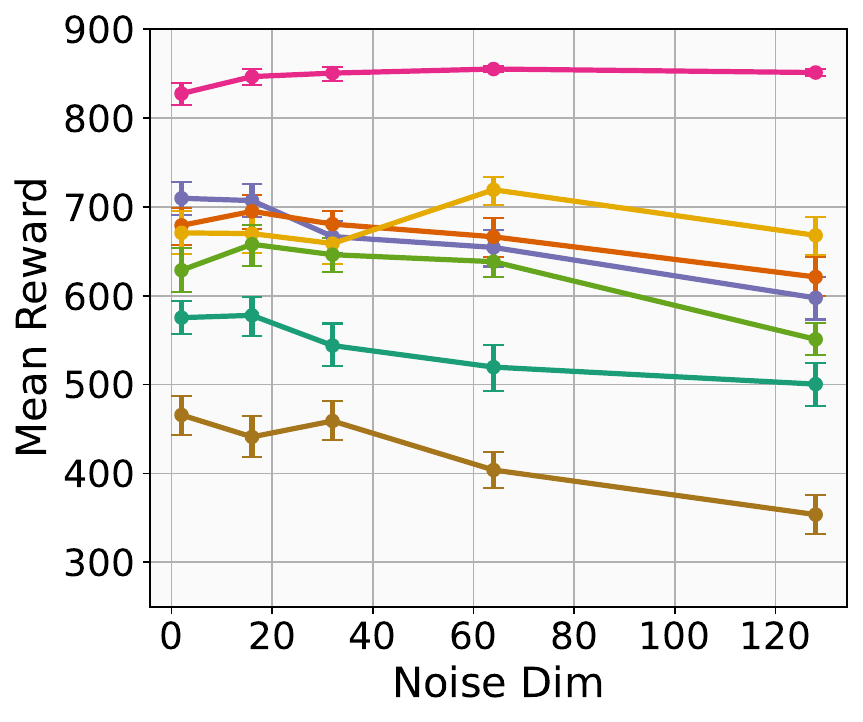}
\includegraphics[width=0.248\linewidth]{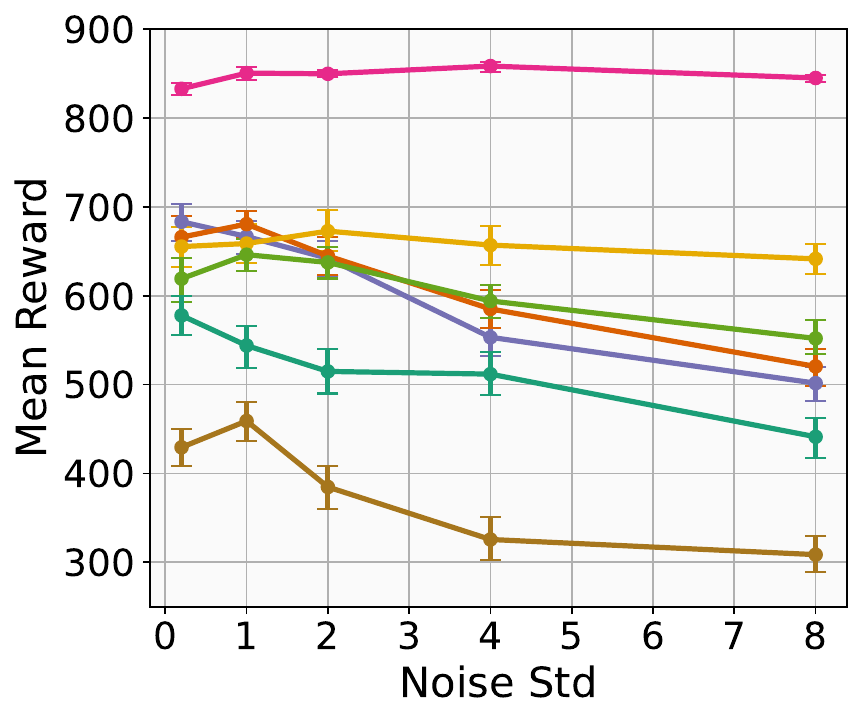}\includegraphics[width=0.248\linewidth]{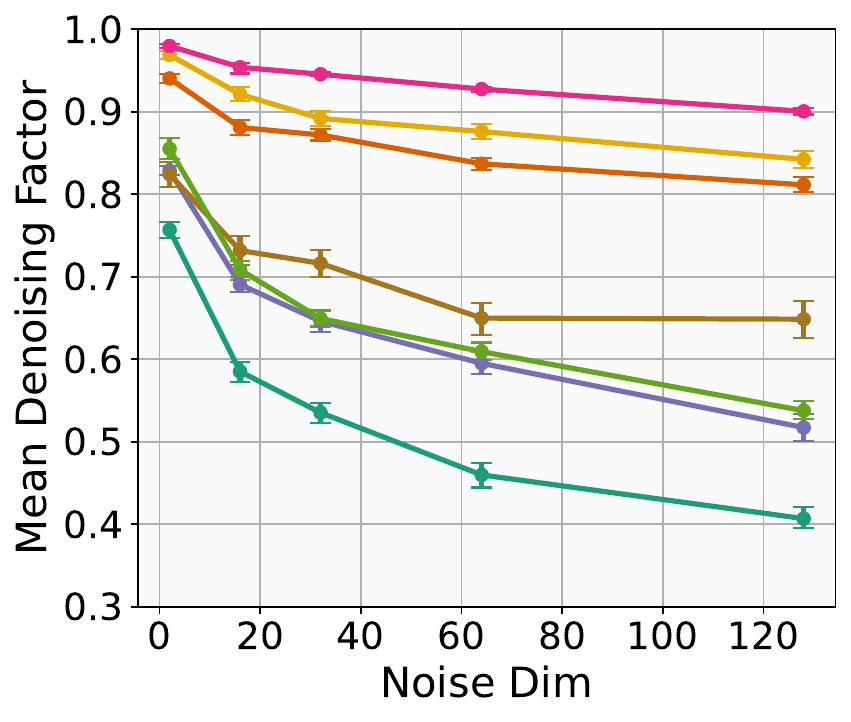}\includegraphics[width=0.248\linewidth]{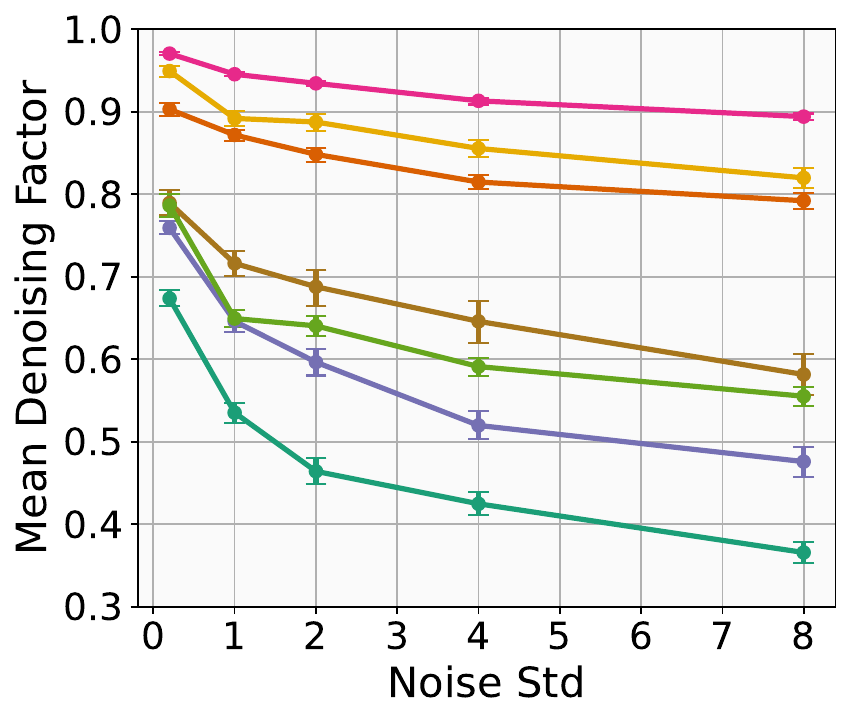}
    \vspace{-2em}
    \caption{\small \textbf{Benchmarking results: reward (left) and denoising factor (right)} of seven methods to IID Gaussian noise dimension (Noise Dim) and standard deviation (Noise Std). Each point is aggregated by 20 state-based tasks in \autoref{tab:state_levels}.}
    \vspace{-1em}
    \label{fig:df_sensitivity}
\end{figure}

\begin{table}[htbp]
\centering
\caption{Performance of state-based DMC tasks for the compared methods in noise setting: noise std=0.2, noise dim=32.}
\begin{adjustbox}{max width=\linewidth}
\begin{tabular}{lllllllll}
\toprule
Task &  & SAC & DeepMDP & DBC & DBC-normed & MICo & RAP & SimSR \\
\midrule
{\footnotesize ball\_in\_cup} & {\footnotesize catch} & 974.2 {\footnotesize $\pm\; 1.4$} & 966.6 {\footnotesize $\pm\; 8.2$} & 958.0 {\footnotesize $\pm\; 3.9$} & 975.0 {\footnotesize $\pm\; 1.8$} & 971.0 {\footnotesize $\pm\; 4.0$} & 883.5 {\footnotesize $\pm\; 114.0$} & \textbf{978.3 {\footnotesize $\pm\; 0.9$}} \\
{\footnotesize cartpole} & {\footnotesize balance} & 968.8 {\footnotesize $\pm\; 16.7$} & 967.6 {\footnotesize $\pm\; 9.9$} & 912.7 {\footnotesize $\pm\; 24.6$} & 988.7 {\footnotesize $\pm\; 7.4$} & 836.8 {\footnotesize $\pm\; 155.5$} & 870.9 {\footnotesize $\pm\; 172.0$} & \textbf{995.0 {\footnotesize $\pm\; 5.9$}} \\
{\footnotesize cartpole} & {\footnotesize balance\_sparse} & 815.9 {\footnotesize $\pm\; 182.7$} & 882.7 {\footnotesize $\pm\; 167.7$} & 790.2 {\footnotesize $\pm\; 143.2$} & 624.0 {\footnotesize $\pm\; 217.8$} & 579.4 {\footnotesize $\pm\; 282.6$} & 760.3 {\footnotesize $\pm\; 193.4$} & \textbf{991.9 {\footnotesize $\pm\; 17.8$}} \\
{\footnotesize walker} & {\footnotesize stand} & 959.2 {\footnotesize $\pm\; 28.5$} & 846.2 {\footnotesize $\pm\; 90.2$} & 621.9 {\footnotesize $\pm\; 164.8$} & 974.6 {\footnotesize $\pm\; 5.0$} & 929.3 {\footnotesize $\pm\; 30.3$} & 842.7 {\footnotesize $\pm\; 65.6$} & \textbf{977.4 {\footnotesize $\pm\; 4.3$}} \\
{\footnotesize cartpole} & {\footnotesize swingup} & 867.0 {\footnotesize $\pm\; 3.9$} & 839.4 {\footnotesize $\pm\; 13.0$} & 776.0 {\footnotesize $\pm\; 95.9$} & 582.5 {\footnotesize $\pm\; 203.8$} & 790.1 {\footnotesize $\pm\; 143.1$} & 872.7 {\footnotesize $\pm\; 2.3$} & \textbf{876.6 {\footnotesize $\pm\; 5.2$}} \\
\bottomrule
\end{tabular}
\end{adjustbox}
\label{tab:noise_std_0.2__noise_dim_32}
\end{table}

\begin{table}[htbp]
\centering
\caption{Performance of state-based DMC tasks for the compared methods in noise setting: noise std=1, noise dim=32.}
\begin{adjustbox}{max width=\linewidth}
\begin{tabular}{lllllllll}
\toprule
Task &  & SAC & DeepMDP & DBC & DBC-normed & MICo & RAP & SimSR \\
\midrule
{\footnotesize ball\_in\_cup} & {\footnotesize catch} & 972.5 {\footnotesize $\pm\; 2.2$} & 972.7 {\footnotesize $\pm\; 1.5$} & 881.7 {\footnotesize $\pm\; 157.0$} & 974.0 {\footnotesize $\pm\; 1.7$} & 969.1 {\footnotesize $\pm\; 3.7$} & 940.7 {\footnotesize $\pm\; 56.6$} & \textbf{978.5 {\footnotesize $\pm\; 0.8$}} \\
{\footnotesize cartpole} & {\footnotesize balance} & 964.7 {\footnotesize $\pm\; 22.2$} & 948.6 {\footnotesize $\pm\; 21.6$} & 863.2 {\footnotesize $\pm\; 31.4$} & 982.3 {\footnotesize $\pm\; 14.5$} & 972.7 {\footnotesize $\pm\; 10.2$} & 616.4 {\footnotesize $\pm\; 262.2$} & \textbf{999.2 {\footnotesize $\pm\; 0.8$}} \\
{\footnotesize cartpole} & {\footnotesize balance\_sparse} & 919.7 {\footnotesize $\pm\; 59.9$} & 934.1 {\footnotesize $\pm\; 45.5$} & 843.9 {\footnotesize $\pm\; 93.4$} & 908.6 {\footnotesize $\pm\; 68.6$} & 953.7 {\footnotesize $\pm\; 27.6$} & 703.9 {\footnotesize $\pm\; 136.3$} & \textbf{990.7 {\footnotesize $\pm\; 11.7$}} \\
{\footnotesize walker} & {\footnotesize stand} & 968.7 {\footnotesize $\pm\; 5.1$} & 781.8 {\footnotesize $\pm\; 122.4$} & 527.4 {\footnotesize $\pm\; 166.8$} & 974.6 {\footnotesize $\pm\; 3.9$} & 956.6 {\footnotesize $\pm\; 9.9$} & 827.2 {\footnotesize $\pm\; 114.0$} & \textbf{977.1 {\footnotesize $\pm\; 3.6$}} \\
{\footnotesize cartpole} & {\footnotesize swingup} & \textbf{865.7 {\footnotesize $\pm\; 1.9$}} & 837.1 {\footnotesize $\pm\; 9.1$} & 824.7 {\footnotesize $\pm\; 22.7$} & 741.2 {\footnotesize $\pm\; 99.6$} & 770.3 {\footnotesize $\pm\; 100.2$} & 829.9 {\footnotesize $\pm\; 78.8$} & 865.3 {\footnotesize $\pm\; 6.4$} \\
\bottomrule
\end{tabular}
\end{adjustbox}
\label{tab:noise_std_1__noise_dim_32}
\end{table}

\begin{table}[htbp]
\centering
\caption{Performance of state-based DMC tasks for the compared methods in noise setting: noise std=2, noise dim=32.}
\begin{adjustbox}{max width=\linewidth}
\begin{tabular}{lllllllll}
\toprule
Task &  & SAC & DeepMDP & DBC & DBC-normed & MICo & RAP & SimSR \\
\midrule
{\footnotesize ball\_in\_cup} & {\footnotesize catch} & 973.9 {\footnotesize $\pm\; 1.2$} & 972.0 {\footnotesize $\pm\; 2.0$} & 938.0 {\footnotesize $\pm\; 16.3$} & 974.9 {\footnotesize $\pm\; 1.4$} & 968.9 {\footnotesize $\pm\; 3.0$} & 908.7 {\footnotesize $\pm\; 78.4$} & \textbf{978.2 {\footnotesize $\pm\; 0.6$}} \\
{\footnotesize cartpole} & {\footnotesize balance} & 962.1 {\footnotesize $\pm\; 21.2$} & 948.8 {\footnotesize $\pm\; 25.3$} & 892.5 {\footnotesize $\pm\; 44.6$} & 982.4 {\footnotesize $\pm\; 11.8$} & 956.9 {\footnotesize $\pm\; 28.5$} & 768.2 {\footnotesize $\pm\; 242.0$} & \textbf{997.8 {\footnotesize $\pm\; 4.3$}} \\
{\footnotesize cartpole} & {\footnotesize balance\_sparse} & 973.1 {\footnotesize $\pm\; 31.2$} & 973.6 {\footnotesize $\pm\; 24.0$} & 684.2 {\footnotesize $\pm\; 234.4$} & 962.7 {\footnotesize $\pm\; 45.4$} & 956.2 {\footnotesize $\pm\; 48.1$} & 755.7 {\footnotesize $\pm\; 159.8$} & \textbf{990.5 {\footnotesize $\pm\; 17.5$}} \\
{\footnotesize walker} & {\footnotesize stand} & 965.2 {\footnotesize $\pm\; 9.5$} & 763.8 {\footnotesize $\pm\; 116.7$} & 416.2 {\footnotesize $\pm\; 166.5$} & 969.5 {\footnotesize $\pm\; 15.1$} & 930.5 {\footnotesize $\pm\; 32.5$} & 853.4 {\footnotesize $\pm\; 89.4$} & \textbf{977.9 {\footnotesize $\pm\; 7.2$}} \\
{\footnotesize cartpole} & {\footnotesize swingup} & 854.6 {\footnotesize $\pm\; 6.6$} & 829.4 {\footnotesize $\pm\; 11.2$} & 711.8 {\footnotesize $\pm\; 144.2$} & 523.6 {\footnotesize $\pm\; 199.3$} & 852.4 {\footnotesize $\pm\; 7.9$} & 864.4 {\footnotesize $\pm\; 9.6$} & \textbf{870.3 {\footnotesize $\pm\; 8.6$}} \\
\bottomrule
\end{tabular}
\end{adjustbox}
\label{tab:noise_std_2__noise_dim_32}
\end{table}

\begin{table}[htbp]
\centering
\caption{Performance of state-based DMC tasks for the compared methods in noise setting: noise std=4, noise dim=32.}
\begin{adjustbox}{max width=\linewidth}
\begin{tabular}{lllllllll}
\toprule
Task &  & SAC & DeepMDP & DBC & DBC-normed & MICo & RAP & SimSR \\
\midrule
{\footnotesize ball\_in\_cup} & {\footnotesize catch} & 973.9 {\footnotesize $\pm\; 1.8$} & 960.3 {\footnotesize $\pm\; 15.7$} & 918.6 {\footnotesize $\pm\; 41.8$} & 972.2 {\footnotesize $\pm\; 4.3$} & 967.0 {\footnotesize $\pm\; 3.5$} & 723.9 {\footnotesize $\pm\; 139.3$} & \textbf{977.2 {\footnotesize $\pm\; 1.1$}} \\
{\footnotesize cartpole} & {\footnotesize balance} & 969.2 {\footnotesize $\pm\; 11.5$} & 931.9 {\footnotesize $\pm\; 29.9$} & 815.1 {\footnotesize $\pm\; 91.9$} & 977.8 {\footnotesize $\pm\; 19.4$} & 957.9 {\footnotesize $\pm\; 13.2$} & 901.4 {\footnotesize $\pm\; 171.1$} & \textbf{996.4 {\footnotesize $\pm\; 5.7$}} \\
{\footnotesize cartpole} & {\footnotesize balance\_sparse} & \textbf{996.2 {\footnotesize $\pm\; 3.6$}} & 899.8 {\footnotesize $\pm\; 86.1$} & 831.1 {\footnotesize $\pm\; 158.1$} & 932.6 {\footnotesize $\pm\; 103.1$} & 917.9 {\footnotesize $\pm\; 50.4$} & 830.1 {\footnotesize $\pm\; 131.0$} & 991.1 {\footnotesize $\pm\; 17.2$} \\
{\footnotesize walker} & {\footnotesize stand} & 951.9 {\footnotesize $\pm\; 19.4$} & 641.9 {\footnotesize $\pm\; 146.9$} & 394.7 {\footnotesize $\pm\; 180.1$} & 971.1 {\footnotesize $\pm\; 6.4$} & 929.3 {\footnotesize $\pm\; 28.9$} & 875.1 {\footnotesize $\pm\; 57.7$} & \textbf{978.8 {\footnotesize $\pm\; 5.8$}} \\
{\footnotesize cartpole} & {\footnotesize swingup} & 858.5 {\footnotesize $\pm\; 5.6$} & 823.9 {\footnotesize $\pm\; 15.7$} & 770.4 {\footnotesize $\pm\; 34.2$} & 684.2 {\footnotesize $\pm\; 176.0$} & 850.4 {\footnotesize $\pm\; 5.4$} & 856.4 {\footnotesize $\pm\; 21.9$} & \textbf{881.2 {\footnotesize $\pm\; 0.5$}} \\
\bottomrule
\end{tabular}
\end{adjustbox}
\label{tab:noise_std_4__noise_dim_32}
\end{table}

\begin{table}[htbp]
\centering
\caption{Performance of state-based DMC tasks for the compared methods in noise setting: noise std=8, noise dim=32.}
\begin{adjustbox}{max width=\linewidth}
\begin{tabular}{lllllllll}
\toprule
Task &  & SAC & DeepMDP & DBC & DBC-normed & MICo & RAP & SimSR \\
\midrule
{\footnotesize ball\_in\_cup} & {\footnotesize catch} & 924.8 {\footnotesize $\pm\; 95.5$} & 957.7 {\footnotesize $\pm\; 22.7$} & 924.7 {\footnotesize $\pm\; 20.2$} & 969.3 {\footnotesize $\pm\; 4.8$} & 965.2 {\footnotesize $\pm\; 3.9$} & 760.4 {\footnotesize $\pm\; 142.3$} & \textbf{977.0 {\footnotesize $\pm\; 1.0$}} \\
{\footnotesize cartpole} & {\footnotesize balance} & 967.5 {\footnotesize $\pm\; 12.3$} & 928.7 {\footnotesize $\pm\; 32.3$} & 814.1 {\footnotesize $\pm\; 86.6$} & 973.7 {\footnotesize $\pm\; 12.4$} & 966.6 {\footnotesize $\pm\; 9.2$} & 950.3 {\footnotesize $\pm\; 71.2$} & \textbf{999.5 {\footnotesize $\pm\; 0.5$}} \\
{\footnotesize cartpole} & {\footnotesize balance\_sparse} & 984.7 {\footnotesize $\pm\; 11.4$} & 823.5 {\footnotesize $\pm\; 160.2$} & 646.6 {\footnotesize $\pm\; 200.8$} & 982.3 {\footnotesize $\pm\; 29.3$} & 935.6 {\footnotesize $\pm\; 34.8$} & 784.0 {\footnotesize $\pm\; 139.7$} & \textbf{985.0 {\footnotesize $\pm\; 15.4$}} \\
{\footnotesize walker} & {\footnotesize stand} & 968.6 {\footnotesize $\pm\; 8.1$} & 522.5 {\footnotesize $\pm\; 154.5$} & 265.4 {\footnotesize $\pm\; 118.4$} & 975.8 {\footnotesize $\pm\; 5.1$} & 932.2 {\footnotesize $\pm\; 13.8$} & 796.8 {\footnotesize $\pm\; 63.3$} & \textbf{978.5 {\footnotesize $\pm\; 5.5$}} \\
{\footnotesize cartpole} & {\footnotesize swingup} & 858.1 {\footnotesize $\pm\; 5.7$} & 828.6 {\footnotesize $\pm\; 11.5$} & 778.1 {\footnotesize $\pm\; 28.9$} & 838.7 {\footnotesize $\pm\; 9.4$} & 851.9 {\footnotesize $\pm\; 8.6$} & 865.1 {\footnotesize $\pm\; 21.4$} & \textbf{878.8 {\footnotesize $\pm\; 4.9$}} \\
\bottomrule
\end{tabular}
\end{adjustbox}
\label{tab:noise_std_8__noise_dim_32}
\end{table}

\begin{table}[htbp]
\centering
\caption{Performance of state-based DMC tasks for the compared methods in noise setting: noise std=1, noise dim=2.}
\begin{adjustbox}{max width=\linewidth}
\begin{tabular}{lllllllll}
\toprule
Task &  & SAC & DeepMDP & DBC & DBC-normed & MICo & RAP & SimSR \\
\midrule
{\footnotesize ball\_in\_cup} & {\footnotesize catch} & 975.4 {\footnotesize $\pm\; 1.2$} & 970.1 {\footnotesize $\pm\; 5.3$} & 935.1 {\footnotesize $\pm\; 19.2$} & 973.5 {\footnotesize $\pm\; 2.7$} & 973.5 {\footnotesize $\pm\; 2.1$} & 973.2 {\footnotesize $\pm\; 3.2$} & \textbf{978.0 {\footnotesize $\pm\; 1.0$}} \\
{\footnotesize cartpole} & {\footnotesize balance} & 950.5 {\footnotesize $\pm\; 18.7$} & 935.2 {\footnotesize $\pm\; 27.7$} & 894.1 {\footnotesize $\pm\; 36.9$} & 983.9 {\footnotesize $\pm\; 22.2$} & 648.2 {\footnotesize $\pm\; 248.1$} & 934.6 {\footnotesize $\pm\; 56.1$} & \textbf{990.8 {\footnotesize $\pm\; 18.6$}} \\
{\footnotesize cartpole} & {\footnotesize balance\_sparse} & 932.6 {\footnotesize $\pm\; 80.2$} & \textbf{976.8 {\footnotesize $\pm\; 19.6$}} & 890.2 {\footnotesize $\pm\; 81.4$} & 861.6 {\footnotesize $\pm\; 124.5$} & 548.1 {\footnotesize $\pm\; 204.1$} & 885.9 {\footnotesize $\pm\; 48.8$} & 949.6 {\footnotesize $\pm\; 85.3$} \\
{\footnotesize walker} & {\footnotesize stand} & 957.6 {\footnotesize $\pm\; 17.1$} & 903.4 {\footnotesize $\pm\; 38.6$} & 737.8 {\footnotesize $\pm\; 123.8$} & 970.3 {\footnotesize $\pm\; 8.3$} & 951.9 {\footnotesize $\pm\; 9.4$} & 907.9 {\footnotesize $\pm\; 31.3$} & \textbf{976.6 {\footnotesize $\pm\; 9.3$}} \\
{\footnotesize cartpole} & {\footnotesize swingup} & 862.9 {\footnotesize $\pm\; 10.2$} & 844.1 {\footnotesize $\pm\; 14.8$} & 806.9 {\footnotesize $\pm\; 42.5$} & 757.6 {\footnotesize $\pm\; 118.0$} & 835.2 {\footnotesize $\pm\; 36.0$} & \textbf{874.8 {\footnotesize $\pm\; 1.6$}} & 873.9 {\footnotesize $\pm\; 5.5$} \\
\bottomrule
\end{tabular}
\end{adjustbox}
\label{tab:noise_std_1__noise_dim_2}
\end{table}

\begin{table}[htbp]
\centering
\caption{Performance of state-based DMC tasks for the compared methods in noise setting: noise std=1, noise dim=16.}
\begin{adjustbox}{max width=\linewidth}
\begin{tabular}{lllllllll}
\toprule
Task &  & SAC & DeepMDP & DBC & DBC-normed & MICo & RAP & SimSR \\
\midrule
{\footnotesize ball\_in\_cup} & {\footnotesize catch} & 974.4 {\footnotesize $\pm\; 1.8$} & 966.5 {\footnotesize $\pm\; 14.1$} & 956.9 {\footnotesize $\pm\; 4.9$} & 975.5 {\footnotesize $\pm\; 1.3$} & 971.9 {\footnotesize $\pm\; 2.3$} & 966.5 {\footnotesize $\pm\; 10.6$} & \textbf{978.1 {\footnotesize $\pm\; 0.8$}} \\
{\footnotesize cartpole} & {\footnotesize balance} & 958.0 {\footnotesize $\pm\; 31.8$} & 933.1 {\footnotesize $\pm\; 17.9$} & 865.9 {\footnotesize $\pm\; 56.5$} & 990.9 {\footnotesize $\pm\; 5.0$} & 932.1 {\footnotesize $\pm\; 58.2$} & 768.3 {\footnotesize $\pm\; 210.6$} & \textbf{996.4 {\footnotesize $\pm\; 5.6$}} \\
{\footnotesize cartpole} & {\footnotesize balance\_sparse} & 962.2 {\footnotesize $\pm\; 31.3$} & 937.8 {\footnotesize $\pm\; 38.4$} & 764.4 {\footnotesize $\pm\; 120.6$} & 727.8 {\footnotesize $\pm\; 135.0$} & 820.0 {\footnotesize $\pm\; 134.4$} & 830.0 {\footnotesize $\pm\; 79.4$} & \textbf{987.1 {\footnotesize $\pm\; 12.8$}} \\
{\footnotesize walker} & {\footnotesize stand} & 968.9 {\footnotesize $\pm\; 6.8$} & 819.0 {\footnotesize $\pm\; 142.4$} & 514.7 {\footnotesize $\pm\; 122.9$} & 976.5 {\footnotesize $\pm\; 3.3$} & 940.2 {\footnotesize $\pm\; 14.5$} & 887.8 {\footnotesize $\pm\; 58.6$} & \textbf{981.3 {\footnotesize $\pm\; 2.8$}} \\
{\footnotesize cartpole} & {\footnotesize swingup} & 865.3 {\footnotesize $\pm\; 5.2$} & 827.4 {\footnotesize $\pm\; 20.2$} & 756.9 {\footnotesize $\pm\; 139.8$} & 794.8 {\footnotesize $\pm\; 113.9$} & 729.5 {\footnotesize $\pm\; 185.0$} & 869.9 {\footnotesize $\pm\; 7.0$} & \textbf{870.0 {\footnotesize $\pm\; 7.1$}} \\
\bottomrule
\end{tabular}
\end{adjustbox}
\label{tab:noise_std_1__noise_dim_16}
\end{table}

\begin{table}[htbp]
\centering
\caption{Performance of state-based DMC tasks for the compared methods in noise setting: noise std=1, noise dim=64.}
\begin{adjustbox}{max width=\linewidth}
\begin{tabular}{lllllllll}
\toprule
Task &  & SAC & DeepMDP & DBC & DBC-normed & MICo & RAP & SimSR \\
\midrule
{\footnotesize ball\_in\_cup} & {\footnotesize catch} & 853.7 {\footnotesize $\pm\; 176.4$} & 970.6 {\footnotesize $\pm\; 3.1$} & 810.6 {\footnotesize $\pm\; 189.0$} & 974.7 {\footnotesize $\pm\; 1.4$} & 965.1 {\footnotesize $\pm\; 2.7$} & 903.0 {\footnotesize $\pm\; 78.8$} & \textbf{978.2 {\footnotesize $\pm\; 1.0$}} \\
{\footnotesize cartpole} & {\footnotesize balance} & 978.8 {\footnotesize $\pm\; 18.9$} & 932.7 {\footnotesize $\pm\; 29.9$} & 889.0 {\footnotesize $\pm\; 37.5$} & 980.5 {\footnotesize $\pm\; 16.9$} & 967.2 {\footnotesize $\pm\; 11.1$} & 506.0 {\footnotesize $\pm\; 259.4$} & \textbf{999.8 {\footnotesize $\pm\; 0.1$}} \\
{\footnotesize cartpole} & {\footnotesize balance\_sparse} & 875.1 {\footnotesize $\pm\; 117.6$} & 927.2 {\footnotesize $\pm\; 79.9$} & 885.8 {\footnotesize $\pm\; 36.0$} & 965.1 {\footnotesize $\pm\; 37.6$} & 938.5 {\footnotesize $\pm\; 35.3$} & 741.4 {\footnotesize $\pm\; 120.5$} & \textbf{989.9 {\footnotesize $\pm\; 13.1$}} \\
{\footnotesize walker} & {\footnotesize stand} & 965.9 {\footnotesize $\pm\; 4.4$} & 698.4 {\footnotesize $\pm\; 163.0$} & 384.4 {\footnotesize $\pm\; 117.1$} & 970.7 {\footnotesize $\pm\; 9.7$} & 923.9 {\footnotesize $\pm\; 29.8$} & 776.7 {\footnotesize $\pm\; 63.6$} & \textbf{980.9 {\footnotesize $\pm\; 3.3$}} \\
{\footnotesize cartpole} & {\footnotesize swingup} & 855.8 {\footnotesize $\pm\; 3.9$} & 836.5 {\footnotesize $\pm\; 12.7$} & 726.1 {\footnotesize $\pm\; 142.4$} & 770.9 {\footnotesize $\pm\; 108.4$} & 853.0 {\footnotesize $\pm\; 5.7$} & \textbf{874.1 {\footnotesize $\pm\; 9.2$}} & 870.7 {\footnotesize $\pm\; 6.0$} \\
\bottomrule
\end{tabular}
\end{adjustbox}
\label{tab:noise_std_1__noise_dim_64}
\end{table}

\begin{table}[htbp]
\centering
\caption{Performance of state-based DMC tasks for the compared methods in noise setting: noise std=1, noise dim=128.}
\begin{adjustbox}{max width=\linewidth}
\begin{tabular}{lllllllll}
\toprule
Task &  & SAC & DeepMDP & DBC & DBC-normed & MICo & RAP & SimSR \\
\midrule
{\footnotesize ball\_in\_cup} & {\footnotesize catch} & 817.5 {\footnotesize $\pm\; 216.7$} & 966.6 {\footnotesize $\pm\; 8.1$} & 602.4 {\footnotesize $\pm\; 260.0$} & 969.4 {\footnotesize $\pm\; 7.6$} & 964.0 {\footnotesize $\pm\; 2.3$} & 793.0 {\footnotesize $\pm\; 143.8$} & \textbf{973.7 {\footnotesize $\pm\; 1.5$}} \\
{\footnotesize cartpole} & {\footnotesize balance} & 926.5 {\footnotesize $\pm\; 47.6$} & 943.7 {\footnotesize $\pm\; 23.0$} & 868.8 {\footnotesize $\pm\; 55.3$} & 982.9 {\footnotesize $\pm\; 10.4$} & 818.9 {\footnotesize $\pm\; 203.4$} & 768.9 {\footnotesize $\pm\; 222.2$} & \textbf{998.9 {\footnotesize $\pm\; 1.7$}} \\
{\footnotesize cartpole} & {\footnotesize balance\_sparse} & 962.6 {\footnotesize $\pm\; 46.1$} & 865.8 {\footnotesize $\pm\; 128.5$} & 631.0 {\footnotesize $\pm\; 158.5$} & 956.1 {\footnotesize $\pm\; 33.5$} & 972.9 {\footnotesize $\pm\; 19.2$} & 671.7 {\footnotesize $\pm\; 147.7$} & \textbf{990.4 {\footnotesize $\pm\; 11.7$}} \\
{\footnotesize walker} & {\footnotesize stand} & 953.3 {\footnotesize $\pm\; 18.0$} & 765.7 {\footnotesize $\pm\; 128.3$} & 315.8 {\footnotesize $\pm\; 91.8$} & 972.1 {\footnotesize $\pm\; 9.2$} & 925.4 {\footnotesize $\pm\; 18.7$} & 863.1 {\footnotesize $\pm\; 72.6$} & \textbf{980.8 {\footnotesize $\pm\; 2.1$}} \\
{\footnotesize cartpole} & {\footnotesize swingup} & 854.4 {\footnotesize $\pm\; 8.7$} & 835.5 {\footnotesize $\pm\; 8.2$} & 779.2 {\footnotesize $\pm\; 27.6$} & 641.8 {\footnotesize $\pm\; 194.5$} & 849.1 {\footnotesize $\pm\; 8.6$} & 864.6 {\footnotesize $\pm\; 4.8$} & \textbf{874.0 {\footnotesize $\pm\; 5.0$}} \\
\bottomrule
\end{tabular}
\end{adjustbox}
\label{tab:noise_std_1__noise_dim_128}
\end{table}

\begin{table}[htbp]
\centering
\caption{Performance of pixel-based DMC tasks for the compared methods with original clean background.}
\begin{adjustbox}{max width=\linewidth}
\begin{tabular}{lllllllll}
\toprule
Task &  & SAC & DeepMDP & DBC & DBC-normed & MICo & RAP & SimSR \\
\midrule
{\footnotesize cartpole} & {\footnotesize balance} & \textbf{994.5 {\footnotesize $\pm\; 5.2$}} & 959.2 {\footnotesize $\pm\; 36.7$} & 957.6 {\footnotesize $\pm\; 38.8$} & 935.7 {\footnotesize $\pm\; 71.0$} & 981.9 {\footnotesize $\pm\; 7.8$} & 990.6 {\footnotesize $\pm\; 8.0$} & 985.1 {\footnotesize $\pm\; 13.0$} \\
{\footnotesize cartpole} & {\footnotesize balance\_sparse} & 978.7 {\footnotesize $\pm\; 46.1$} & 928.9 {\footnotesize $\pm\; 86.9$} & 604.7 {\footnotesize $\pm\; 618.1$} & \textbf{998.9 {\footnotesize $\pm\; 2.5$}} & 990.8 {\footnotesize $\pm\; 23.2$} & 990.9 {\footnotesize $\pm\; 21.5$} & 994.7 {\footnotesize $\pm\; 7.9$} \\
{\footnotesize walker} & {\footnotesize stand} & 968.6 {\footnotesize $\pm\; 9.0$} & 961.1 {\footnotesize $\pm\; 20.0$} & 923.8 {\footnotesize $\pm\; 63.4$} & 963.2 {\footnotesize $\pm\; 8.6$} & 968.3 {\footnotesize $\pm\; 7.7$} & \textbf{971.1 {\footnotesize $\pm\; 6.3$}} & 875.2 {\footnotesize $\pm\; 219.7$} \\
{\footnotesize finger} & {\footnotesize spin} & 918.2 {\footnotesize $\pm\; 173.4$} & \textbf{973.9 {\footnotesize $\pm\; 29.9$}} & 837.5 {\footnotesize $\pm\; 197.7$} & 913.5 {\footnotesize $\pm\; 179.2$} & 845.5 {\footnotesize $\pm\; 231.1$} & 967.1 {\footnotesize $\pm\; 40.7$} & 918.6 {\footnotesize $\pm\; 181.1$} \\
{\footnotesize cartpole} & {\footnotesize swingup} & 855.3 {\footnotesize $\pm\; 3.1$} & \textbf{859.7 {\footnotesize $\pm\; 16.9$}} & 845.4 {\footnotesize $\pm\; 40.6$} & 827.8 {\footnotesize $\pm\; 15.8$} & 838.5 {\footnotesize $\pm\; 50.1$} & 859.2 {\footnotesize $\pm\; 2.8$} & 826.0 {\footnotesize $\pm\; 56.1$} \\
\bottomrule
\end{tabular}
\end{adjustbox}
\label{tab:img_source_none}
\end{table}

\begin{table}[htbp]
\centering
\caption{Performance of pixel-based DMC tasks for the compared methods with grayscale images background.}
\begin{adjustbox}{max width=\linewidth}
\begin{tabular}{lllllllll}
\toprule
Task &  & SAC & DeepMDP & DBC & DBC-normed & MICo & RAP & SimSR \\
\midrule
{\footnotesize cartpole} & {\footnotesize balance} & 985.8 {\footnotesize $\pm\; 5.8$} & 981.6 {\footnotesize $\pm\; 20.0$} & 729.1 {\footnotesize $\pm\; 347.2$} & 856.4 {\footnotesize $\pm\; 299.6$} & 853.6 {\footnotesize $\pm\; 297.8$} & \textbf{989.1 {\footnotesize $\pm\; 7.8$}} & 983.1 {\footnotesize $\pm\; 10.3$} \\
{\footnotesize cartpole} & {\footnotesize balance\_sparse} & 998.7 {\footnotesize $\pm\; 2.9$} & \textbf{999.7 {\footnotesize $\pm\; 0.4$}} & 752.3 {\footnotesize $\pm\; 520.5$} & 761.8 {\footnotesize $\pm\; 753.8$} & 998.5 {\footnotesize $\pm\; 1.8$} & 995.0 {\footnotesize $\pm\; 15.2$} & 986.0 {\footnotesize $\pm\; 21.1$} \\
{\footnotesize walker} & {\footnotesize stand} & 955.7 {\footnotesize $\pm\; 23.0$} & 950.5 {\footnotesize $\pm\; 38.0$} & 822.1 {\footnotesize $\pm\; 189.7$} & 922.8 {\footnotesize $\pm\; 84.9$} & 952.6 {\footnotesize $\pm\; 13.5$} & \textbf{961.4 {\footnotesize $\pm\; 10.0$}} & 952.0 {\footnotesize $\pm\; 20.2$} \\
{\footnotesize finger} & {\footnotesize spin} & 972.2 {\footnotesize $\pm\; 16.6$} & 926.4 {\footnotesize $\pm\; 155.9$} & 558.5 {\footnotesize $\pm\; 632.9$} & 705.9 {\footnotesize $\pm\; 499.9$} & 906.8 {\footnotesize $\pm\; 210.6$} & 955.1 {\footnotesize $\pm\; 66.0$} & \textbf{982.8 {\footnotesize $\pm\; 4.8$}} \\
{\footnotesize cartpole} & {\footnotesize swingup} & 856.0 {\footnotesize $\pm\; 13.3$} & 842.1 {\footnotesize $\pm\; 33.8$} & 818.4 {\footnotesize $\pm\; 39.6$} & 764.8 {\footnotesize $\pm\; 110.6$} & 842.9 {\footnotesize $\pm\; 25.9$} & \textbf{859.3 {\footnotesize $\pm\; 13.1$}} & 492.2 {\footnotesize $\pm\; 402.0$} \\
\bottomrule
\end{tabular}
\end{adjustbox}
\label{tab:img_source_images_gray}
\end{table}

\begin{table}[htbp]
\centering
\caption{Performance of pixel-based DMC tasks for the compared methods with colored images background.}
\begin{adjustbox}{max width=\linewidth}
\begin{tabular}{lllllllll}
\toprule
Task &  & SAC & DeepMDP & DBC & DBC-normed & MICo & RAP & SimSR \\
\midrule
{\footnotesize cartpole} & {\footnotesize balance} & \textbf{984.1 {\footnotesize $\pm\; 8.9$}} & 977.8 {\footnotesize $\pm\; 13.1$} & 978.4 {\footnotesize $\pm\; 13.3$} & 924.4 {\footnotesize $\pm\; 54.5$} & 978.6 {\footnotesize $\pm\; 15.2$} & 972.4 {\footnotesize $\pm\; 21.8$} & 978.7 {\footnotesize $\pm\; 17.2$} \\
{\footnotesize cartpole} & {\footnotesize balance\_sparse} & 989.8 {\footnotesize $\pm\; 15.1$} & \textbf{999.2 {\footnotesize $\pm\; 1.1$}} & 787.1 {\footnotesize $\pm\; 500.9$} & 792.4 {\footnotesize $\pm\; 502.7$} & 996.1 {\footnotesize $\pm\; 5.6$} & 987.2 {\footnotesize $\pm\; 29.5$} & 994.7 {\footnotesize $\pm\; 9.3$} \\
{\footnotesize walker} & {\footnotesize stand} & 941.3 {\footnotesize $\pm\; 53.9$} & 949.7 {\footnotesize $\pm\; 47.8$} & 883.3 {\footnotesize $\pm\; 166.4$} & 958.7 {\footnotesize $\pm\; 9.7$} & 938.8 {\footnotesize $\pm\; 57.4$} & \textbf{963.4 {\footnotesize $\pm\; 8.1$}} & 948.4 {\footnotesize $\pm\; 26.4$} \\
{\footnotesize finger} & {\footnotesize spin} & 975.3 {\footnotesize $\pm\; 17.4$} & 902.4 {\footnotesize $\pm\; 207.2$} & 489.8 {\footnotesize $\pm\; 583.7$} & 670.4 {\footnotesize $\pm\; 504.0$} & 981.8 {\footnotesize $\pm\; 2.9$} & 786.8 {\footnotesize $\pm\; 546.2$} & \textbf{984.4 {\footnotesize $\pm\; 4.0$}} \\
{\footnotesize cartpole} & {\footnotesize swingup} & \textbf{850.4 {\footnotesize $\pm\; 9.2$}} & 844.0 {\footnotesize $\pm\; 41.5$} & 827.6 {\footnotesize $\pm\; 20.5$} & 805.0 {\footnotesize $\pm\; 28.2$} & 837.1 {\footnotesize $\pm\; 23.9$} & 845.8 {\footnotesize $\pm\; 34.7$} & 708.3 {\footnotesize $\pm\; 371.2$} \\
\bottomrule
\end{tabular}
\end{adjustbox}
\label{tab:img_source_images}
\end{table}

\begin{table}[htbp]
\centering
\caption{Performance of pixel-based DMC tasks for the compared methods with grayscale video background.}
\begin{adjustbox}{max width=\linewidth}
\begin{tabular}{lllllllll}
\toprule
Task &  & SAC & DeepMDP & DBC & DBC-normed & MICo & RAP & SimSR \\
\midrule
{\footnotesize cartpole} & {\footnotesize balance} & 976.9 {\footnotesize $\pm\; 13.6$} & \textbf{987.8 {\footnotesize $\pm\; 8.2$}} & 950.9 {\footnotesize $\pm\; 24.7$} & 880.2 {\footnotesize $\pm\; 271.4$} & 973.8 {\footnotesize $\pm\; 16.2$} & 979.1 {\footnotesize $\pm\; 18.2$} & 956.3 {\footnotesize $\pm\; 59.8$} \\
{\footnotesize cartpole} & {\footnotesize balance\_sparse} & 992.8 {\footnotesize $\pm\; 11.9$} & 995.6 {\footnotesize $\pm\; 10.7$} & 866.2 {\footnotesize $\pm\; 126.0$} & 802.9 {\footnotesize $\pm\; 547.4$} & \textbf{996.5 {\footnotesize $\pm\; 2.9$}} & 974.5 {\footnotesize $\pm\; 41.3$} & 990.6 {\footnotesize $\pm\; 13.2$} \\
{\footnotesize walker} & {\footnotesize stand} & 961.5 {\footnotesize $\pm\; 3.7$} & 946.6 {\footnotesize $\pm\; 64.9$} & 643.8 {\footnotesize $\pm\; 253.2$} & 961.1 {\footnotesize $\pm\; 3.1$} & 947.8 {\footnotesize $\pm\; 25.7$} & \textbf{971.0 {\footnotesize $\pm\; 3.9$}} & 957.0 {\footnotesize $\pm\; 8.7$} \\
{\footnotesize finger} & {\footnotesize spin} & 976.9 {\footnotesize $\pm\; 11.8$} & \textbf{982.7 {\footnotesize $\pm\; 2.9$}} & 0.0 {\footnotesize $\pm\; 0.0$} & 974.8 {\footnotesize $\pm\; 13.1$} & 979.9 {\footnotesize $\pm\; 7.7$} & 971.4 {\footnotesize $\pm\; 21.4$} & 902.6 {\footnotesize $\pm\; 261.3$} \\
{\footnotesize cartpole} & {\footnotesize swingup} & 854.9 {\footnotesize $\pm\; 24.2$} & 866.5 {\footnotesize $\pm\; 10.9$} & 826.8 {\footnotesize $\pm\; 43.7$} & 410.4 {\footnotesize $\pm\; 410.3$} & 854.2 {\footnotesize $\pm\; 5.8$} & 863.5 {\footnotesize $\pm\; 7.8$} & \textbf{868.0 {\footnotesize $\pm\; 22.2$}} \\
\bottomrule
\end{tabular}
\end{adjustbox}
\label{tab:img_source_video_gray}
\end{table}

\begin{table}[htbp]
\centering
\caption{Performance of pixel-based DMC tasks for the compared methods with colored video background.}
\begin{adjustbox}{max width=\linewidth}
\begin{tabular}{lllllllll}
\toprule
Task &  & SAC & DeepMDP & DBC & DBC-normed & MICo & RAP & SimSR \\
\midrule
{\footnotesize cartpole} & {\footnotesize balance} & 980.3 {\footnotesize $\pm\; 8.7$} & \textbf{983.2 {\footnotesize $\pm\; 21.8$}} & 904.4 {\footnotesize $\pm\; 76.5$} & 774.3 {\footnotesize $\pm\; 355.0$} & 952.9 {\footnotesize $\pm\; 43.0$} & 949.0 {\footnotesize $\pm\; 47.2$} & 982.0 {\footnotesize $\pm\; 9.8$} \\
{\footnotesize cartpole} & {\footnotesize balance\_sparse} & 985.1 {\footnotesize $\pm\; 13.8$} & 995.0 {\footnotesize $\pm\; 8.5$} & 759.3 {\footnotesize $\pm\; 197.7$} & 972.2 {\footnotesize $\pm\; 60.5$} & 978.4 {\footnotesize $\pm\; 63.0$} & 973.3 {\footnotesize $\pm\; 66.0$} & \textbf{997.6 {\footnotesize $\pm\; 6.7$}} \\
{\footnotesize walker} & {\footnotesize stand} & \textbf{961.8 {\footnotesize $\pm\; 10.8$}} & 916.2 {\footnotesize $\pm\; 128.5$} & 534.5 {\footnotesize $\pm\; 347.5$} & 951.0 {\footnotesize $\pm\; 5.3$} & 943.3 {\footnotesize $\pm\; 17.7$} & 960.0 {\footnotesize $\pm\; 12.1$} & 957.7 {\footnotesize $\pm\; 5.6$} \\
{\footnotesize finger} & {\footnotesize spin} & 962.1 {\footnotesize $\pm\; 27.0$} & \textbf{982.2 {\footnotesize $\pm\; 4.6$}} & 163.9 {\footnotesize $\pm\; 521.6$} & 905.2 {\footnotesize $\pm\; 177.5$} & 978.8 {\footnotesize $\pm\; 12.7$} & 860.1 {\footnotesize $\pm\; 331.9$} & 908.8 {\footnotesize $\pm\; 212.0$} \\
{\footnotesize cartpole} & {\footnotesize swingup} & 848.4 {\footnotesize $\pm\; 13.9$} & \textbf{861.4 {\footnotesize $\pm\; 169.4$}} & 753.1 {\footnotesize $\pm\; 41.3$} & 293.9 {\footnotesize $\pm\; 337.8$} & 849.4 {\footnotesize $\pm\; 10.1$} & 857.1 {\footnotesize $\pm\; 30.5$} & 830.5 {\footnotesize $\pm\; 56.6$} \\
\bottomrule
\end{tabular}
\end{adjustbox}
\label{tab:img_source_video}
\end{table}

\begin{table}[htbp]
\centering
\caption{Performance of state-based DMC tasks for the compared methods with IID Gaussian noise background.}
\begin{adjustbox}{max width=\linewidth}
\begin{tabular}{lllllllll}
\toprule
Task &  & SAC & DeepMDP & DBC & DBC-normed & MICo & RAP & SimSR \\
\midrule
{\footnotesize cartpole} & {\footnotesize balance} & 975.0 {\footnotesize $\pm\; 12.3$} & 973.0 {\footnotesize $\pm\; 40.5$} & 951.0 {\footnotesize $\pm\; 18.1$} & 982.6 {\footnotesize $\pm\; 26.8$} & 959.6 {\footnotesize $\pm\; 23.9$} & 981.3 {\footnotesize $\pm\; 8.0$} & \textbf{988.4 {\footnotesize $\pm\; 6.1$}} \\
{\footnotesize cartpole} & {\footnotesize balance\_sparse} & 984.8 {\footnotesize $\pm\; 29.0$} & 999.8 {\footnotesize $\pm\; 0.4$} & 855.4 {\footnotesize $\pm\; 118.4$} & \textbf{1000.0 {\footnotesize $\pm\; 0.0$}} & 873.7 {\footnotesize $\pm\; 171.8$} & 962.5 {\footnotesize $\pm\; 61.3$} & 911.9 {\footnotesize $\pm\; 105.7$} \\
{\footnotesize walker} & {\footnotesize stand} & 951.5 {\footnotesize $\pm\; 14.2$} & 951.2 {\footnotesize $\pm\; 37.9$} & 896.1 {\footnotesize $\pm\; 311.1$} & 373.2 {\footnotesize $\pm\; 397.6$} & 941.0 {\footnotesize $\pm\; 33.3$} & \textbf{959.0 {\footnotesize $\pm\; 9.2$}} & 936.9 {\footnotesize $\pm\; 18.3$} \\
{\footnotesize finger} & {\footnotesize spin} & 720.5 {\footnotesize $\pm\; 250.6$} & 587.1 {\footnotesize $\pm\; 663.7$} & 572.1 {\footnotesize $\pm\; 648.5$} & \textbf{945.1 {\footnotesize $\pm\; 67.5$}} & 814.5 {\footnotesize $\pm\; 411.0$} & 863.8 {\footnotesize $\pm\; 306.2$} & 786.3 {\footnotesize $\pm\; 545.8$} \\
{\footnotesize cartpole} & {\footnotesize swingup} & 655.5 {\footnotesize $\pm\; 282.0$} & 835.9 {\footnotesize $\pm\; 43.3$} & 796.4 {\footnotesize $\pm\; 64.1$} & 737.3 {\footnotesize $\pm\; 258.3$} & 804.7 {\footnotesize $\pm\; 34.6$} & \textbf{851.6 {\footnotesize $\pm\; 10.4$}} & 811.7 {\footnotesize $\pm\; 88.7$} \\
\bottomrule
\end{tabular}
\end{adjustbox}
\label{tab:img_source_noise}
\end{table}

\begin{figure}[h]
    \makebox[\linewidth]{%
        \includegraphics[width=0.85\linewidth]{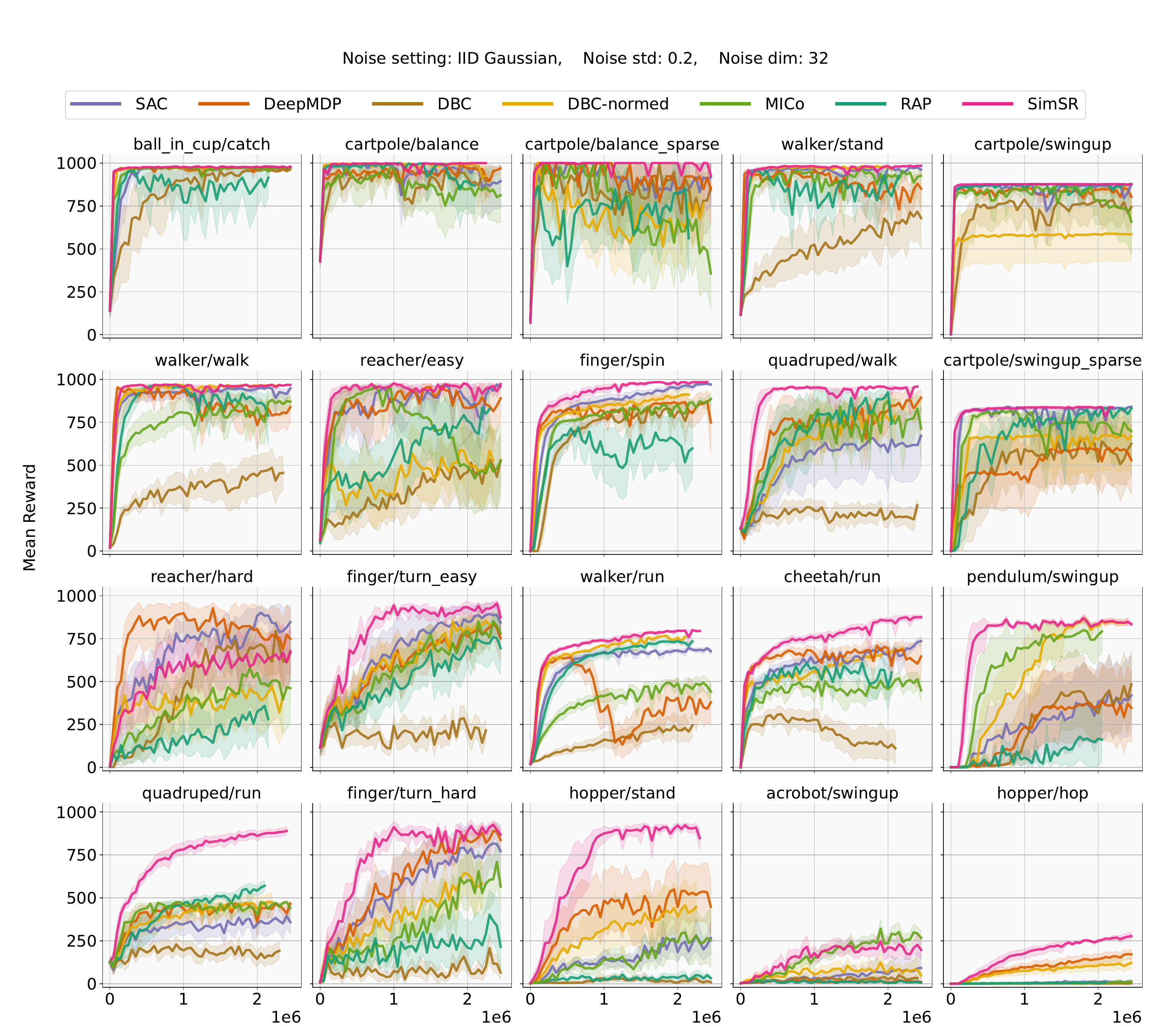}%
    }
    \vspace{-2em}
    \caption{Performance on individual state-based tasks.}
    \label{fig:all_std02}
\end{figure}

\begin{figure}[h]
    \makebox[\linewidth]{%
        \includegraphics[width=0.85\linewidth]{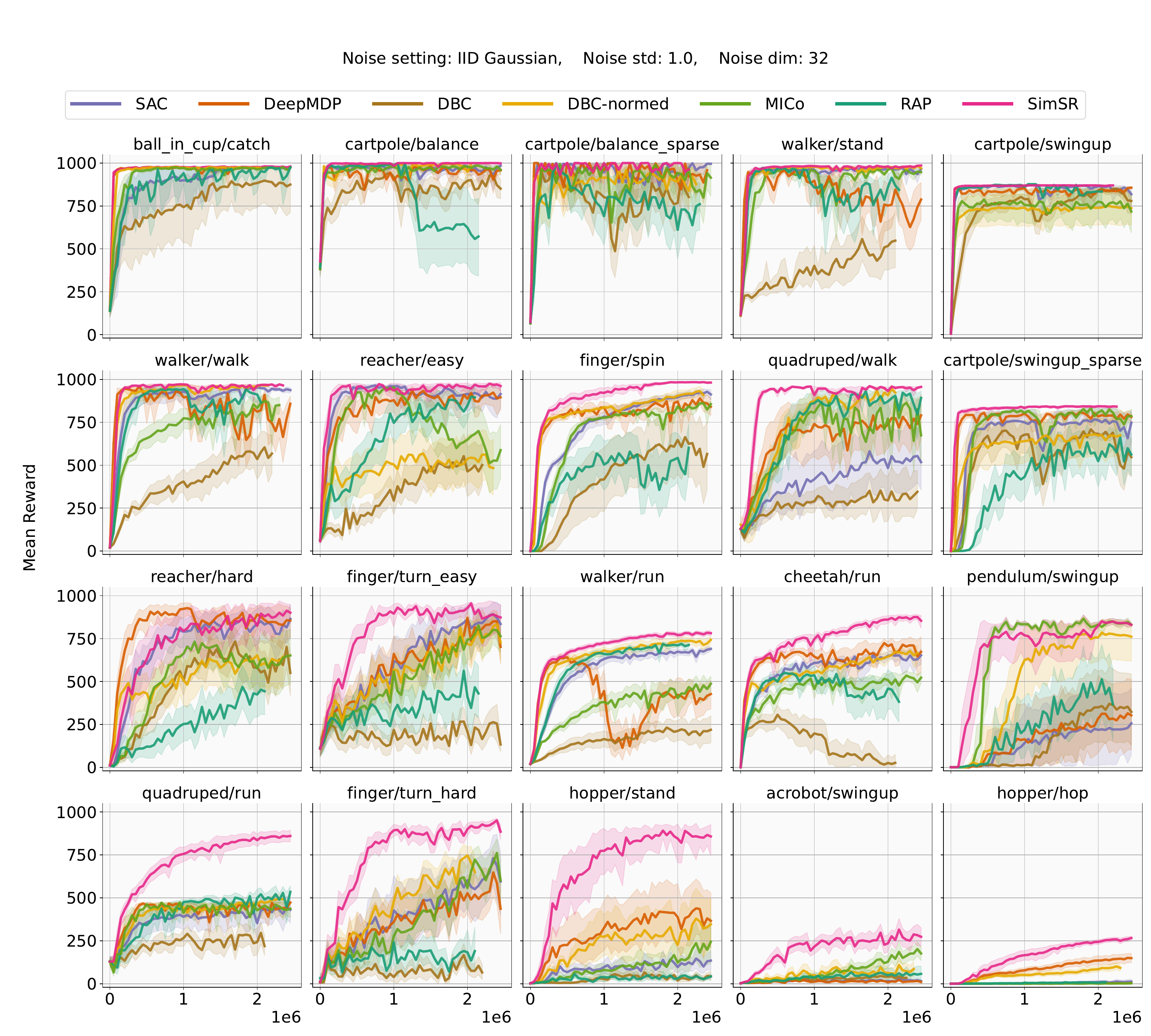}%
    }
    \vspace{-2em}
    \caption{Performance on individual state-based tasks.}
    \label{fig:all_std1}
\end{figure}

\begin{figure}[h]
    \makebox[\linewidth]{%
        \includegraphics[width=0.85\linewidth]{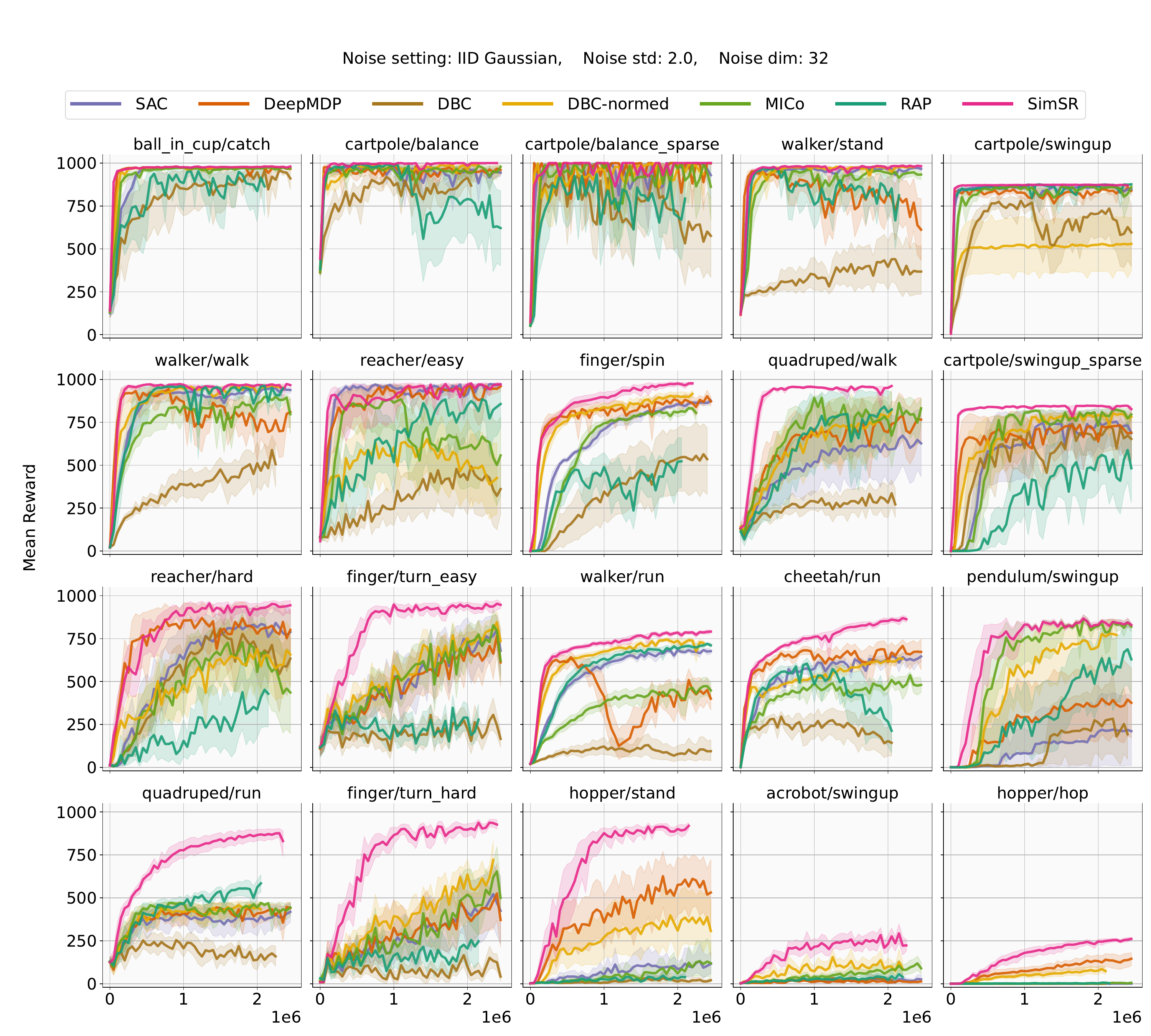}%
    }
        \vspace{-2em}
    \caption{Performance on individual state-based tasks.}
    \label{fig:all_std2}
\end{figure}

\begin{figure}[h]
    \makebox[\linewidth]{%
        \includegraphics[width=0.85\linewidth]{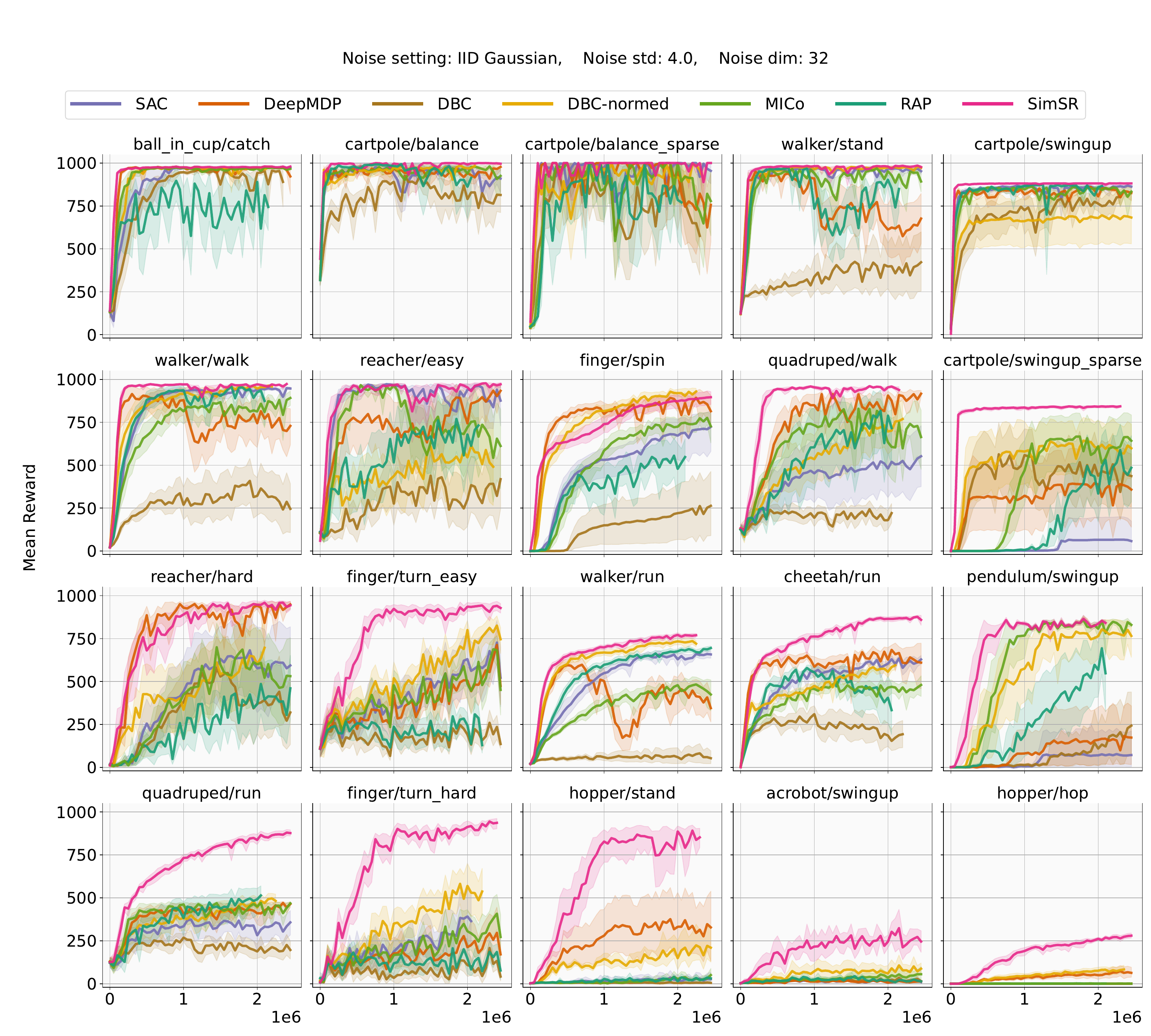}%
    }
        \vspace{-2em}
    \caption{Performance on individual state-based tasks.}
    \label{fig:all_std4}
\end{figure}

\begin{figure}[h]
    \makebox[\linewidth]{%
        \includegraphics[width=0.85\linewidth]{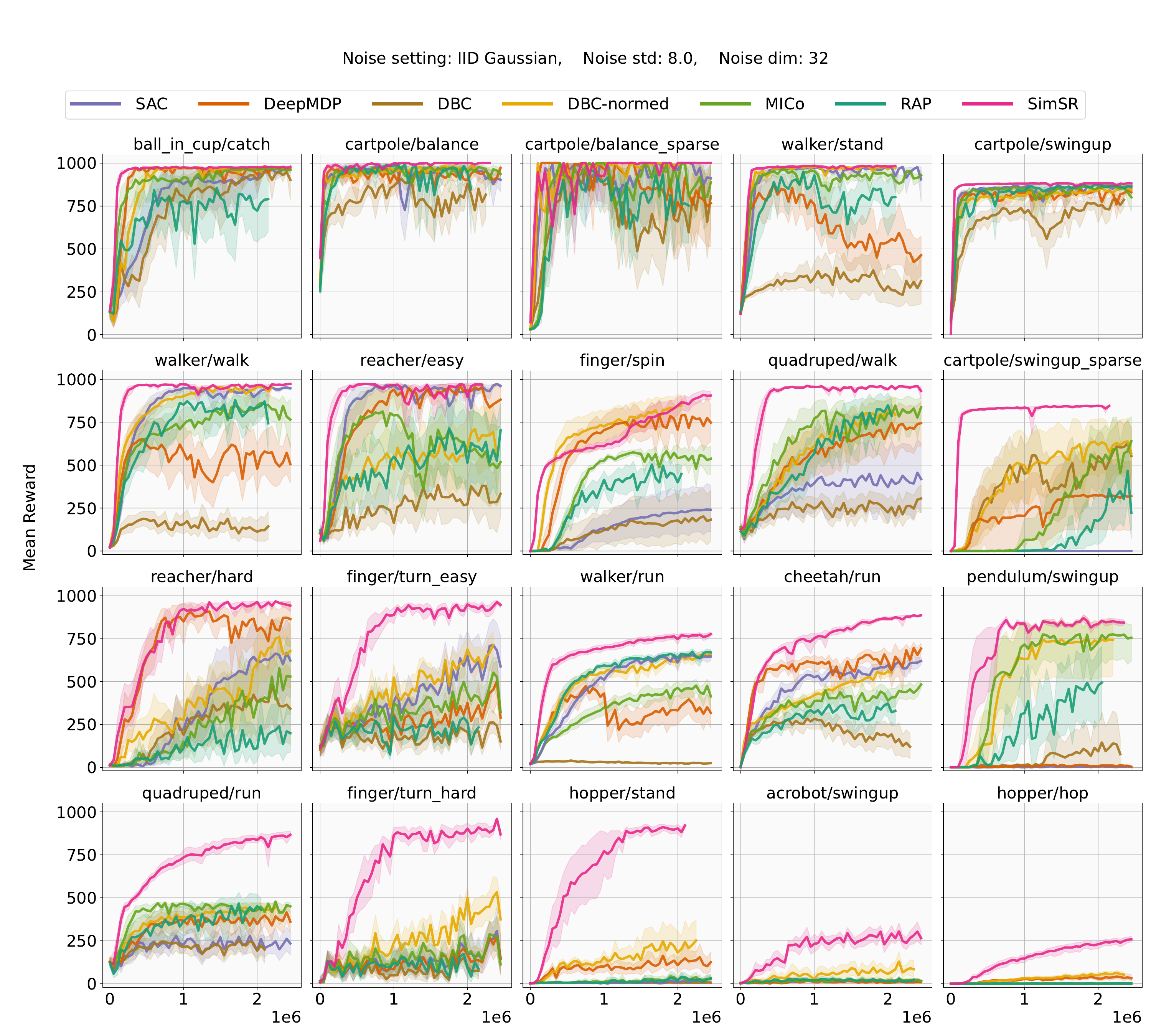}%
    }
        \vspace{-2em}
    \caption{Performance on individual state-based tasks.}
    \label{fig:all_std8}
\end{figure}

\begin{figure}[h]
    \makebox[\linewidth]{%
        \includegraphics[width=0.85\linewidth]{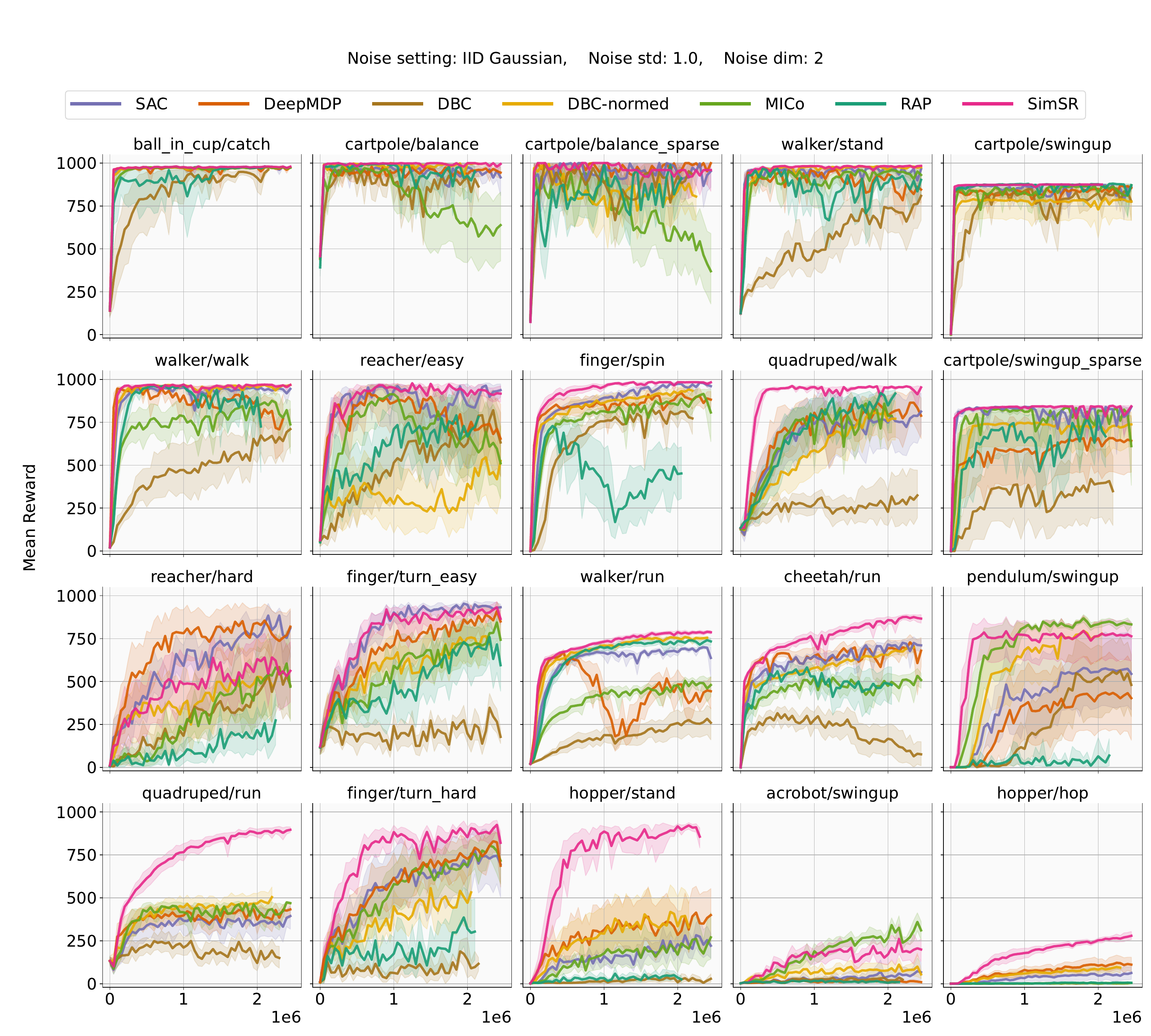}%
    }
        \vspace{-2em}
    \caption{Performance on individual state-based tasks.}
    \label{fig:all_dim2}
\end{figure}

\begin{figure}[h]
    \makebox[\linewidth]{%
        \includegraphics[width=0.85\linewidth]{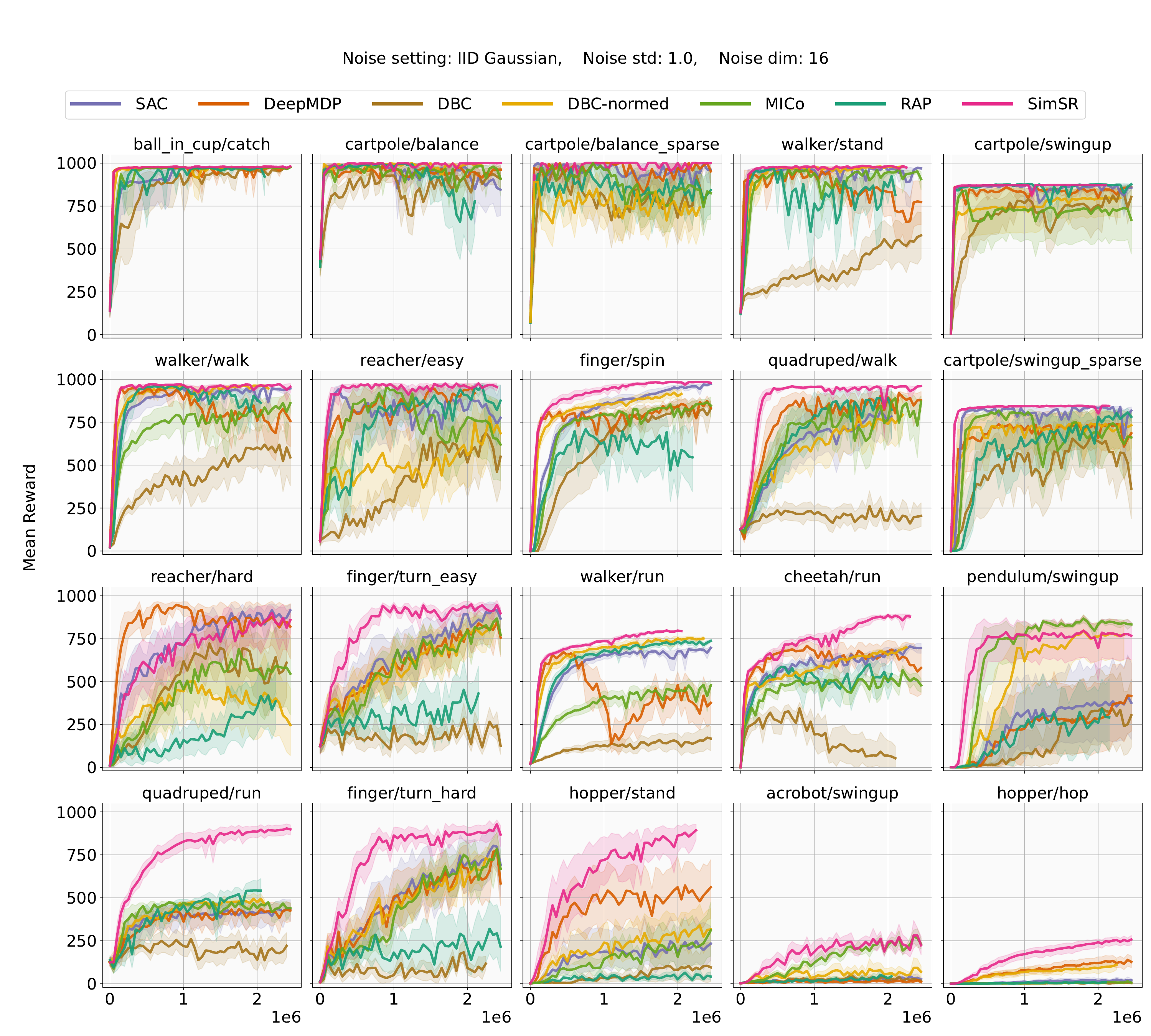}%
    }
        \vspace{-2em}
    \caption{Performance on individual state-based tasks.}
    \label{fig:all_dim16}
\end{figure}

\begin{figure}[h]
    \makebox[\linewidth]{%
        \includegraphics[width=0.85\linewidth]{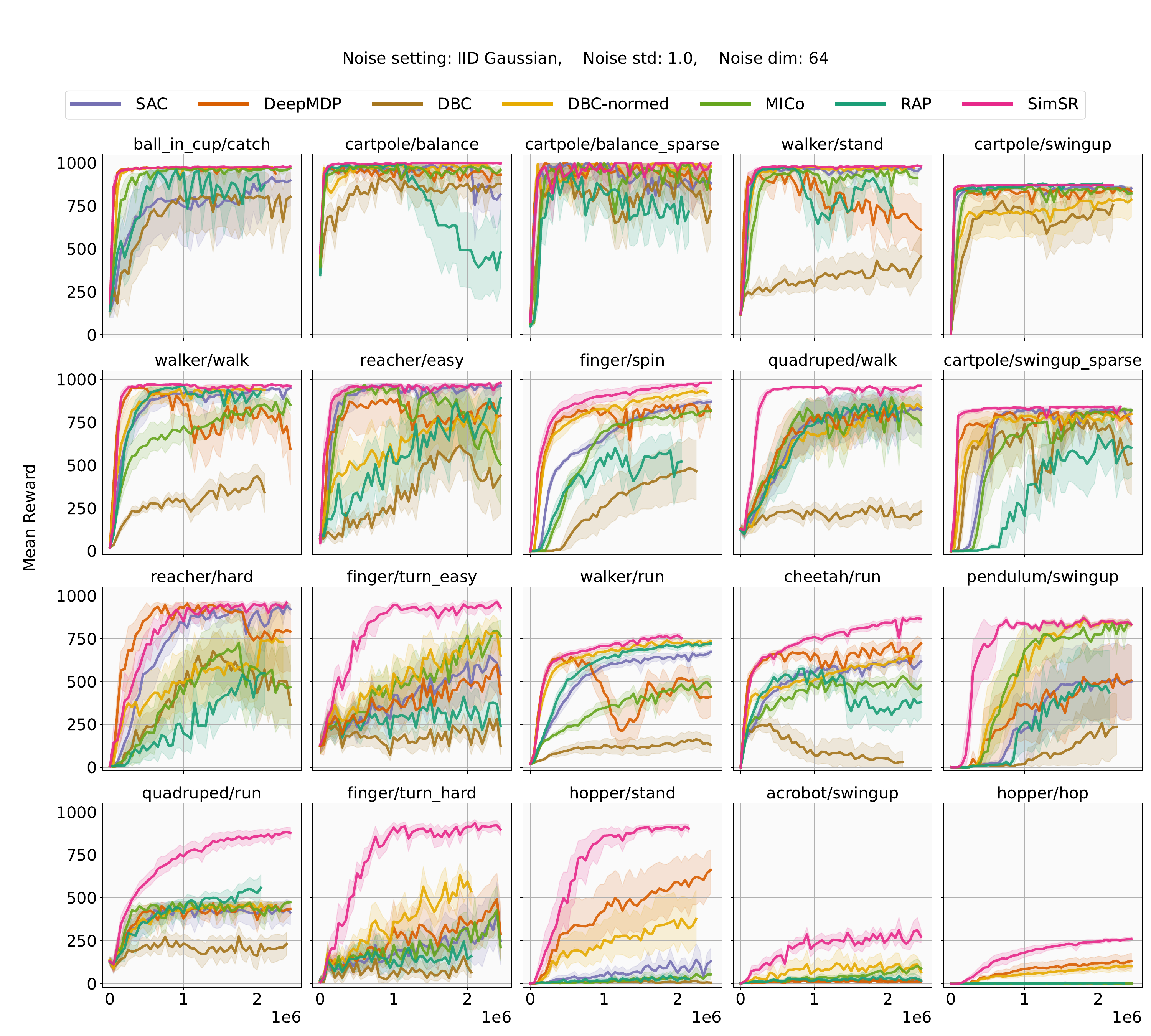}%
    }
        \vspace{-2em}
    \caption{Performance on individual state-based tasks.}
    \label{fig:all_dim64}
\end{figure}

\begin{figure}[h]
    \makebox[\linewidth]{%
        \includegraphics[width=0.85\linewidth]{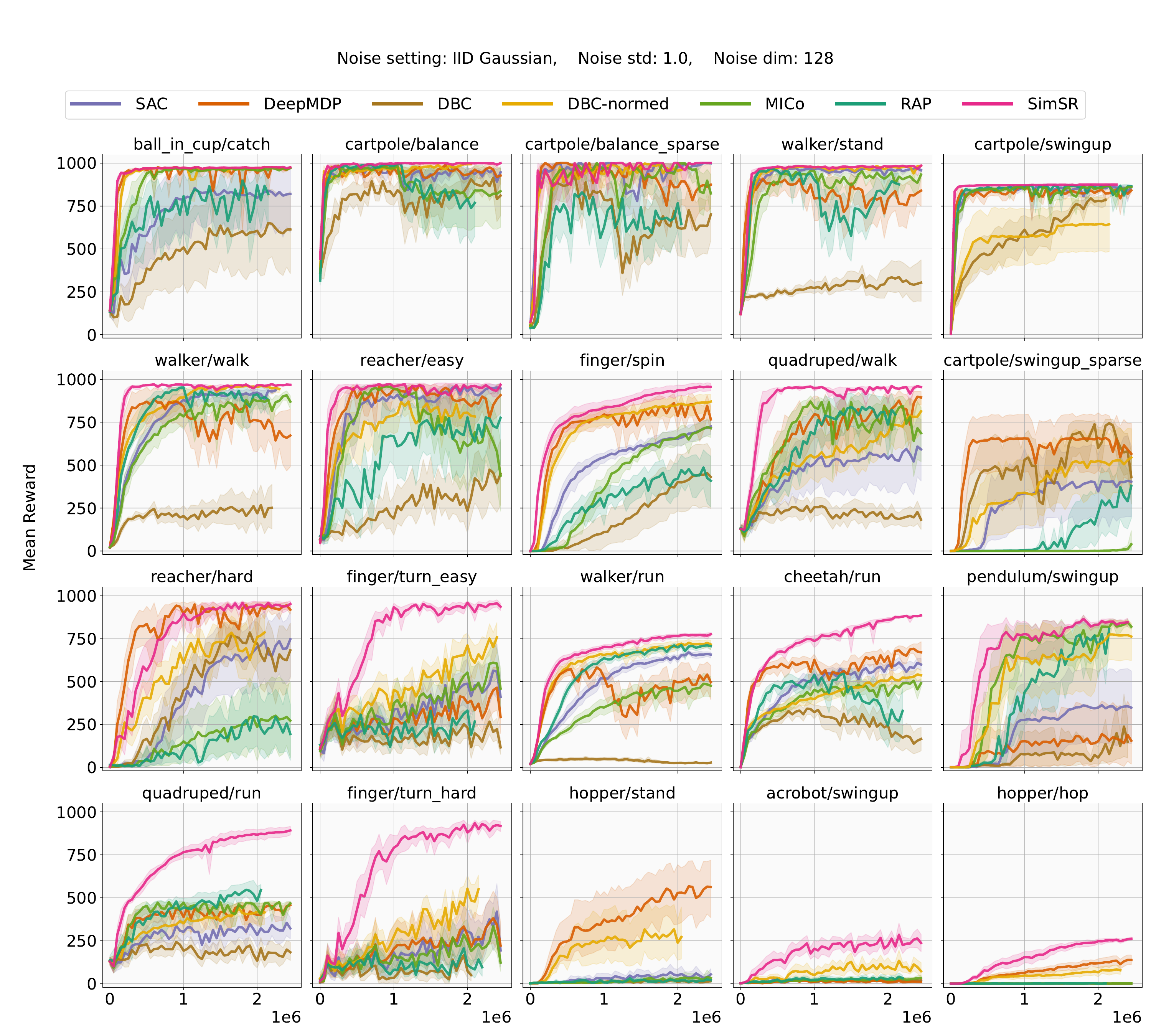}%
    }
            \vspace{-2em}
    \caption{Performance on individual state-based tasks.}
    \label{fig:all_dim128}
\end{figure}

\begin{figure}[h]
    \makebox[\linewidth]{%
        \includegraphics[width=1\linewidth]{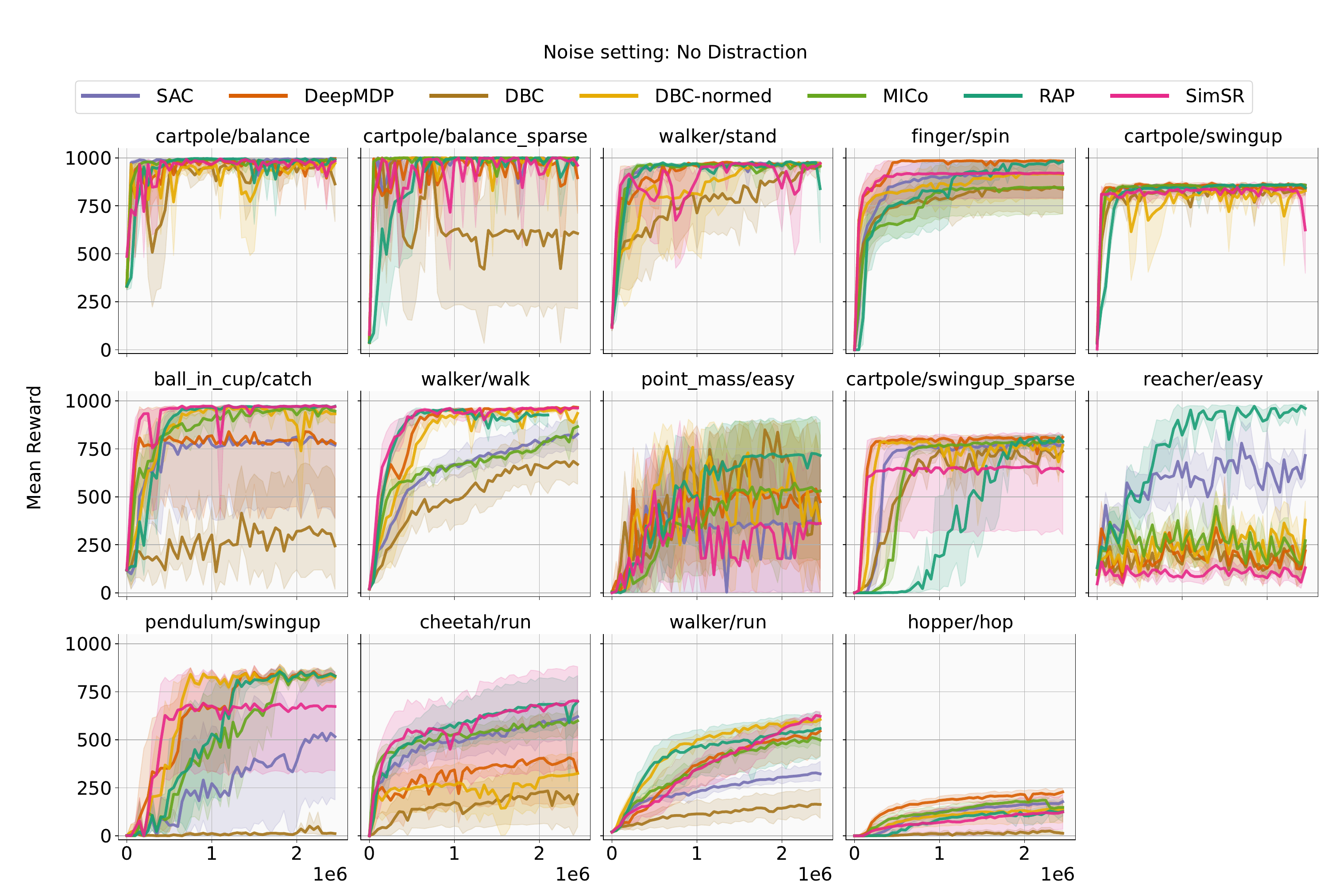}%
    }
                \vspace{-2em}
    \caption{Performance on individual pixel-based tasks.}
    \label{fig:all_none}
\end{figure}

\begin{figure}[h]
    \makebox[\linewidth]{%
        \includegraphics[width=1\linewidth]{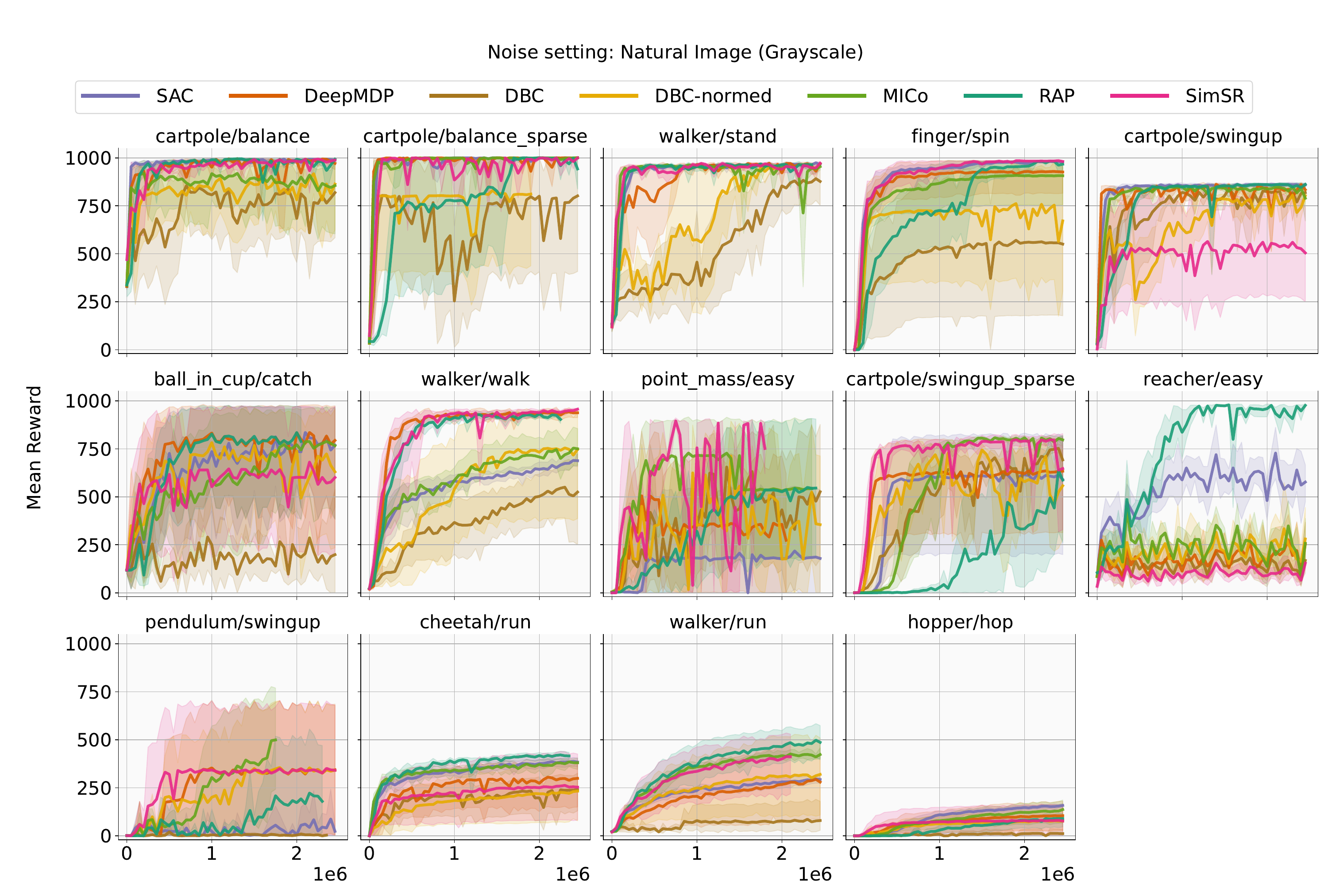}%
    }
                \vspace{-2em}
    \caption{Performance on individual pixel-based tasks.}
    \label{fig:all_imageg}
\end{figure}

\begin{figure}[h]
    \makebox[\linewidth]{%
        \includegraphics[width=1\linewidth]{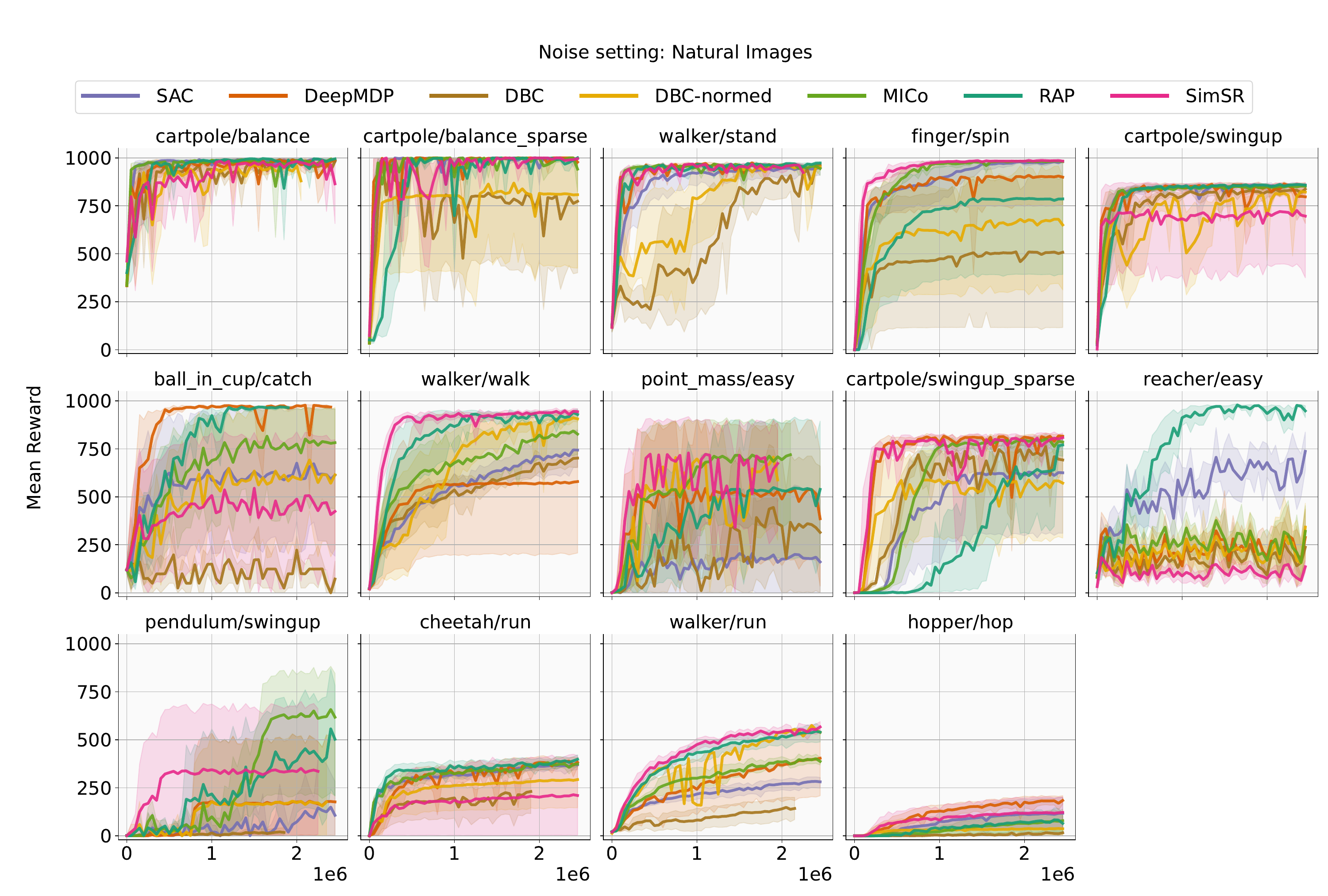}%
    }
                \vspace{-2em}
    \caption{Performance on individual pixel-based tasks.}
    \label{fig:all_image}
\end{figure}

\begin{figure}[h]
    \makebox[\linewidth]{%
        \includegraphics[width=1\linewidth]{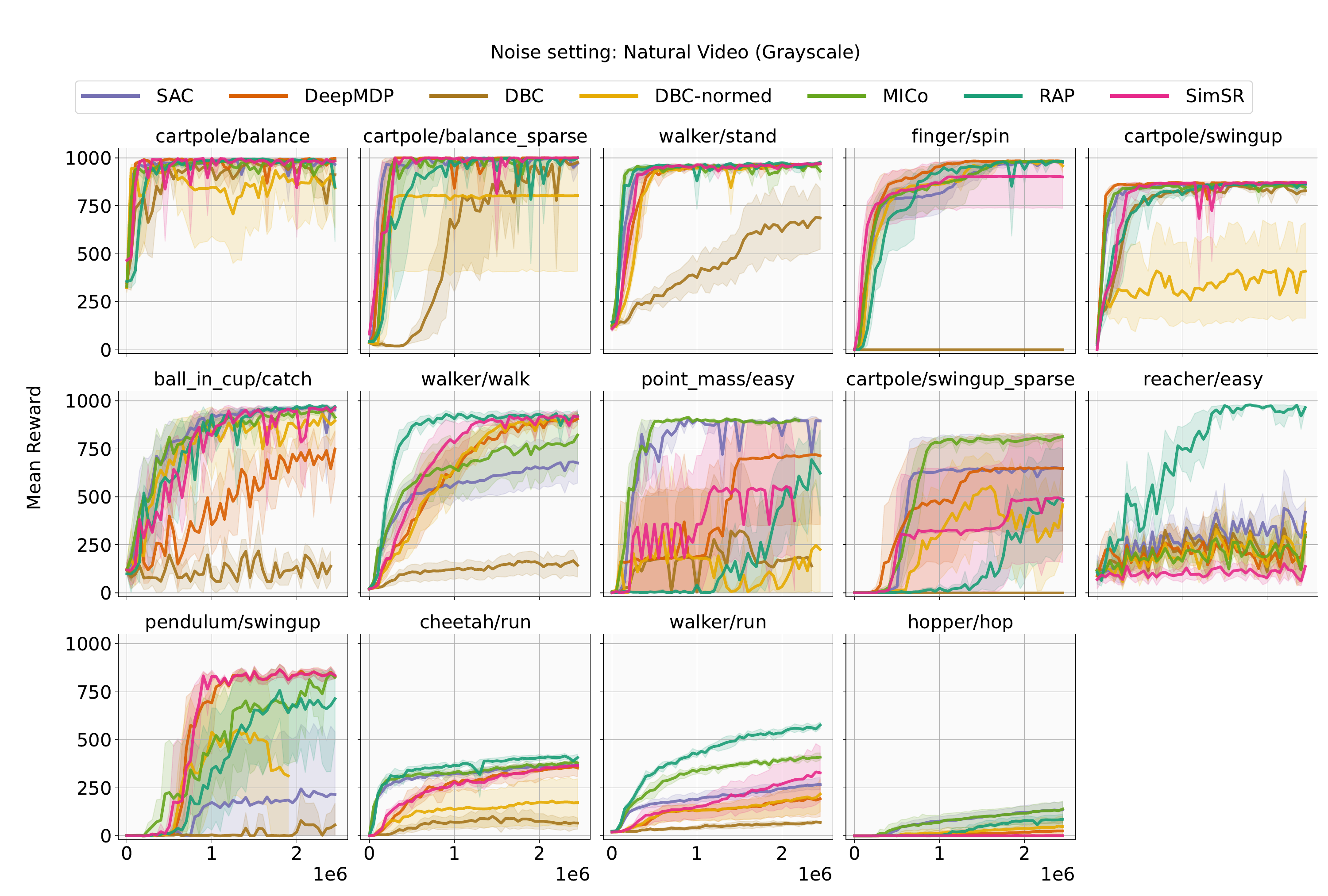}%
    }
                \vspace{-2em}
    \caption{Performance on individual pixel-based tasks.}
    \label{fig:all_videog}
\end{figure}

\begin{figure}[h]
    \makebox[\linewidth]{%
        \includegraphics[width=1\linewidth]{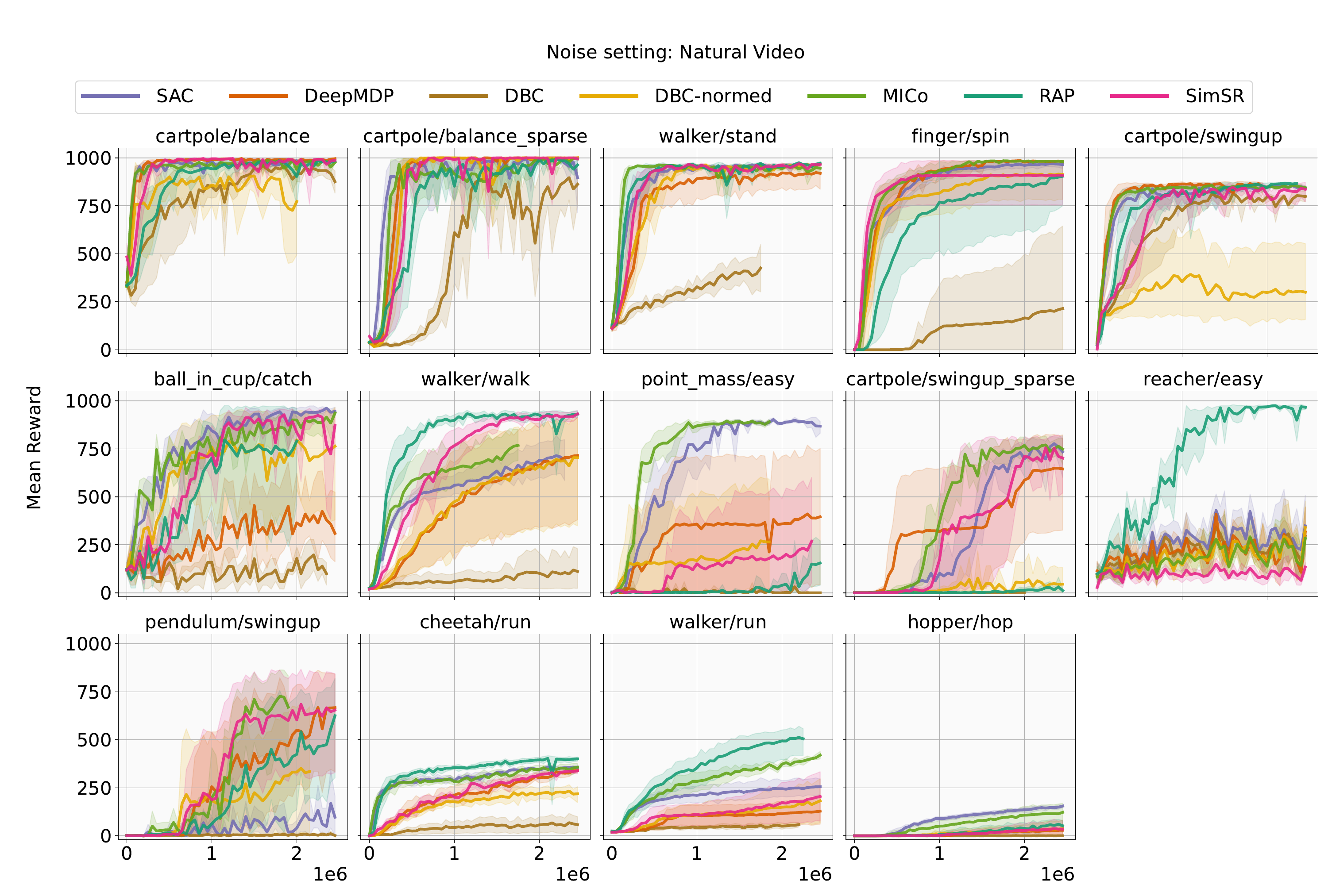}%
    }
                \vspace{-2em}
    \caption{Performance on individual pixel-based tasks.}
    \label{fig:all_video}
\end{figure}

\begin{figure}[h]
    \makebox[\linewidth]{%
        \includegraphics[width=1\linewidth]{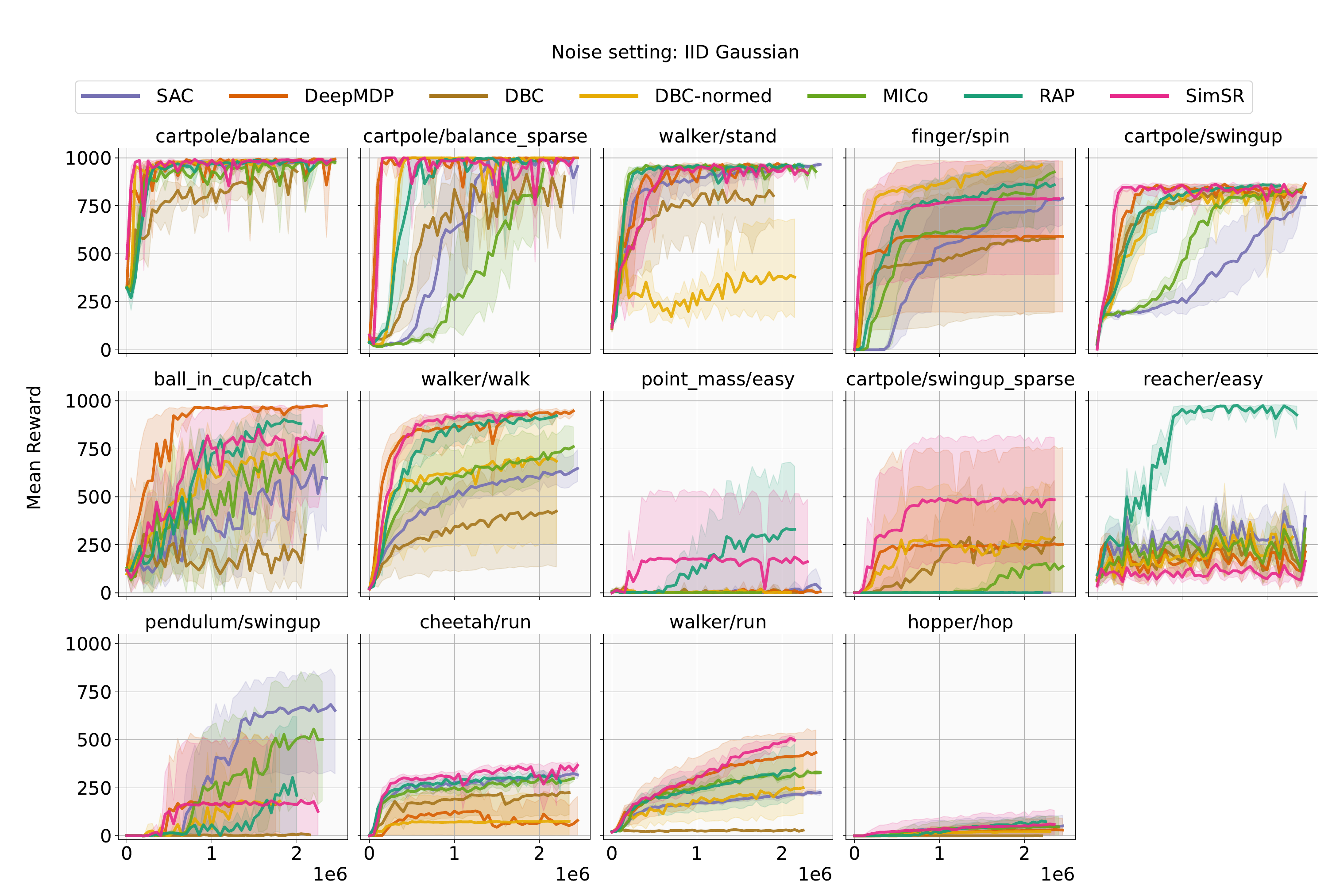}%
    }
                \vspace{-2em}
    \caption{Performance on individual pixel-based tasks.}
    \label{fig:all_pixelnoise}
\end{figure}

\begin{figure}[ht]
    \centering
    \includegraphics[width=0.75\linewidth]{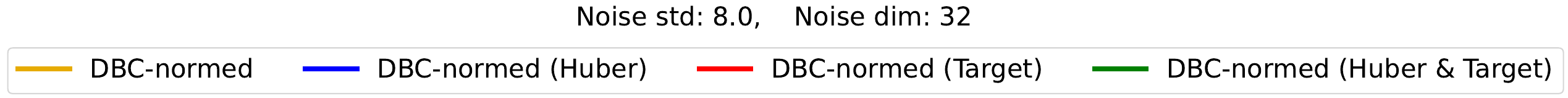}
    \includegraphics[width=\linewidth]{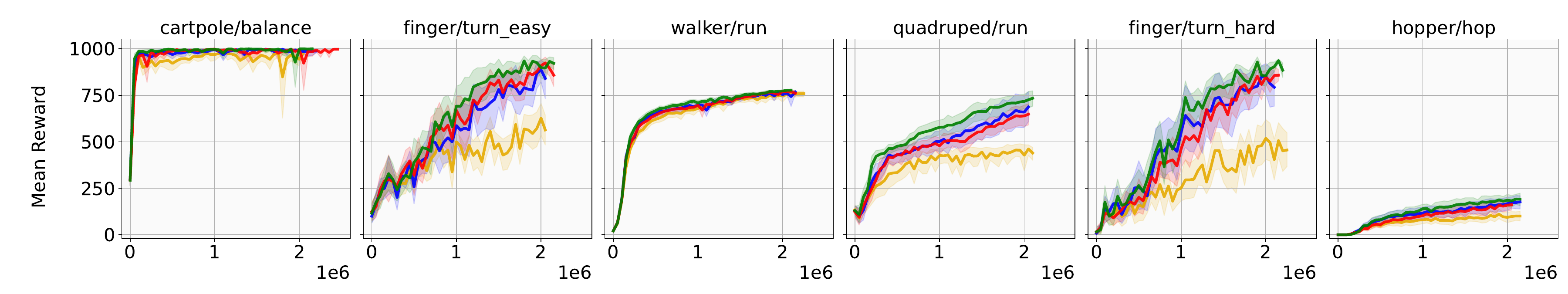}
    \includegraphics[width=\linewidth]{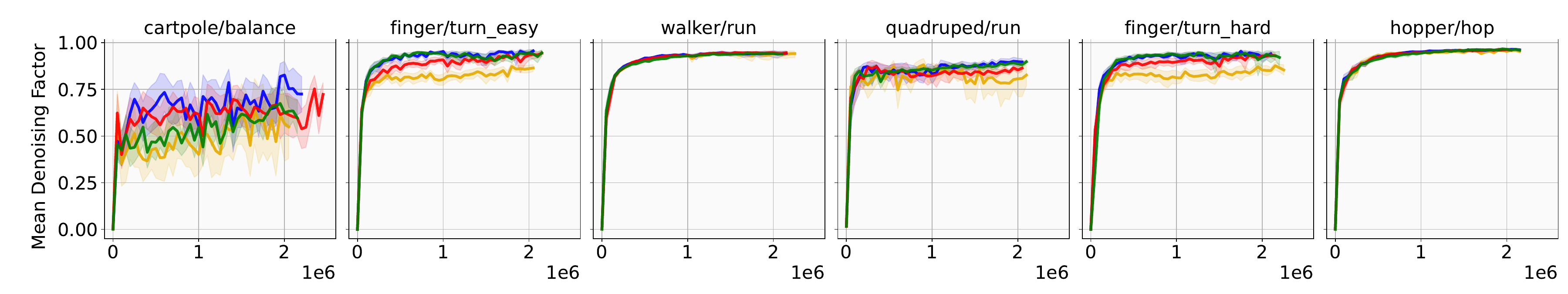}
    
    \caption{\small Case study on 6 state-based DMC tasks on \textbf{DBC-normed with LayerNorm} and its variants:
    \textbf{applying the target trick} (DBC-normed (Target)), \textbf{using Huber loss} (DBC-normed (Huber)) instead of MSE as the metric loss, and \textbf{applying both} (DBC-normed (Huber \& Target)).
    X-axis is the environmental step.
}
    \label{fig:5.2DBC-normed}
\end{figure} 

\begin{figure}[h]
    \centering
    \includegraphics[width=0.8\linewidth]{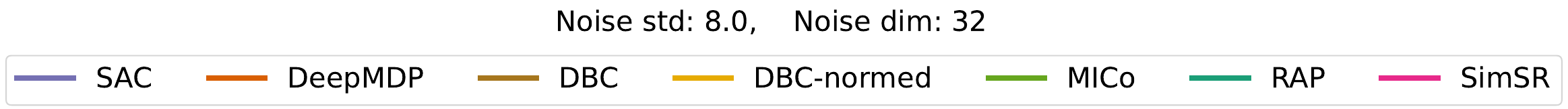}

    \includegraphics[width=\linewidth]{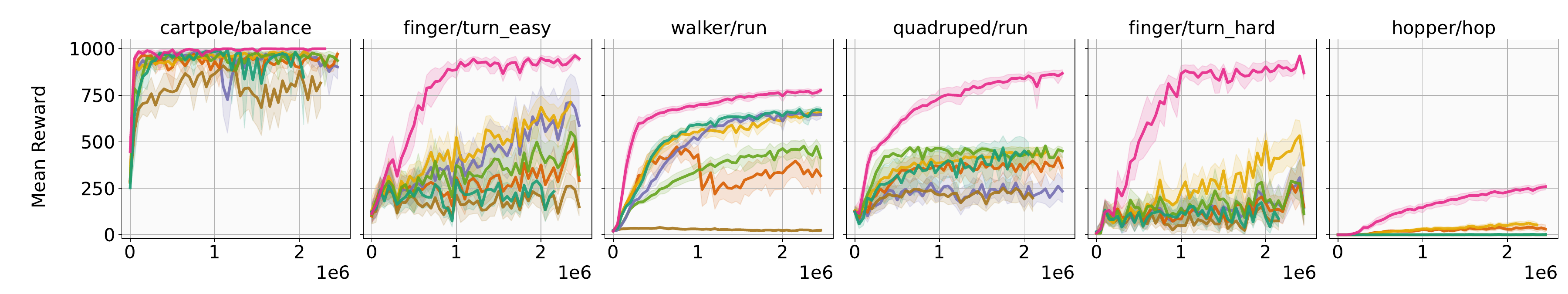}
    \includegraphics[width=\linewidth]{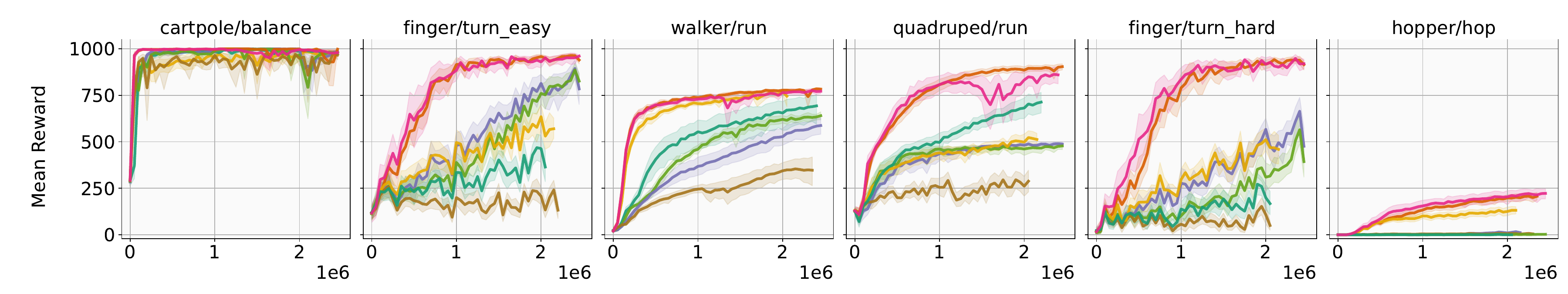}

    \includegraphics[width=\linewidth]{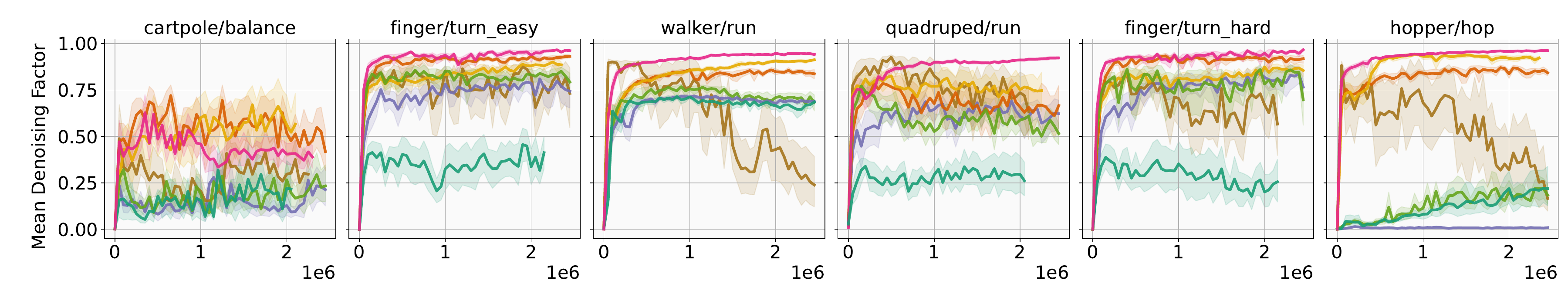}

    \includegraphics[width=\linewidth]{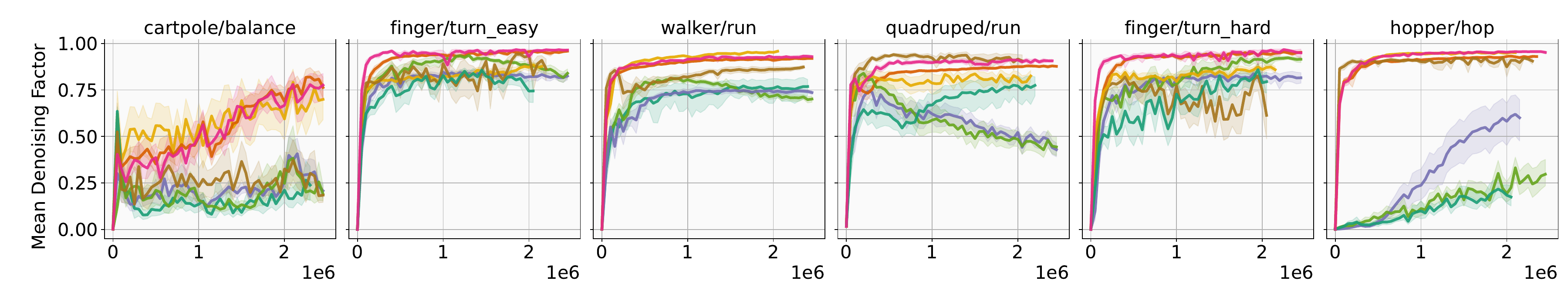}
    
    \caption{\small Case study on six DMC state-based tasks 
    examining the effects of \textbf{including LayerNorm} (the first and third vs. second and fourth rows).
    X-axis stands for the environmental step. See \autoref{tab:5.2ln} for a tabular presentation.
}
    \label{fig:5.2ln_huber}
\end{figure}

\begin{figure}[h]
\vspace{-1em}
    \centering
\includegraphics[width=0.75\linewidth]{fig/trend_legend_std_Overall_episode_reward_2000000.pdf} \\
\includegraphics[width=0.35\linewidth]{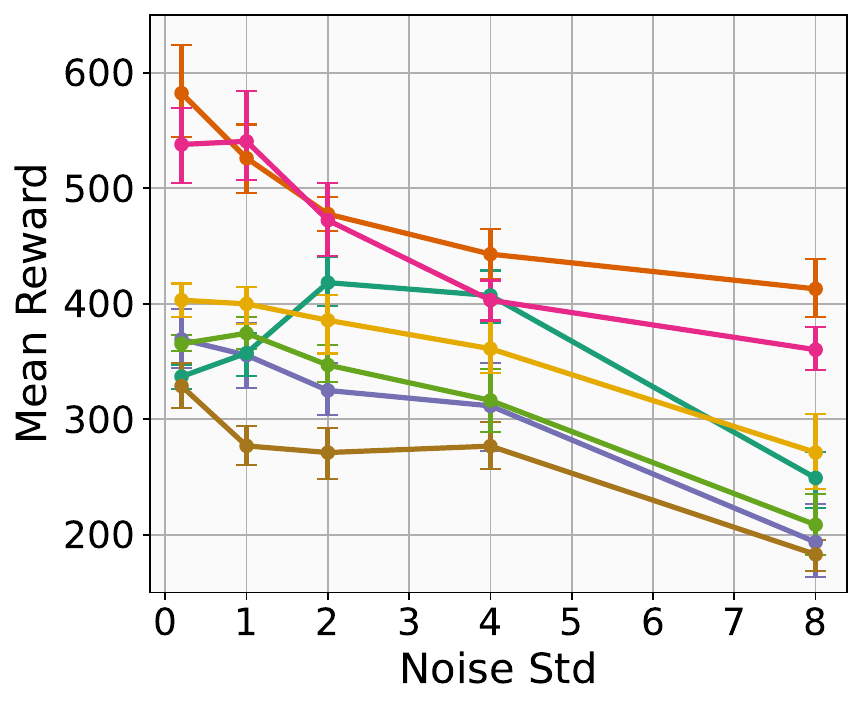}
\includegraphics[width=0.35\linewidth]{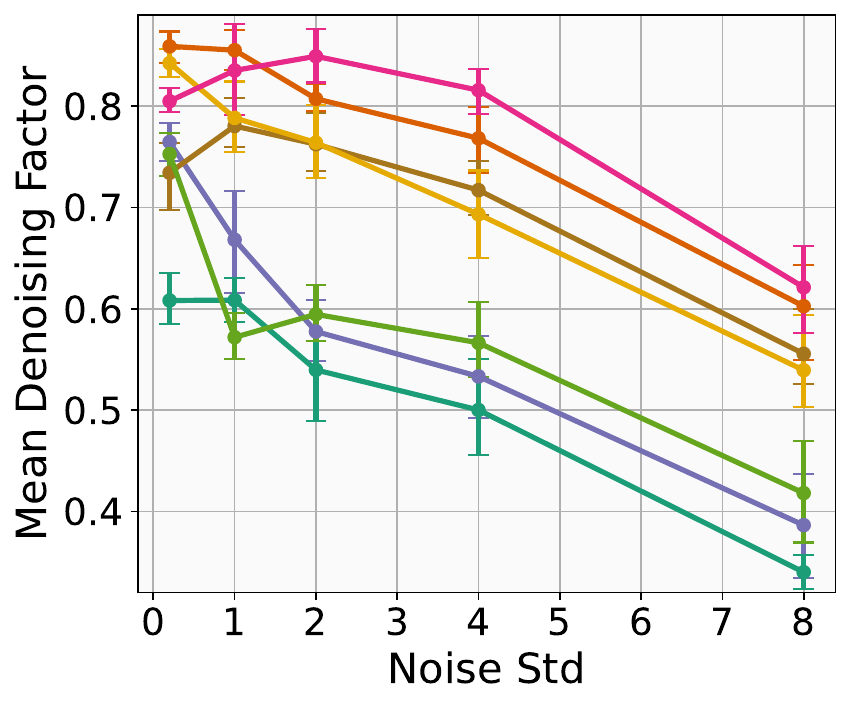}
    \caption{\small \textbf{Aggregated reward (left) and denoising factor (right)} of methods on \textbf{IID Gaussian noise with random projection} setting, varying noise standard deviation, in the 6 selected state-based tasks.}
\label{fig:rp_sensitivity}
\end{figure}

\begin{figure}[h]
\includegraphics[width=\linewidth]{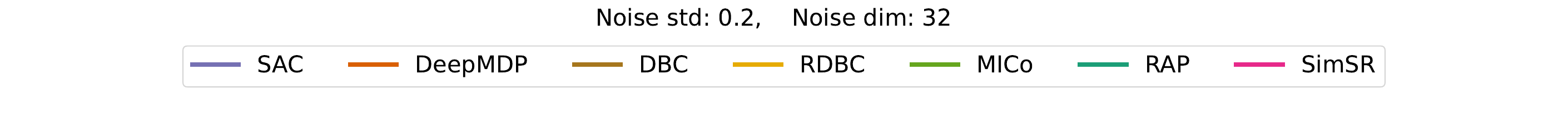}
\vspace{-1em}
\makebox[\linewidth]{
\includegraphics[width=\linewidth]{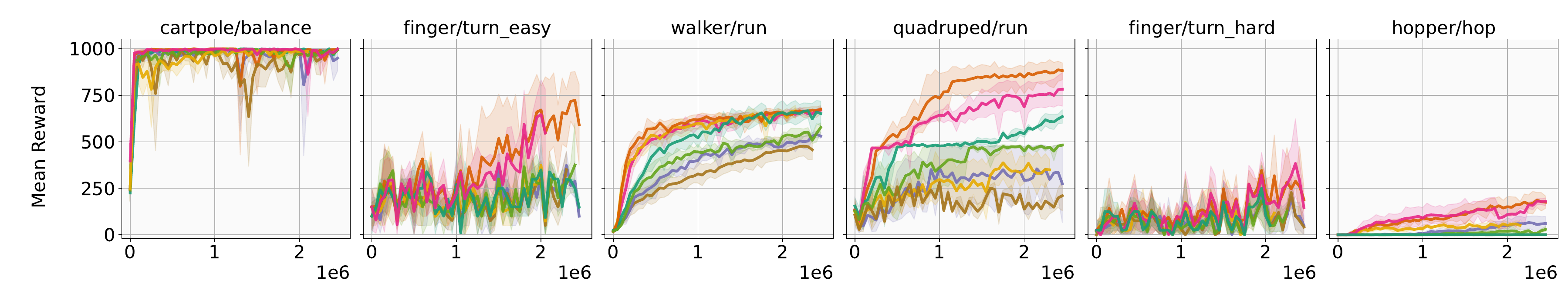}}
\makebox[\linewidth]{
\includegraphics[width=\linewidth]{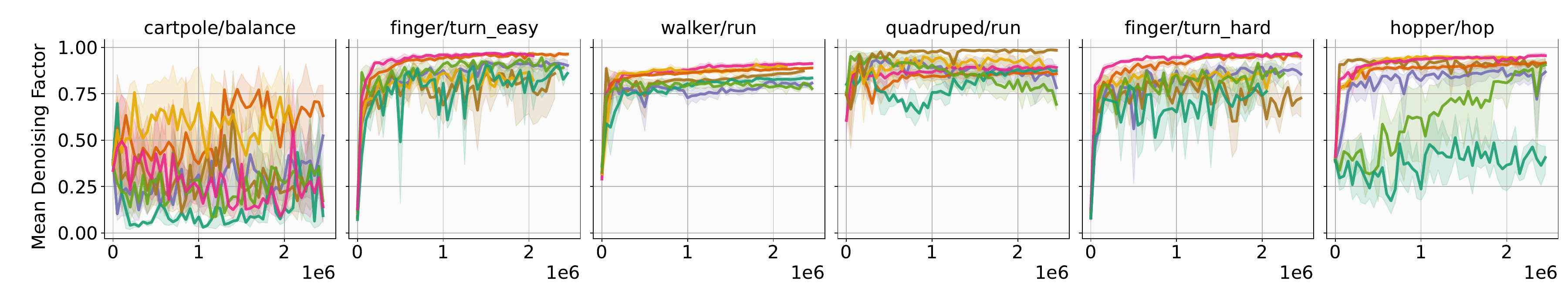}}
            \vspace{-2em}
    \caption{Performance on individual state-based tasks under the IID Gaussian noise with random projection setting.}
    \label{fig:rp_pertask02}
\end{figure}

\begin{figure}[h]
\includegraphics[width=\linewidth]{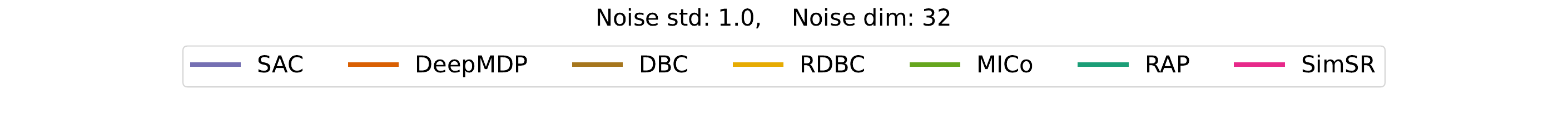}
\vspace{-1em}
\makebox[\linewidth]{
\includegraphics[width=\linewidth]{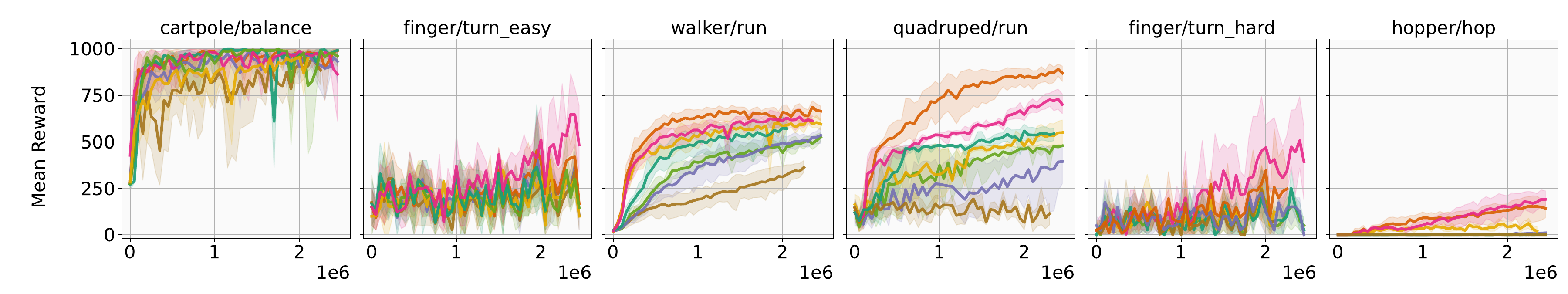}}
\makebox[\linewidth]{
\includegraphics[width=\linewidth]{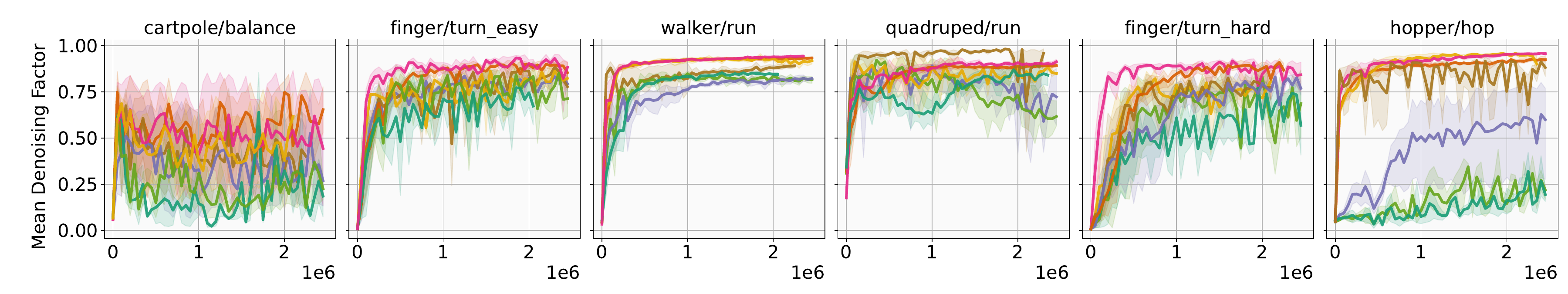}}
            \vspace{-2em}
    \caption{Performance on individual state-based tasks under the IID Gaussian noise with random projection setting.}
    \label{fig:rp_pertask1}
\end{figure}

\begin{figure}[h]
\includegraphics[width=\linewidth]{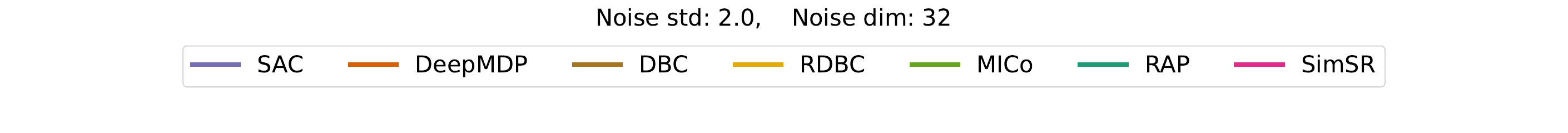}
\vspace{-1em}
\makebox[\linewidth]{
\includegraphics[width=\linewidth]{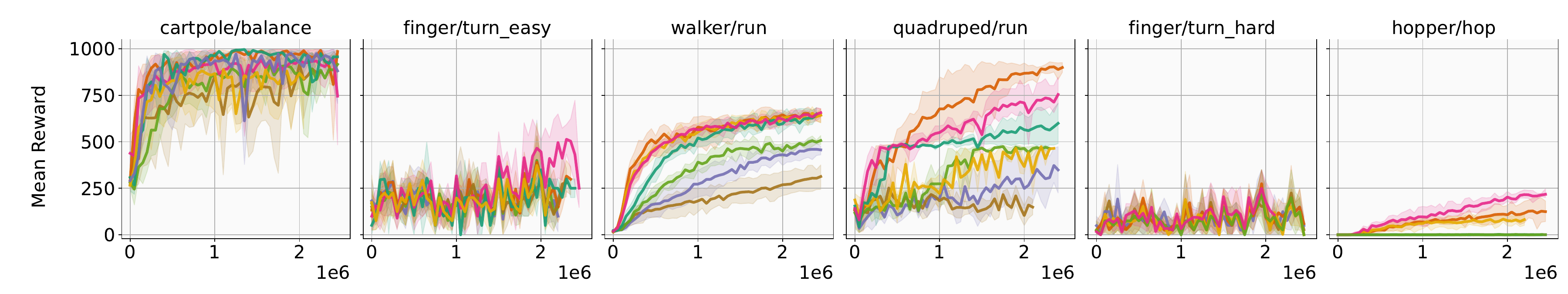}}
\makebox[\linewidth]{
\includegraphics[width=\linewidth]{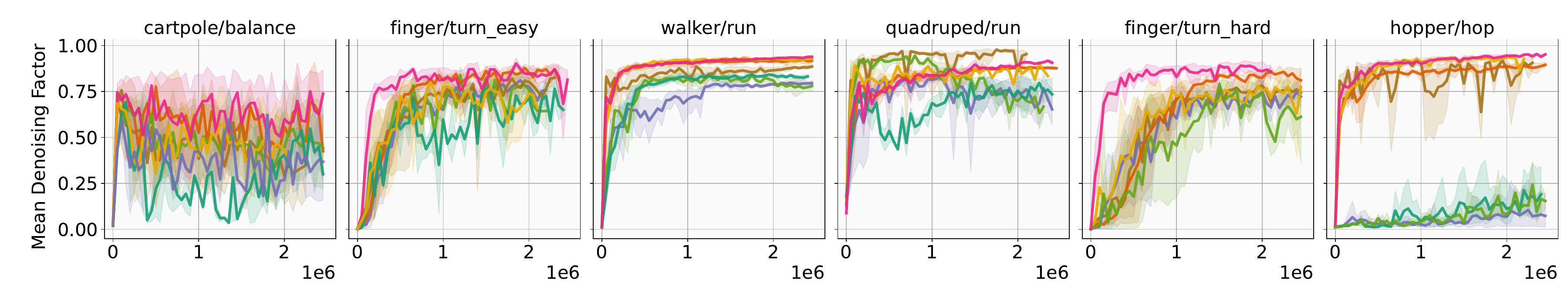}}
            \vspace{-2em}
    \caption{Performance on individual state-based tasks under the IID Gaussian noise with random projection setting.}
    \label{fig:rp_pertask2}
\end{figure}

\begin{figure}[h]
\includegraphics[width=\linewidth]{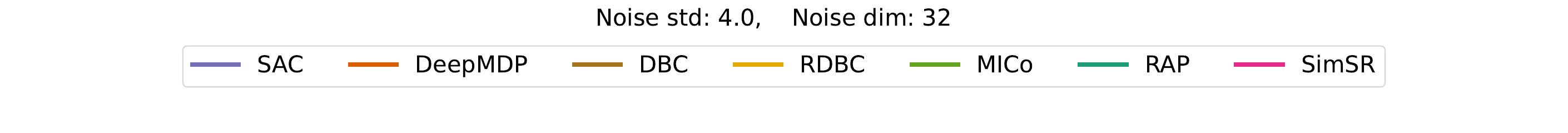}
\vspace{-1em}
\makebox[\linewidth]{
\includegraphics[width=\linewidth]{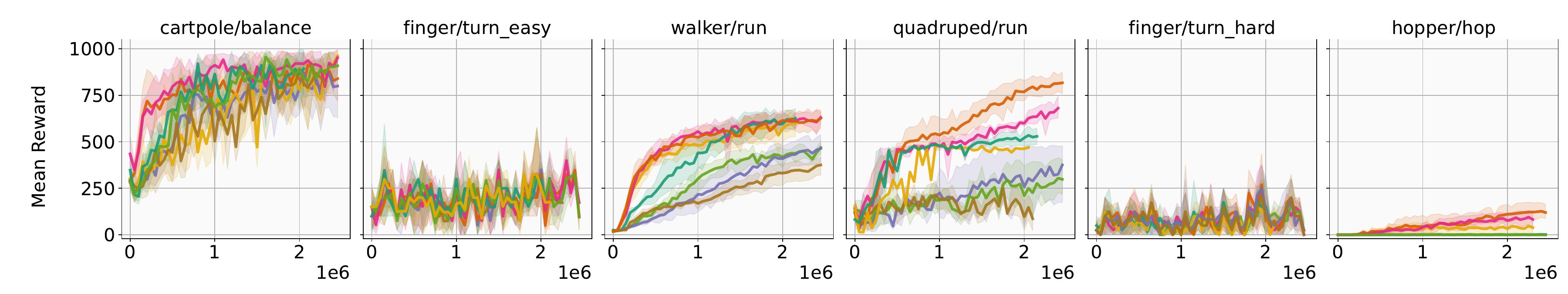}}
\makebox[\linewidth]{
\includegraphics[width=\linewidth]{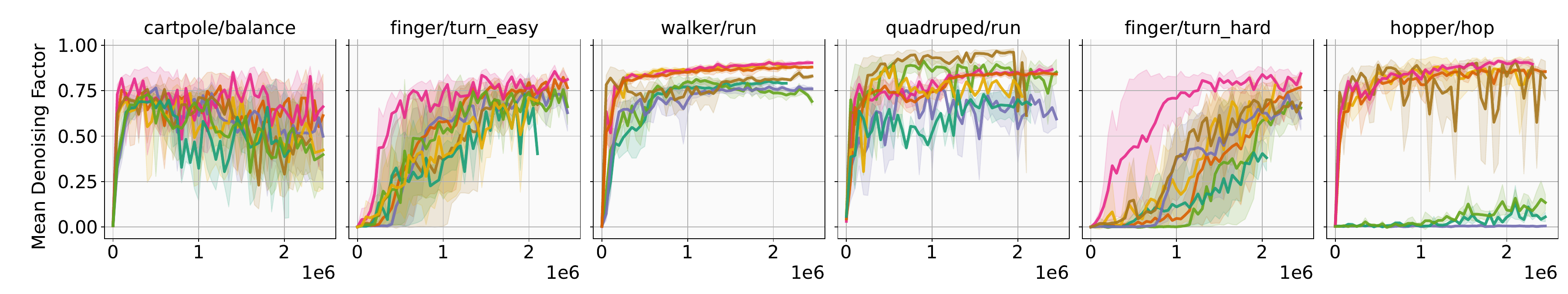}}
            \vspace{-2em}
    \caption{Performance on individual state-based tasks under the IID Gaussian noise with random projection setting.}
                \vspace{-1em}
    \label{fig:rp_pertask4}
\end{figure}

\begin{figure}[h]
\includegraphics[width=\linewidth]{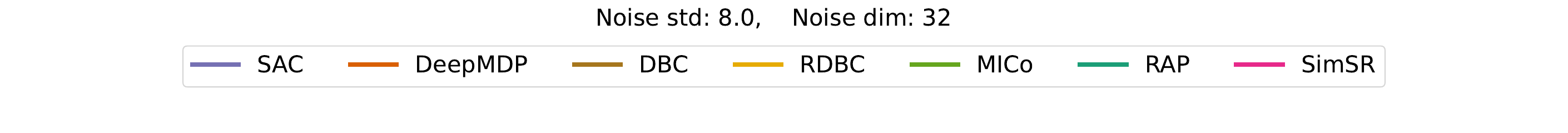}
\vspace{-1em}
\makebox[\linewidth]{
\includegraphics[width=\linewidth]{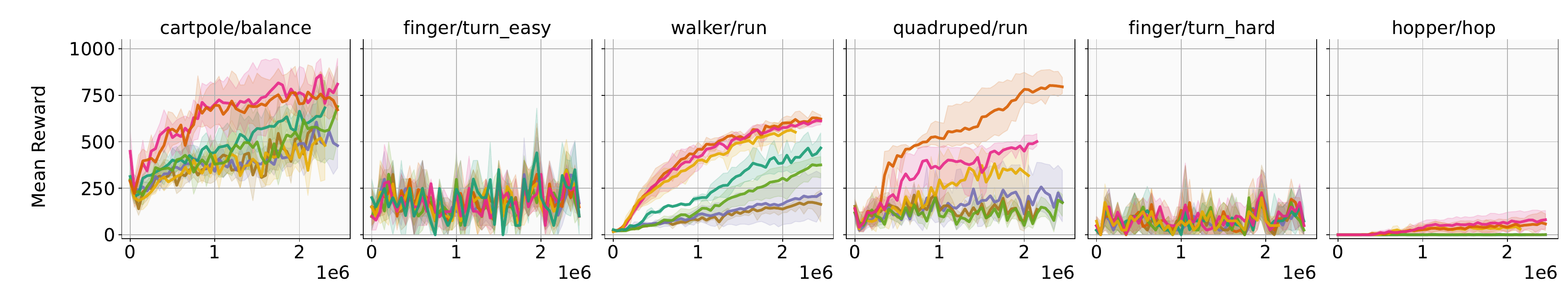}}
\makebox[\linewidth]{
\includegraphics[width=\linewidth]{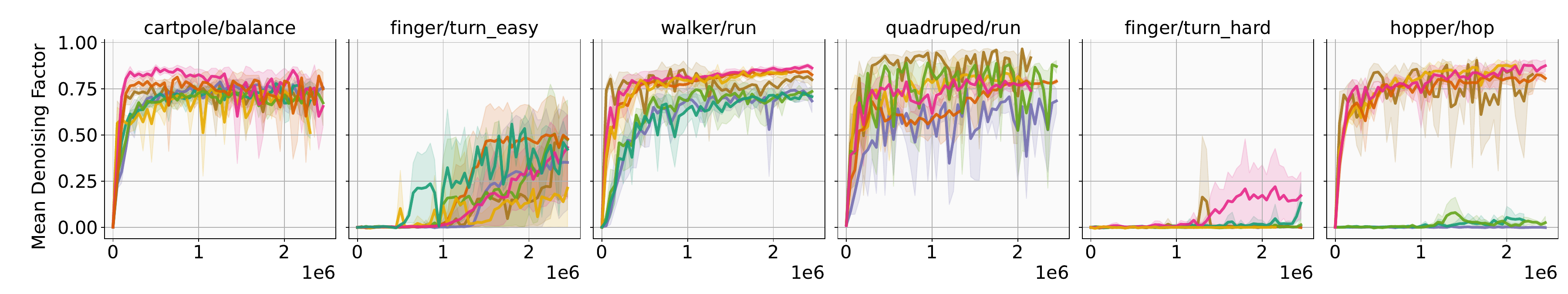}}
            \vspace{-2em}
    \caption{Performance on individual state-based tasks under the IID Gaussian noise with random projection setting.}
                \vspace{-1em}
    \label{fig:rp_pertask8}
\end{figure}

\begin{figure}[h]
    \centering
    \includegraphics[width=0.9\linewidth]{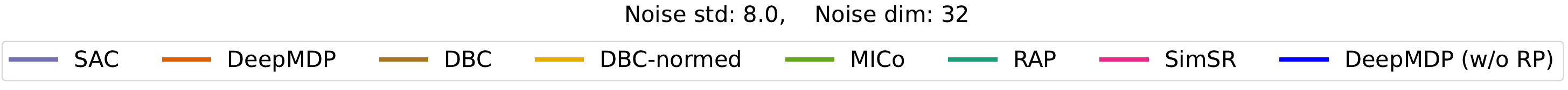}
    \makebox[\linewidth]{%
    \includegraphics[width=\linewidth]{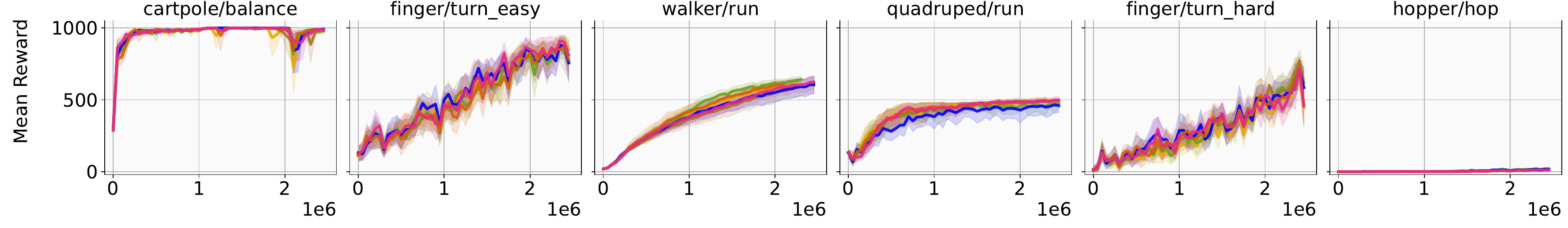}
    }
  \vspace{-2em}
    \caption{Reward curves for the isolated metric estimation setting, using the SAC agent with LayerNorm. Since this setting isolates metric learning, all methods differ only in how the metric encoder $\tilde{\phi}$ is shaped, and thus similar performance across methods is expected. This also ensures that the collected data distribution remains consistent when evaluating different metric learning methods.}
    \label{fig:5.3iso_sacln_rew}
                \vspace{-0.8em}
\end{figure}

\begin{figure}[h]
    \centering
    \includegraphics[width=0.9\linewidth]{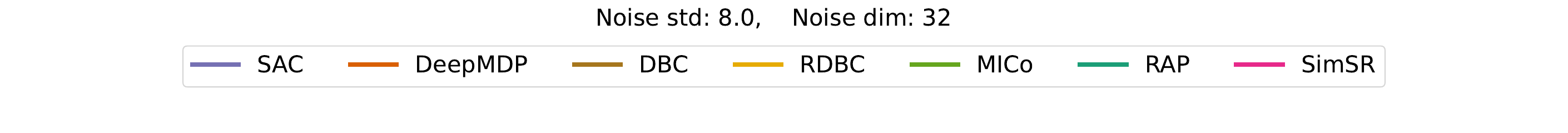}
            \vspace{-1em}
    \includegraphics[width=\linewidth]{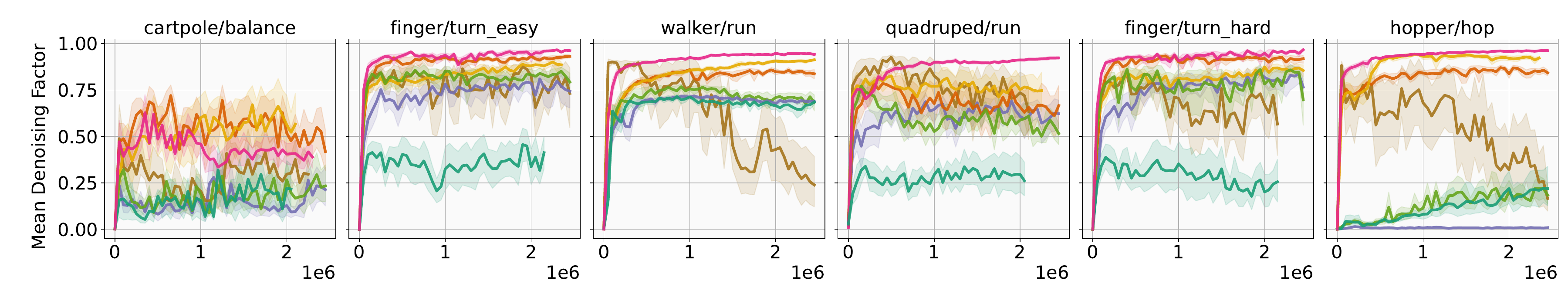}
            \vspace{-1em}
    \caption{DF for the agent encoder $\phi$ (co-trained with RL in \autoref{sec:5.1}) without LayerNorm.}
    \vspace{-1em}
    \label{fig:df_co_training_RL}
\end{figure}

\begin{figure}[h]
    \centering
    \makebox[\linewidth]{%
    \includegraphics[width=\linewidth]{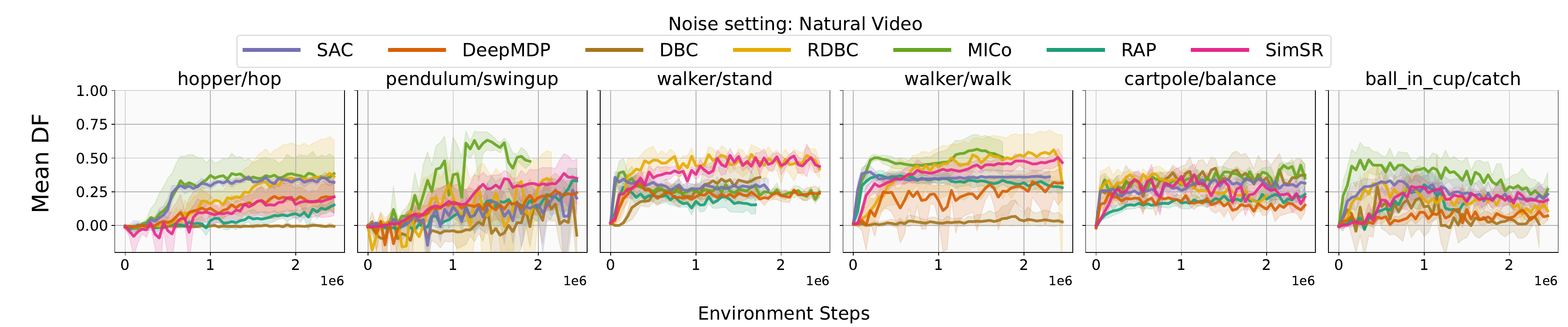}
    }
            \vspace{-2em}
    \caption{Denoising factor curves on six pixel-based DMC tasks under the \textbf{OOD evaluation setting}, measured on the agent encoder $\phi$.}
    \label{fig:5.3iso_id_vidgray_agent}
\end{figure}

\begin{figure}[h]
    \centering
    \makebox[\linewidth]{%
    \includegraphics[width=\linewidth]{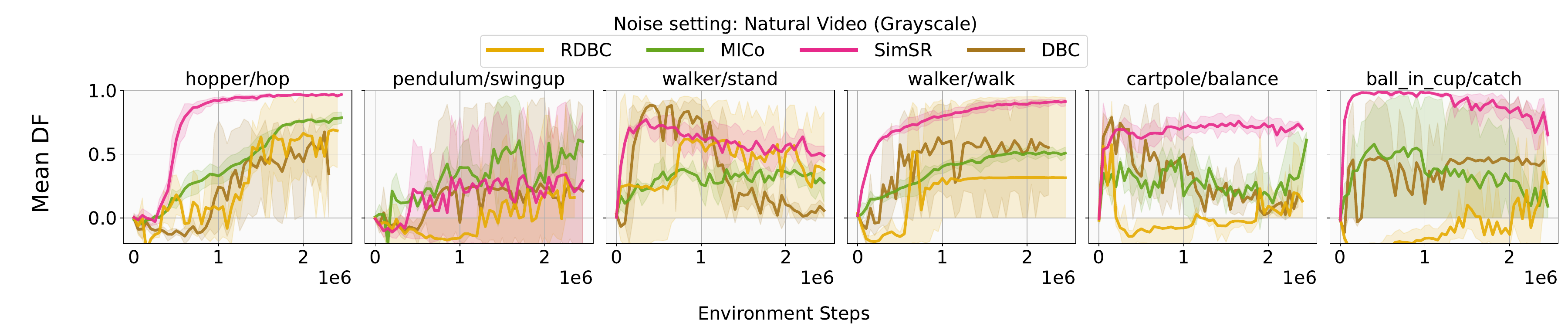}
    }
            \vspace{-2em}
    \caption{Denoising factor curves on six pixel-based DMC tasks under the \textbf{ID evaluation setting}, measured on the isolated metric encoder $\Tilde{\phi}$.}
    \label{fig:5.3iso_id_vidgray}
\end{figure}

\begin{figure}[h]
    \centering
    \makebox[\linewidth]{%
    \includegraphics[width=\linewidth]{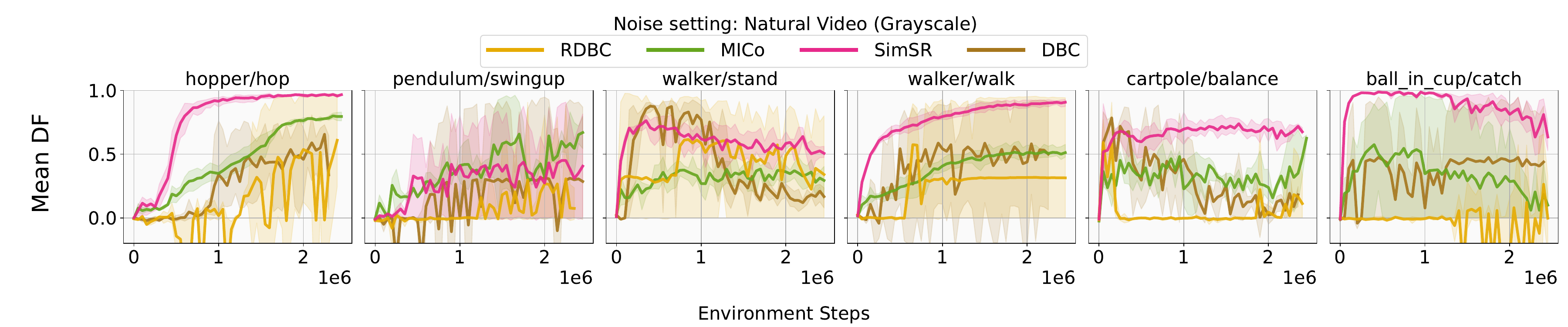}
    }
                \vspace{-2em}
    \caption{Denoising factor curves on six pixel-based DMC tasks under the \textbf{OOD evaluation setting}, measured on the isolated metric encoder $\Tilde{\phi}$.}
    \label{fig:5.3iso_ood_vidgray}
\end{figure}

\begin{figure}[h]
    \centering
    \makebox[\linewidth]{%
    \includegraphics[width=\linewidth]{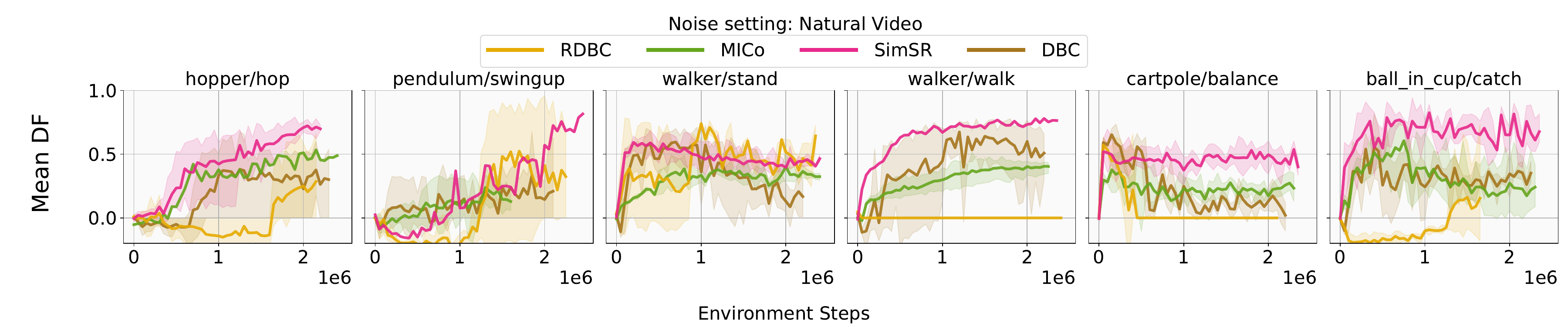}
    }
                \vspace{-2em}
    \caption{Denoising factor curves on six pixel-based DMC tasks under the \textbf{ID evaluation setting}, measured on the isolated metric encoder $\Tilde{\phi}$.}
    
    \label{fig:5.3iso_id_vid}
\end{figure}

\begin{figure}[h]
    \centering
    \makebox[\linewidth]{%
    \includegraphics[width=\linewidth]{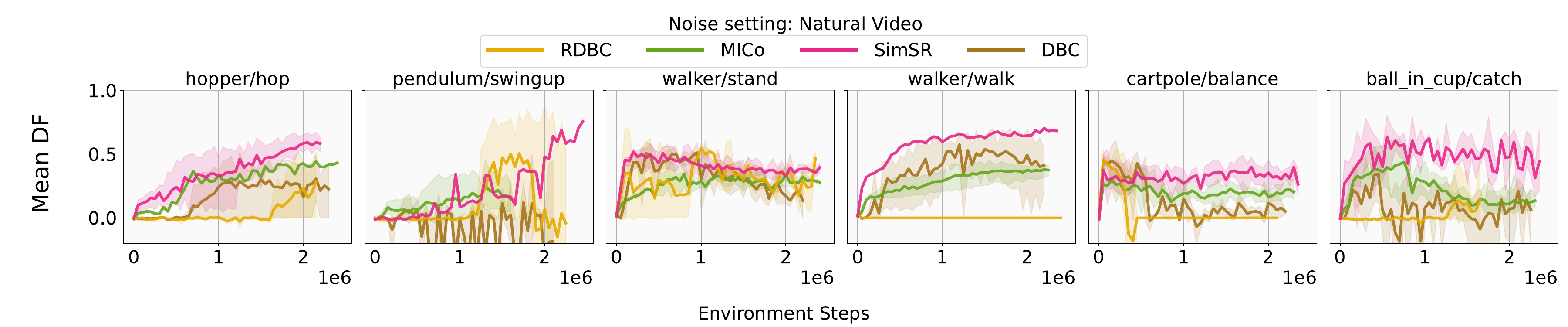}
    }
                \vspace{-2em}
    \caption{Denoising factor curves on six pixel-based DMC tasks under the \textbf{OOD evaluation setting}, measured on the isolated metric encoder $\Tilde{\phi}$.}
    \label{fig:5.3iso_ood_vid}
\end{figure}

\begin{figure}[h]
    \makebox[\linewidth]{%
        \includegraphics[width=1\linewidth]{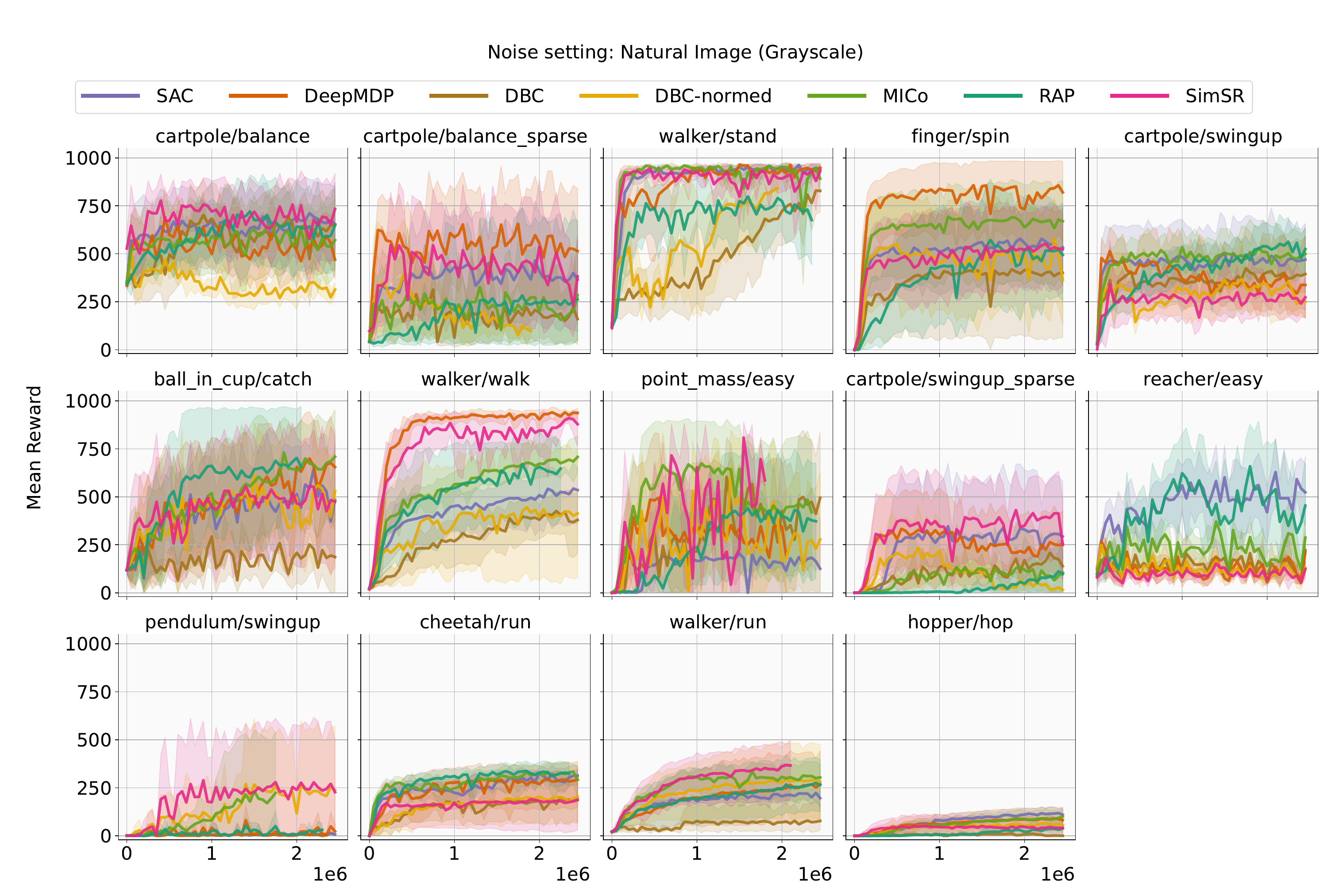}%
    }
                \vspace{-2em}
    \caption{Per-task OOD generalization reward curves for pixel-based DMC tasks.}
    \label{fig:ood_images_gray}
\end{figure}

\begin{figure}[h]
    \makebox[\linewidth]{%
        \includegraphics[width=1\linewidth]{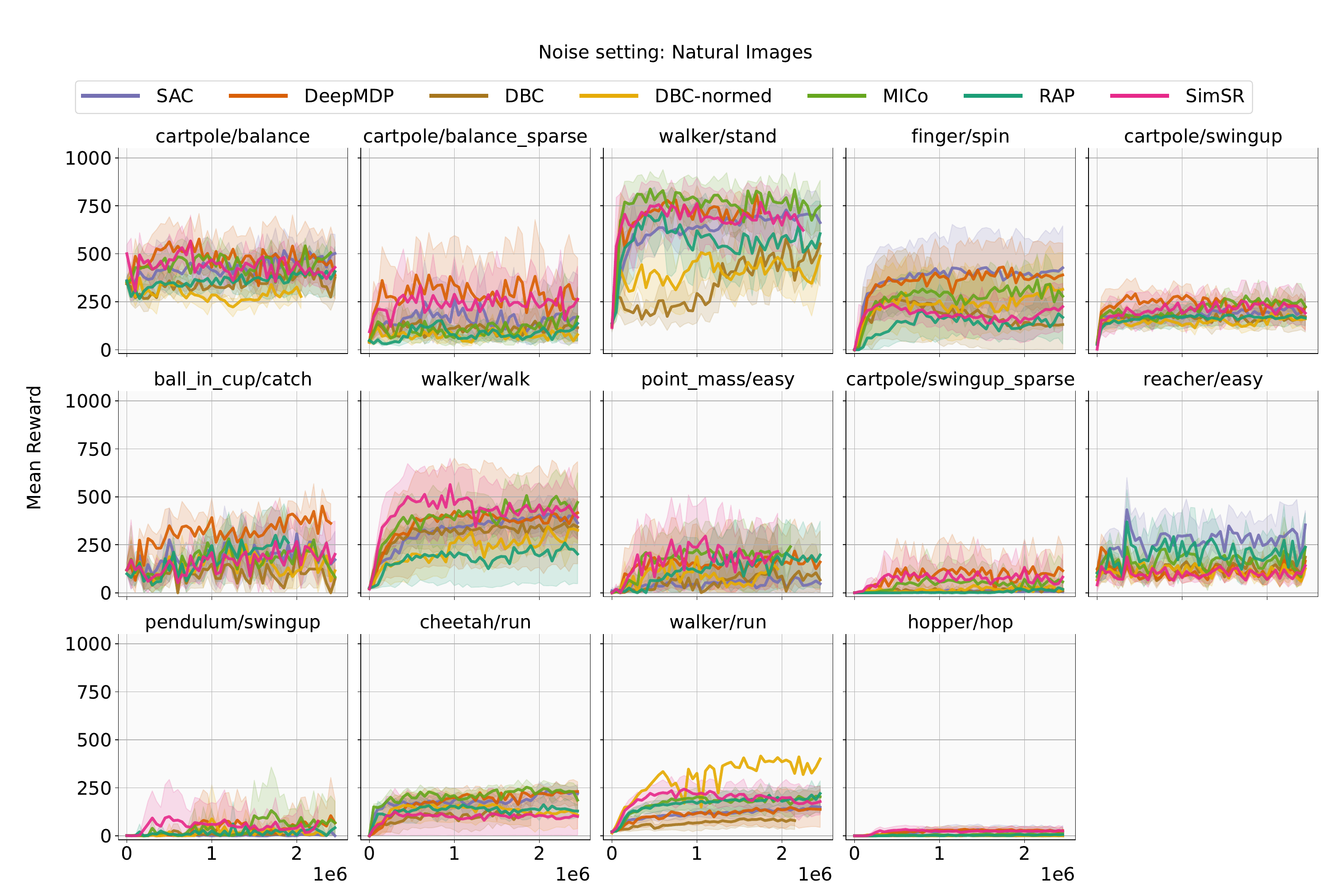}%
    }
                \vspace{-2em}
    \caption{Per-task OOD generalization reward curves for pixel-based DMC tasks.}
    \label{fig:ood_images}
\end{figure}

\begin{figure}[h]
    \makebox[\linewidth]{%
        \includegraphics[width=1\linewidth]{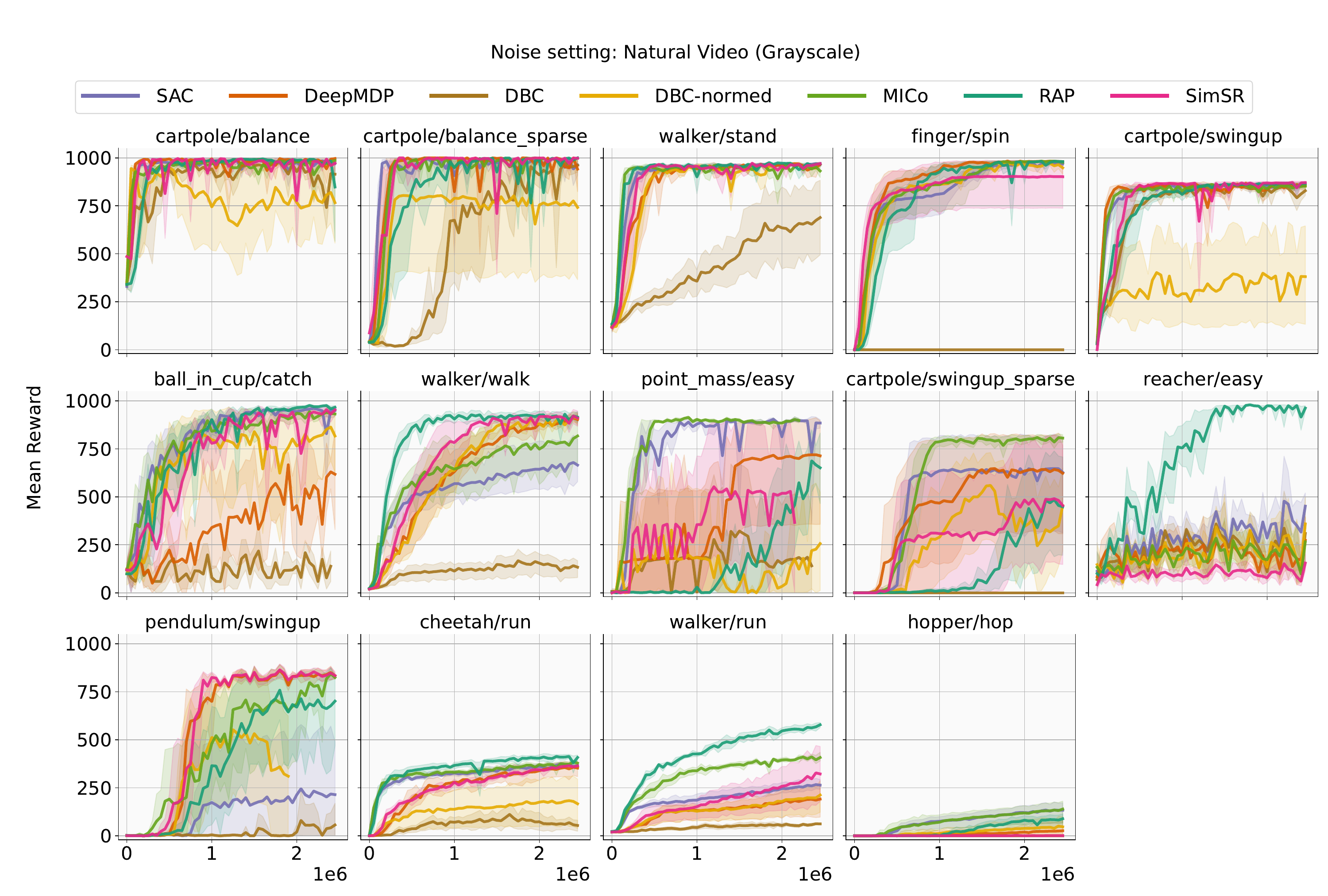}%
    }
                \vspace{-2em}
    \caption{Per-task OOD generalization reward curves for pixel-based DMC tasks.}
    \label{fig:ood_vg}
\end{figure}

\begin{figure}[h]
    \makebox[\linewidth]{%
        \includegraphics[width=1\linewidth]{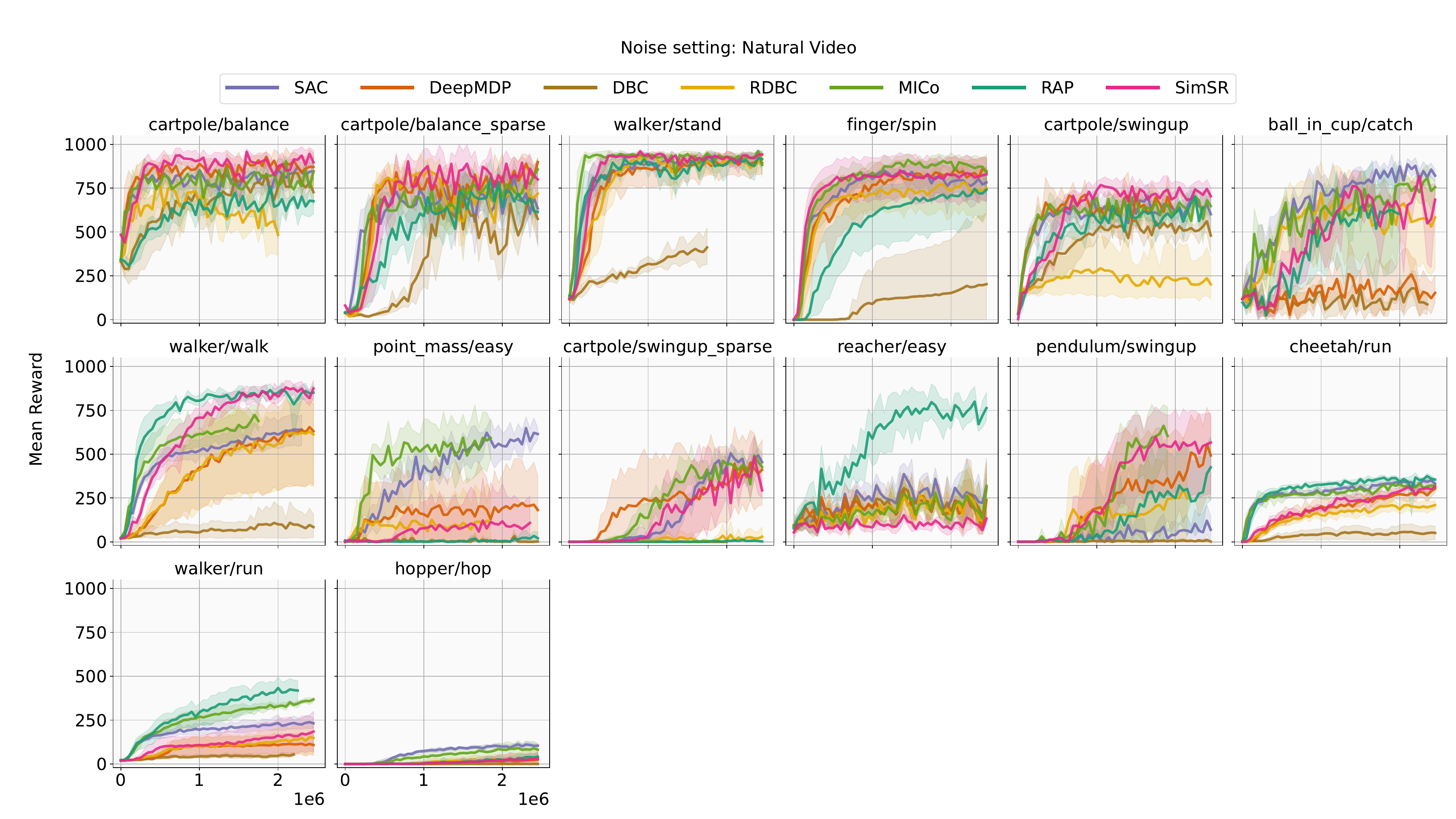}%
    }
                \vspace{-2em}
    \caption{Per-task OOD generalization reward curves for pixel-based DMC tasks.}
    \label{fig:ood_v}
\end{figure}

\begin{figure}[htbp]
    \makebox[\linewidth]{%
        \includegraphics[width=1\linewidth]{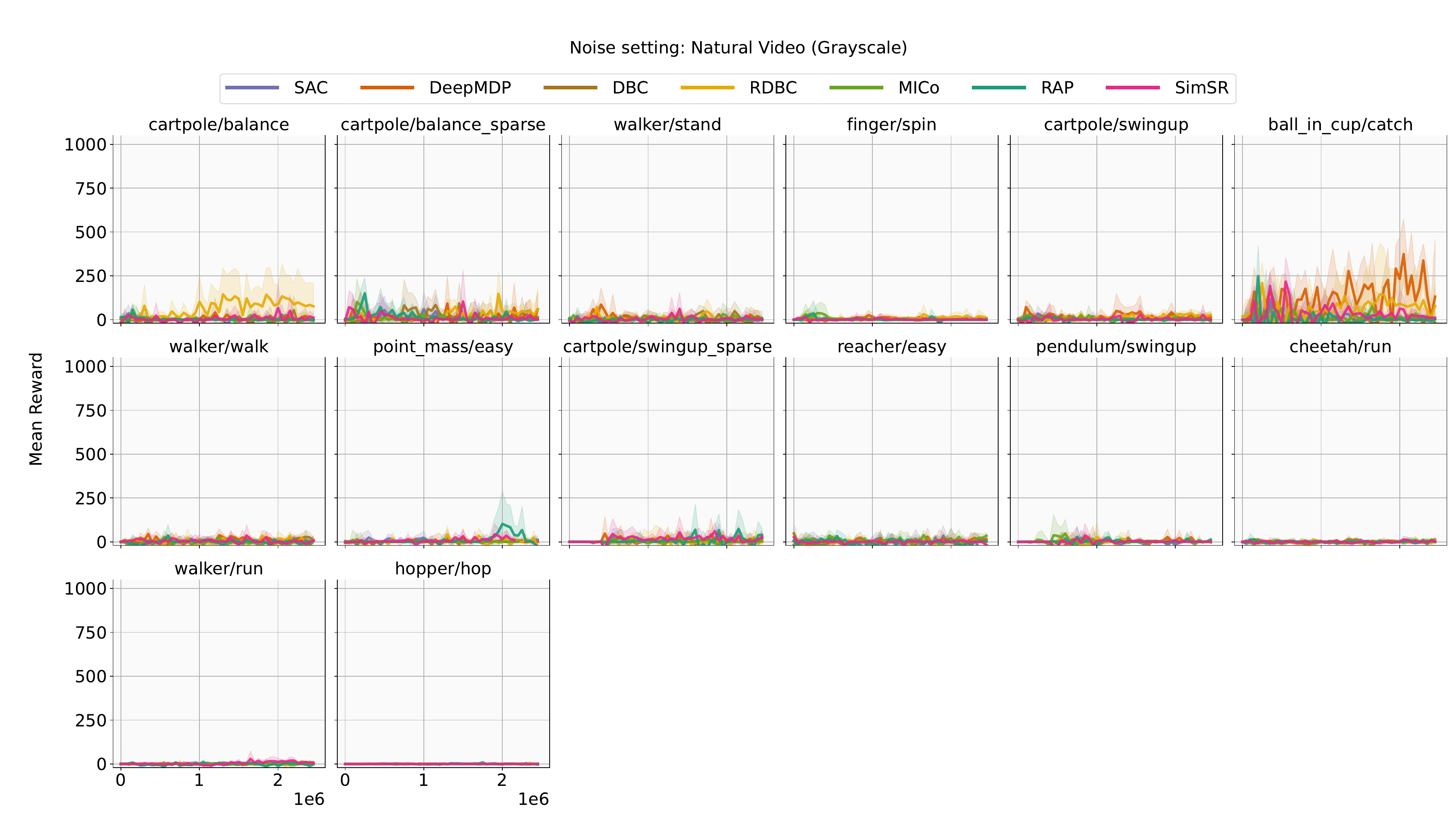}%
    }
                    \vspace{-2em}
    \caption{Pixel-based DMC per-task generalization \textbf{reward difference} curves.}
    \label{fig:gengap_vg}
\end{figure}

\newpage
\begin{figure}[htbp]
    \makebox[\linewidth]{%
        \includegraphics[width=1\linewidth]{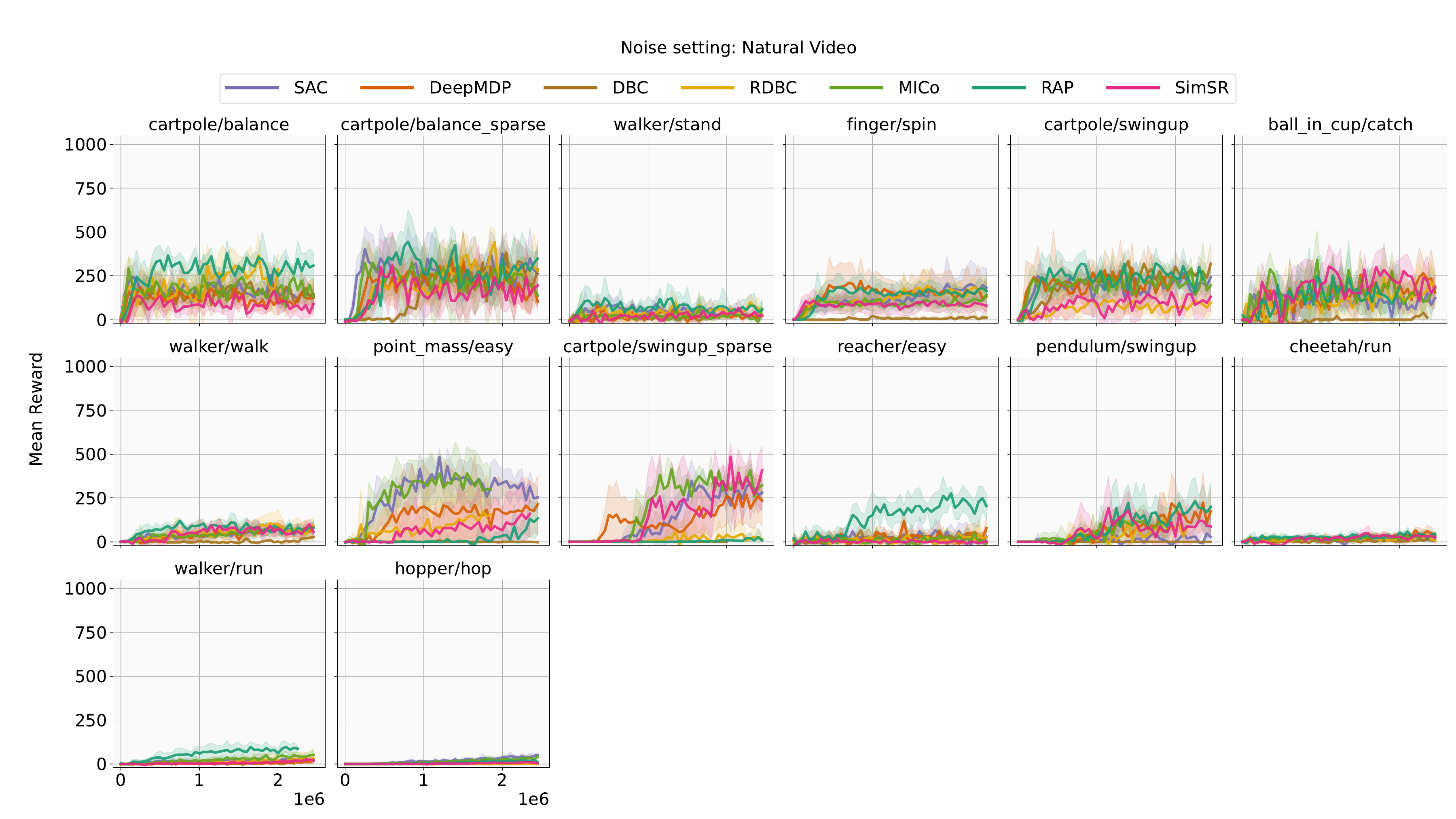}%
    }
                    \vspace{-2em}
    \caption{Pixel-based DMC per-task generalization \textbf{reward difference} curves.}
    \label{fig:gengap_v}
\end{figure}

\end{document}